\documentclass[11pt]{article}
\usepackage[latin9]{inputenc}
\usepackage{color}
\usepackage{varioref}
\usepackage{mathtools}
\usepackage{bm}
\usepackage{amsmath}
\usepackage{amsthm}
\usepackage{amssymb}
\usepackage[normalem]{ulem}
\usepackage[unicode=true,pdfusetitle,
 bookmarks=true,bookmarksnumbered=false,bookmarksopen=false,
 breaklinks=false,pdfborder={0 0 1},backref=false,colorlinks=true]
 {hyperref}
\usepackage{cleveref}
\usepackage{thm-restate}
\makeatletter
\theoremstyle{plain}
\newtheorem{thm}{\protect\theoremname}
\theoremstyle{definition}
\newtheorem{defn}[thm]{\protect\definitionname}
\theoremstyle{plain}
\newtheorem{prop}[thm]{\protect\propositionname}
\theoremstyle{plain}
\newtheorem{lem}[thm]{\protect\lemmaname}
\theoremstyle{plain}

\theoremstyle{plain}
\newtheorem{observation}[thm]{\protect\observationname}

\usepackage{epsfig}
\usepackage{times}

\usepackage{fullpage}

\newcommand{\thickhline}{%
    \noalign {\ifnum 0=`}\fi \hrule height 1pt
    \futurelet \reserved@a \@xhline
}

\newtheorem{claim}[thm]{Claim}

\newcommand{\spn}[1]{\text{Span}\left( #1 \right)}
\DeclareMathOperator*{\argmaxop}{arg\,max}
\DeclareMathOperator*{\argminop}{arg\,min}
\usepackage{bbm}

\newcommand{\myeqold}[1]{\stackrel{\mathclap{\normalfont\mbox{#1}}}{=}}
\newcommand{\myleqold}[1]{\stackrel{\mathclap{\normalfont\mbox{#1}}}{\leq}}
\newcommand{\mygeqold}[1]{\stackrel{\mathclap{\normalfont\mbox{#1}}}{\geq}}

\newcommand{\red}[1]{{ #1}}
\newcommand{\blue}[1]{{ #1}}
\global\long\def\conv#1{\mathop{\text{conv}}\left(#1\right)}%

\makeatother

\providecommand{\corollaryname}{Corollary}
\providecommand{\definitionname}{Definition}
\providecommand{\lemmaname}{Lemma}
\providecommand{\observationname}{Observation}
\providecommand{\propositionname}{Proposition}
\providecommand{\theoremname}{Theorem}

\usepackage{array}
\usepackage{epsfig}
\usepackage{fullpage}
\usepackage{enumerate}
\usepackage{algorithm}
\usepackage{comment}
\usepackage{cryptocode}
\usepackage{graphicx}
\usepackage{url}
\usepackage{subcaption}
\usepackage{indentfirst}
\usepackage{booktabs}
\usepackage{xfrac}
\usepackage{physics}

\newcommand{\E}{\mathop{\mathbb{E}}}

\newcommand{\R}{\mathbb{R}}

\def\poly{\operatorname{poly}}

\title{Testing assumptions summary}
\author{arsenvasilyan }
\date{February 2022}

\begin{document}

\global\long\def\B{\mathcal{B}}%

\global\long\def\red#1{{\color{red}#1}}%
\global\long\def\blue#1{{\color{blue}#1}}%

\global\long\def\dimvc{\dim_{\text{VC}}}%

\global\long\def\spn#1{\text{Span}\left(#1\right)}%

\global\long\def\argmax{\argmaxop}%

\global\long\def\argmin{\argminop}%

\global\long\def\E{\mathbb{E}}%

\global\long\def\R{\mathbb{R}}%

\global\long\def\P{\mathbb{P}}%

\global\long\def\C{\mathbb{C}}%

\global\long\def\Z{\mathbb{Z}}%

\global\long\def\F{\mathcal{F}}%

\global\long\def\Ep{\mathcal{E}}%

\global\long\def\sign{\mathbb{\text{sign}}}%

\global\long\def\myeq#1{\myeqold{#1}}%

\global\long\def\myleq#1{\myleqold{#1}}%

\global\long\def\mygeq#1{\mygeqold{#1}}%

\global\long\def\indicator{\mathbbm{1}}%

\global\long\def\parr#1{\left(#1\right)}%

\global\long\def\pars#1{\left[#1\right]}%

\global\long\def\fig#1{\left\{  #1\right\}  }%

\global\long\def\abs#1{\left|#1\right|}%

\global\long\def\norm#1{\left\Vert #1\right\Vert }%

\global\long\def\bra#1{\left\langle #1\right\rangle }%

\global\long\def\d{\,d}%

\global\long\def\conditional{\bigg\vert}%

\global\long\def\vect#1{\bm{#1}}%

\global\long\def\mydot#1{\stackrel{\mathbf{.}}{#1}}%

\global\long\def\mydoubledot#1{\stackrel{\mathbf{.}\mathbf{.}}{#1}}%

\global\long\def\mytripledot#1{\stackrel{\mathbf{.}\mathbf{.}\mathbf{.}}{#1}}%

\global\long\def\mybar#1{\overline{#1}}%

\global\long\def\myhat#1{\widehat{#1}}%

\global\long\def\spn#1{\mathop{\text{span}}\left(#1\right)}%

\global\long\def\righ{\rightarrow}%

\global\long\def\var{\text{var}}%

\global\long\def\N{\mathcal{N}}%

\begin{titlepage}
		\def\thepage{}
		\thispagestyle{empty}
		
		\title{Testing distributional assumptions of learning algorithms} 
		
		\date{}
		\author{
			Ronitt Rubinfeld\thanks{MIT, {\tt ronitt@csail.mit.edu}.  
			Supported in part by NSF awards CCF-2006664, DMS-2022448 and Fintech@CSAIL.  }
			\and
			Arsen Vasilyan \thanks{MIT, {\tt vasilyan@mit.edu}.
			Supported in part by NSF awards CCF-1565235, CCF-1955217, DMS-2022448, Big George Fellowship and Fintech@CSAIL.
			}
		}
		
		\maketitle

\abstract{
There are many important {\em high dimensional} function classes
that have fast agnostic learning algorithms when strong assumptions
on the distribution of examples can be made, such as Gaussianity or uniformity over the domain.
But how can one be sufficiently confident that the data indeed
satisfies the distributional assumption, so that one can  
trust in the output quality of the agnostic learning algorithm?
We  propose a model by which to systematically study the design of 
\emph{tester-learner pairs} $(\mathcal{A},\mathcal{T})$, 
such that if the distribution on examples in the data
passes the tester $\mathcal{T}$ then one can {\em safely trust}  the
output of the  agnostic
learner $\mathcal{A}$ on the data.

To demonstrate the power of the model, we apply it to the classical problem of
agnostically learning halfspaces under the standard Gaussian distribution
and 
present a tester-learner pair with a combined run-time of 
$n^{\tilde{O}(1/\epsilon^4)}$.
This qualitatively matches that of the
best known ordinary agnostic learning algorithms for this task. 
In contrast, finite sample Gaussian distribution testers do not
exist for the $L_1$ and EMD distance measures. 
Previously it was known that half-spaces are well-approximated with low-degree polynomials relative to the Gaussian distribution. A key step in our analysis is showing that this is the case even relative to distributions whose low-degree moments approximately match those of a Gaussian.

We also go beyond spherically-symmetric distributions, and give a tester-learner pair for halfspaces under the uniform distribution on $\{0,1\}^n$ with combined run-time of $n^{\tilde{O}(1/\epsilon^4)}$. This is achieved using polynomial approximation theory and critical index machinery of \cite{DiakonikolasGJSV09}. 


Can one design agnostic learning algorithms under distributional assumptions and count on future technical work to produce, as a matter of course, tester-learner pairs with similar run-time? Our answer is a resounding no, as we show there exist some well-studied settings for which $2^{\tilde{O}(\sqrt{n})}$ run-time agnostic learning algorithms are available, yet the combined
run-times of tester-learner pairs must be as high as $2^{\Omega(n)}$. On that account, the design of tester-learner pairs is a research direction in its own right independent of standard agnostic learning.
To be specific, our lower bounds apply to the problems of agnostically learning convex sets under
the Gaussian distribution and for monotone Boolean functions
under the uniform distribution over $\{0,1\}^n$.

}
\end{titlepage}

\maketitle


\section{Introduction.} 
\subsection{Motivation.}

Suppose one wants to learn from i.i.d. example-label pairs, but some unknown fraction of labels are corrupted by an adversary.
The well-studied field of \textbf{agnostic learning} seeks to develop learning algorithms that are robust to such corruptions.
Agnostic learning can be notoriously harder than standard learning (see for example \cite{GuruswamiR06,FeldmanGKP06,Daniely16}).  
Nevertheless, there are many important \textbf{high dimensional}
function classes that do have fast agnostic learning algorithms, including halfspaces, convex sets and monotone Boolean functions. However, these learning algorithms 
make strong assumptions about the underlying distribution on examples
, such as Gaussianity or uniformity over $\{0,1\}^n$.

Thus, to be confident in such a learning algorithm one needs to be confident in the distributional assumption. 
In some cases, users can attain confidence in their distributional assumptions 
by creating their own set of examples which conform to the distribution, and querying  labels for these examples.
Yet, this approach
requires query access, which is often unavailable.
Is there a way to ascertain  that the examples are indeed coming from a distribution for which the learning algorithm will give a robust answer? 

We propose to systematically study the design of \textbf{tester-learner pairs} $(\mathcal{A},\mathcal{T})$, such that \textbf{tester $\mathcal{T}$ tests the distributional assumptions of agnostic learner $\mathcal{A}$}. In other words,
the tester-learner pair is to be designed such that if the distribution on examples
in the data pass the tester, then one can {\em safely use the learner on the data}. 
By considering the most basic requirements that such a pair ought to satisfy,
we propose a new model that makes the following end-to-end requirements on 
a  tester-learner pair
$(\mathcal{A},\mathcal{T})$: 
\begin{itemize}
    \item  \textbf{Composability:}
    For any example-label distribution, it should be unlikely that simultaneously (i) the tester $\mathcal{T}$ accepts but (ii) the learner $\mathcal{A}$ outputs something not satisfying the agnostic learning guarantee. 
    \item \textbf{Completeness:} If the distribution on examples conforms to the distributional assumption, tester $\mathcal{T}$ will likely accept. 
    \item The performance of the tester-learner pair is judged by the combined run-time of $\mathcal{A}$ and $\mathcal{T}$.
\end{itemize}
 See Section \ref{sec: definition of tester/learner} for the fully formal definition and see Subsection \ref{subsec: more comments} for more comments. 

We emphasize that assumptions on the distribution of examples are in fact made in a very large number of works on agnostic learning
\footnote{The reason for this ubiquity of distributional assumptions in high-dimensional agnostic learning is that with no assumption at all on the distribution the task of agnostic learning is usually intractable. For example (i) The task of learning indicators of convex sets over $\R^n$ cannot be achieved with finite number of samples if nothing is assumed about the distribution. If the distribution is assumed to be Gaussian, this task can be achieved with run-time of $n^{\tilde{O}(\sqrt{n}/\epsilon^4)}$ \cite{KlivansOS08}.  
 (ii) If one is unwilling to make any distributional assumption, no agnostic learning algorithm for halfspaces with run-time of $2^{o(n)}$ is known despite decades of research (also see \cite{GuruswamiR06,FeldmanGKP06,Daniely16} for some known hardness results). 
However, as we mentioned if the examples are distributed according to the standard Gaussian, a dramatically faster run-time of $n^{\tilde{O}(1/\epsilon^2)}$ is achievable \cite{kalai_agnostically_2005,DiakonikolasGJSV09}. 
}. Here is an incomplete list of such papers that only scratches the surface:  \cite{kalai_agnostically_2005,odonnell_learning_2006,blais_polynomial_2008,KlivansOS08,gopalan_learning_2010,kane_gaussian_2010,wimmer_agnostically_2010,harsha_invariance_2010,diakonikolas_bounding_2010,cheraghchi_submodular_2012,AwasthiBL14,dachman-soled_approximate_2014,feldman_tight_2015,feldman_agnostic_2015,blais_learning_2015,canonne_testing_2017,feldman_tight_2017,diakonikolas_agnostic_2021}. Hence, we think it is important to understand to what extent these distributional assumptions can be tested. 

Perhaps surprisingly, in spite of how natural this definition is, nothing was previously known on how well it can be achieved for various well-studied  problems. The gamut of open possibilities included the most optimistic one: that for all these problems one can test the assumption with very small overhead relative to the existing agnostic learning algorithms. It also included the most pessimistic one: that for all these problems one can test the assumption only at a very steep additional cost in terms of run-time. We note that such steep additional cost would indeed be payed if one were to use existing identity testers of $n$-dimensional distributions, as these testers have run-times of $2^{\Omega(n)}$ (see below for more information on this).

We commence the charting of the landscape of these possibilities. 
We find that neither of these extreme possibilities holds in general. On one hand, we find that for some natural problems the most optimistic possibility does materialize and there is a tester-learner pair whose run-time is of the same order as that of the best known agnostic learning algorithm. Specifically, for agnostically learning the class of half-spaces with respect to standard\footnote{Note that the case of Gaussian distribution with arbitrary known mean and covariance reduces to the case of standard Gaussian via a change of coordinates. } Gaussian distribution, we design a tester-learner pair $(\mathcal{A},\mathcal{T})$ with combined run-time of $n^{\tilde{O}(1/\epsilon^4)}$.
This run-time qualitatively matches the run-time of $n^{\tilde{O}(1/\epsilon^2)}$ \cite{kalai_agnostically_2005,DiakonikolasGJSV09} achieved by the best algorithm\footnote{However, note that the work of \cite{Daniely15} shows how to obtain an even faster run-time of $\poly\left({n,\frac{1}{\epsilon}}\right)$
if one is willing to settle for a weaker guarantee than the standard agnostic learning guarantee. Specifically, for any absolute constant $\mu$, \cite{Daniely15} gives a predictor, such that, if the best halfspace has error $\text{opt}$, the predictor of \cite{Daniely15} will have error of at most $(1+\mu)\text{opt}+\epsilon$ (note that standard agnostic learning requires an error bound of $\text{opt}+\epsilon$). In this work we only consider standard agnostic learning.} and the statistical query lower bound of $n^{\Omega(1/\epsilon^2)}$ by  \cite{GoelGK20,DiakonikolasKZ20,DiakonikolasKPZ21}. We also go beyond spherically-symmetric distributions, and give a tester-learner pair for halfspaces under the uniform distribution on $\{0,1\}^n$ with combined run-time of $n^{\tilde{O}(1/\epsilon^4)}$. 
Here also, the run-time qualitatively matches the run-time of $n^{\tilde{O}(1/\epsilon^2)}$ \cite{kalai_agnostically_2005,DiakonikolasGJSV09} achieved by the best algorithm.
Additionally, we remark that positive results in our framework extend to function classes beyond halfspaces and, as a proof of concept, we give a simple tester-learner pair for agnostically learning decision lists\footnote{For this example,
a {\em decision list} is a special case of a decision tree
corresponding to a path.  More formally, for some ordering of the
variables $x_{\pi(1)},
\ldots, x_{\pi(n)}$, values $v_1,\ldots,v_n$ and bits $b_1,\ldots,b_n$,
~a decision list does the following:  For $i = 1 {\rm ~to~} n$,
if $x_{\pi(i)}=b_{\pi(i)}$ output $v_{\pi(i)}$, else continue.
A more general definition is given in \cite{Rivest87}.
}  under uniform distribution on $\{0,1\}^n$ (see \Cref{appendix: proof for decision lists}).

On the other hand, for some other natural problems, we show that the most pessimistic scenario holds and the additional requirement of testing the distributional assumption comes at a steep price in terms of run-time. Specifically: 
\begin{itemize}\item A well-known algorithm of \cite{KlivansOS08} agnostically learns convex sets under the Gaussian distribution with a run-time of $n^{\tilde{O}(\sqrt{n}/\epsilon^4)}$.
    We show that if a tester $\mathcal{T}$ tests the distributional assumption of this algorithm, then $\mathcal{T}$ has run-time of $2^{\Omega(n)}$. More generally, \textbf{any} tester-learner pair for this task requires $2^{\Omega(n)}$ run-time combined.
    \item A well-known algorithm of \cite{bshouty_fourier_1995,kalai_agnostically_2005} agnostically learns monotone Boolean functions under uniform distribution over $\{0,1\}^n$ with a run-time of $2^{\tilde{O}\left(\frac{\sqrt{n}}{\epsilon^2}\right)}$. We show that if a tester $\mathcal{T}$ tests the distributional assumption of this algorithm, then $\mathcal{T}$ has run-time of $2^{\Omega(n)}$. Again, \textbf{any} tester-learner pair for this task requires $2^{\Omega(n)}$ run-time combined.
\end{itemize}
We emphasize that these lower bounds exhibit natural problems where there is a dramatic gap between standard agnostic learning run-time and the run-time of the best tester-learner pair.
Therefore, there is provably no general method that allows one to automatically convert standard agnostic learning algorithms into tester-learner pairs with low run-time overhead. 

Additionally, lower bounds for tester-learner pairs can imply lower bounds for standard agnostic learning: Specifically, our lower bounds imply that agnostic learning of monotone functions under distributions $\frac{1}{2^{n^{0.99}}}$-close\footnote{In total variation distance.} to $n^{0.99}$-wise independent distributions requires $2^{\Omega(n)}$ run-time. 
The reason is that by \cite{ODonnellZ18,AlonAKMRX07,AlonGM03} one can test $n^{0.99}$-wise independence up to error $\frac{1}{2^{n^{0.99}}}$ in time $2^{\tilde{O}(n^{0.99})}$, and therefore the existence of such an algorithm would contradict our general lower bound for tester-learner pairs. 
As there are $2^{\tilde{O}(\sqrt{n}/\epsilon^2)}$ time learners for monotone functions over the
uniform distribution \cite{bshouty_fourier_1995,kalai_agnostically_2005}, this lower bound highlights the 
sensitivity of agnostic learners  to the assumption on the input distribution.

\paragraph{Distribution testing perspective.}
Existing work on identity testing of $n$-dimensional distributions has focused on testing with respect to very strict distance measures (i.e. TV distance, earth-mover distance, etc.). On one hand this yields strong general-purpose guarantees on distributions accepted by the tester -- it is hard to think of a situation where closeness in TV distance is unsatisfactory. On the other hand, in $n$ dimensions this leads to run-times of $2^{\Omega(n)}$. As a concrete example, distinguishing the uniform distribution over $\{0,1\}^n$ from a distribution that is $\epsilon$-far from it in total variation distance requires a run-time of $\Theta \left(\frac{1}{\epsilon^2}2^{n/2}\right)$ (see text \cite{ClementSurvey}).

Yet, run-times of $2^{\Omega(n)}$ can be prohibitive. Indeed, as we explained above, the theory of $n$-dimensional agnostic learning aims at developing algorithms with run-times of $2^{o(n)}$ or even $n^{O_\epsilon(1)}$. If one were to combine these algorithms with a $2^{\Omega(n)}$-run-time distribution tester, the total run-time would rise precipitously. 

From the distribution testing perspective, this work studies \textbf{application-targeted testers} that, in favor of much faster run-time, forgo the general-purpose guarantees provided by these strict distance measures. The application domain  which this work considers  is the testing of distributional assumptions made by agnostic learning algorithms. Here, the application-targeted testers are developed with a view towards special-purpose guarantees sufficient to ensure that the learning algorithms  are still robust. For some problems in this domain -- this work shows -- the use of general-purpose testers can indeed be circumvented, with a dramatic gain in run-time.   

In general, surprisingly little is known about such application-targeted testers and we hope more application-targeted distribution testers can be developed for other domains.
\paragraph{Recent followup work \cite{ConcurrentWork22}.}
In an exciting new development we were contacted regarding a follow up work \cite{ConcurrentWork22} that builds on an earlier version of this paper, which had been made available to them. \cite{ConcurrentWork22} develops novel techniques for the design and analysis of tester-learner pairs that leverage connections with the notion of fooling a function class from the field of pseudorandomness. This allows \cite{ConcurrentWork22} to
\begin{itemize}
    \item Give tester-learner pairs for more general function classes, such as intersections of halfspaces.
    \item Handle more general classes of distributional assumptions, such as strictly subexponential distributions in $\R^n$ and uniform over $\{0,1\}^n$.
    \item Present a new connection between the notion of tester-learner pairs and Rademacher complexity.
    \item Improve on our run-time for tester-learner pairs for halfspaces under the Gaussian distribution on $\R^n$. Specifically, they give a bound of $n^{\tilde{O}(1/\epsilon^2)}$ which improves upon our bound of $n^{\tilde{O}(1/\epsilon^4)}$. Their tighter bound also matches the known statistical query lower bounds  \cite{GoelGK20,DiakonikolasKZ20,DiakonikolasKPZ21}.
\end{itemize}

We would like to note that Theorem \ref{thm: main testing learning theorem for LTFs on cube} (tester-learner pairs for halfspaces under the uniform distribution on $\{0,1\}^n$) is concurrent work with \cite{ConcurrentWork22} (they give a faster run-time of $n^{\tilde{O}(1/\epsilon^2)}$ for this problem and also give more general results as explained above). The earlier version of our work (which they build upon) already contained the other results presented in our current version, i.e. (i) the definition of tester-learner pairs (ii) the tester learner pair for half-spaces under the Gaussian distribution with run-time $n^{\tilde{O}(1/\epsilon^4)}$ (Theorem \ref{thm: main testing learning theorem for LTFs}) (iii) the intractability results for tester-learner pairs in Theorems \ref{thm: hardness for convex strong} and \ref{thm: hardness for monotone strong}.

\subsection{Our techniques.}

\begin{table}[!htbp]
\begin{center}
\begin{tabular}{p{5cm}  c  c  c } 
 \hline
 Function class & Halfspaces  & Halfspaces    \\ \midrule
 Distributional assumption & Standard Gaussian in $\R^n$ & Uniform on $\{0,1\}^n$   \\ \midrule
 Standard agnostic learning \emph{run-time} from literature & \begin{tabular}{@{}c@{}} $n^{\tilde{O}(1/\epsilon^2)}$\\  \cite{kalai_agnostically_2005,DiakonikolasGJSV09} \end{tabular}   & \begin{tabular}{@{}c@{}} $n^{\tilde{O}(1/\epsilon^2)}$\\  \cite{kalai_agnostically_2005,DiakonikolasGJSV09} \end{tabular}     \\ \midrule
  Standard agnostic learning \emph{intractability} from literature & 
  \begin{tabular}{@{}c@{}} $n^{\Omega(1/\epsilon^2)}$ statistical queries \\ \cite{GoelGK20,DiakonikolasKZ20,DiakonikolasKPZ21} \end{tabular}
     & \begin{tabular}{@{}c@{}} We are not aware of published intractability \\ results in this precise setting. \end{tabular}
         \\ \midrule
  Examples needed for testing assumption in TV distance & infinite &   \begin{tabular}{@{}c@{}}$\Theta \left(\frac{1}{\epsilon^2}2^{n/2}\right)$ \\ (see text \cite{ClementSurvey})\end{tabular}  
  \\\midrule
 The run-time of our tester-learner pair & 
\scalebox{1.2}{$n^{\tilde{O}\left(1/\epsilon^4\right)}$}
& \scalebox{1.2}{$n^{\tilde{O}\left(1/\epsilon^4\right)}$}  \\
 \bottomrule
\end{tabular}
\end{center}

\caption{Summary of our algorithms and relevant previous work.}
\label{table:1}
\end{table}
\begin{table}[!htbp]
\begin{center}
\begin{tabular}{p{5cm}  c  c  c } 
 \hline
 Function class &  Convex sets  & Monotone functions  \\ \midrule
 Distributional assumption  & Standard Gaussian in $\R^n$  & Uniform on $\{0,1\}^n$ \\ \midrule
 Standard agnostic learning \emph{run-time} from literature &  \begin{tabular}{@{}c@{}}$n^{\tilde{O}(\sqrt{n}/\epsilon^4)}$ \\ \cite{KlivansOS08} \end{tabular}    & \begin{tabular}{@{}c@{}} $2^{\tilde{O}(\sqrt{n}/\epsilon^2)}$ \\  \cite{bshouty_fourier_1995,kalai_agnostically_2005} \end{tabular} \\ \midrule
  Standard agnostic learning \emph{intractability} from literature 
     & \begin{tabular}{@{}c@{}} $n^{\Omega(\sqrt{n})}$\\  \cite{KlivansOS08} \end{tabular}
       & \begin{tabular}{@{}c@{}} $2^{\tilde{\Omega}(\sqrt{n})}$\\ \cite{blais_learning_2015}  \end{tabular}  \\ \midrule
  Examples needed for testing assumption in TV distance & infinite  & 
  \begin{tabular}{@{}c@{}}$\Theta \left(\frac{1}{\epsilon^2}2^{n/2}\right)$ \\ (see text \cite{ClementSurvey})\end{tabular}
  \\ \midrule
 Our lower bound for combined run-time of a tester-learner pair 
& \scalebox{1.2}{$2^{\Omega(n)}$} & \scalebox{1.2}{$2^{\Omega(n)}$}  \\
 \bottomrule
\end{tabular}
\end{center}

\caption{Summary of our intractability results and relevant previous work.}
\label{table:2}
\end{table}

We summarize our contributions and relevant background in Table \ref{table:1} on  \cpageref{table:1} and Table \ref{table:2} on  \cpageref{table:2}.

\paragraph{Tester-learner pair for agnostically learning halfspaces under Gaussian distribution}
We first give an overview of our tester-learner pair $(\mathcal{A},\mathcal{T})$ with combined run-time of $n^{\tilde{O}(1/\epsilon^4)}$ for the class of half-spaces with respect to standard Gaussian distribution. We also discuss the techniques we use to analyze it. See \Cref{section: tester-learner under Gaussian,sec: proof of two main lemmas,section: proof of main theorem} for complete details.

A natural first approach would be to try to take advantage of the literature on testing and learning distributions.
However, almost all results we are aware of on testing and learning high-dimensional distributions (without assuming the distribution already belongs to some highly restricted family as in \cite{cai2013optimal}) require a number of samples that is  exponentially large in the dimension. 
It follows from well-known techniques that Gaussianity over an infinite domain 
cannot be tested with respect to total variation distance in finite samples.
Potentially, one could obtain a tester-learner pair for Gaussianity with respect to
the earth-mover distance via the tester\footnote{This tester requires that the distribution is confined to a box $[-B, B]^n$, but this by itself is not a devastating problem, since most of probability mass of a Gaussian is confined to such a box.} of \cite{BaNNR11}, yielding a tester of run-time $2^{\tilde{O}(n)}$. However one can see that, in earth-mover distance, no significantly better (i.e. $2^{o(n)}$) bound can be obtained\footnote{Even when truncating the distribution to a box around the origin.}.  
Such enormous run-times far exceed the run-times that can be achieved for agnostically learning halfspaces.


Previously it was known that half-spaces are well-approximated with low-degree polynomials relative to the Gaussian distribution. A key step in our analysis is showing that this is the case even relative to distributions whose low-degree moments approximately match those of a Gaussian.
One of our ideas is to start with a proof of the exact Gaussian case and modify it so it only relies on low-degree properties of the distribution.
We are aware of three distinct proofs of this exact Gaussian case in the literature:
\begin{enumerate}
    \item The method of \cite{kalai_agnostically_2005} that uses specific facts about Hermite polynomials. 
    \item The noise sensitivity method of \cite{KlivansOS08}. This method also uses Hermite polynomials to argue that functions that tend to be stable to perturbations of their input tend to be well-approximated by low-degree polynomials. 
    \item The method of \cite{DiakonikolasGJSV09} that, in order to approximate a halfspace $\sign(\vect{v}\cdot \vect{x}+\theta)$, constructs a polynomial $P(\vect{v}\cdot \vect{x})$ that approximates this halfspace tightly for values of $\abs{\vect{v}\cdot \vect{x}}$ that are not too large. It is then argued that large values of $\abs{\vect{v}\cdot \vect{x}}$ do not contribute much to the total $L_1$ error of the polynomial because its contribution is weighted by a rapidly decaying Gaussian weight. 
\end{enumerate}
As Hermite polynomials are the unique family of polynomials orthogonal under the Gaussian distribution, the proof strategies of \cite{kalai_agnostically_2005} and \cite{KlivansOS08} seem highly specialized to the distribution being exactly Gaussian. Because of this, a method similar to the one of \cite{DiakonikolasGJSV09} is the one serving as our starting point. 

This method needs to be modified in a thoroughgoing way in order to rely merely on the low-degree moments of the distribution being close to those of Gaussian. For instance, a very easy-to-show property of the $n$-dimensional standard Gaussian distribution is its anti-concentration when projected on any direction. This property becomes much less obvious once one is only promised that low-degree moments of the distribution are close to those of Gaussian, which is something we do show. We note that this step of our proof is similar in spirit to the work of \cite{KarmalkarKK19} that introduces a notion of low-degree certified anti-concentration and shows it for various distributions. Our proofs use extensively tools from polynomial approximation theory. 

Given these ideas, our tester-learner pair does the following. The tester estimates the low-degree moments of the distribution and compares them to the corresponding moments of the standard Gaussian. It follows then that halfspaces are well-approximated by low-degree polynomials with respect to this distribution.   The learning algorithm takes advantage of this by performing low-degree polynomial $L_1$ regression similar to the one used in \cite{kalai_agnostically_2005}. 

A technical complication, which we deal with, is that both our tester and learner work with a truncated version of the distribution. In other words, they discard the examples whose coordinates are too large. This guarantees to us that we can actually produce estimates for the moments of the truncated distribution (if distribution is not truncated, moments could even be infinite).

Note that our arguments use strongly the fact that we are working with halfspaces and not with some arbitrary function class that is well-approximated by low-degree polynomials under the Gaussian distribution. This is due to how we use the concentration and anti-concentration properties of the distribution. In a certain sense this is necessary, as shown by our intractability results for indicators of convex sets. 
Even though these functions are also well-approximated by low-degree polynomials  \cite{KlivansOS08}, for them a similar method based on estimating low-degree moments will provably not succeed. This underscores that designing tester-learner pairs can be subtle and does not generally follow by mere extension of already existing analyses of agnostic learning algorithms. 

\paragraph{Tester-learner pair for agnostically learning halfspaces under uniform distribution on $\{\pm 1\}^n$.}
We now discuss the techniques used to give our tester-learner pair for halfspaces under the uniform distribution on $\{\pm 1\}^n$. As we mentioned, the run-time we show here is $n^{\tilde{O}\left(1/\epsilon^4\right)}$ and this is concurrent work with \cite{ConcurrentWork22}, who use other techniques. See \Cref{section: tester-learner under cube} for complete details.

Our tester tests $\poly(1/\epsilon)$-wise independence of the input distribution with respect to the TV distance using 
\cite{ODonnellZ18,AlonAKMRX07,AlonGM03}. The learning algorithm uses the low-degree polynomial $L_1$ regression of  \cite{kalai_agnostically_2005}. 
To show that these two algorithms indeed form a valid tester-learner pair we show that every halfspace is well-approximated by a low-degree polynomial relative to any $\poly(1/\epsilon)$-wise independent distribution. 

Suppose for a halfspace $\sign(\vect{v} \cdot \vect{x}+\theta)$ it is the case that the norm of the vector $\vect{v}$ is well-distributed among all the coordinates. Then, by Berry-Esseen theorem, for $\vect{x}$ that is uniform over $\{\pm 1\}^n$ the inner product $\vect{v} \cdot \vect{x}$ is distributed similarly to a Gaussian. Roughly, we use this to argue that if $\vect{x}$ is merely $\poly(1/\epsilon)$-wise independent then $\vect{v} \cdot \vect{x}$ has low-degree moments close to those of a Gaussian. This allows us to use methods similar to the ones we use to give tester-learner pairs for halfspaces under the standard Gaussian distribution.

Finally, we handle halfspaces $\sign(\vect{v} \cdot \vect{x}+\theta)$ for whom the norm of the vector $\vect{v}$ is not well-spread across all the coordinates. We use the \emph{critical index} machinery of \cite{DiakonikolasGJSV09} to handle such halfspaces.

\paragraph{Intractability results.}

Finally, we discuss the techniques used to show that $2^{\Omega(n)}$ samples are required by (i) any tester-learner pair for learning indicator functions of convex sets under the standard Gaussian on $\R^n$ (ii) any tester-learner pair for learning monotone functions under the uniform distribution on $\{0,1\}^n$. See \Cref{section: hardness results} for complete details.

From technical standpoint, we find these lower bounds surprising: The mentioned standard agnostic learning algorithms in these settings rely on low-degree polynomial regression. This suggests that testing low-degree moments of the distribution (as we did for halfspaces) ought to lead to the development of a fast tester-learner pair. Yet, the lower bounds show that this can not be done.

We now roughly explain how we prove these lower bounds. Let us focus on the lower bound for tester-learner pairs for convex sets under standard Gaussian distribution (the lower bound for monotone functions is similar).
Take samples $\vect{z}_1,\cdots,\vect{z}_M$  from the standard Gaussian, and let $D$ be the uniform distribution on $\{\vect{z}_1,\cdots,\vect{z}_M\}$. The first idea is to show that the tester will have a hard time distinguishing $D$ from the standard Gaussian if it uses much fewer than $M$ samples\footnote{Out actual argument also takes into account that the tester sees labels and not only examples.}. The second idea is to show that (very likely over the choice of ${\vect{z}_1,\cdots,\vect{z}_M}$) one can obtain, by excluding only a small fraction of elements from $\{\vect{z}_1,\cdots,\vect{z}_M\}$, a subset $Q$ of them such that no point in $Q$ is in the convex hull of the other points in $Q$. Once we have such a set, we essentially\footnote{This is an oversimplification, as one still needs to figure out what to do with elements outside $Q$. We show that, for all these elements, we can either include them into or exclude them from the convex set in such a way as to reveal no information about which of the points in $Q$ were included in the convex set.} define our hard-to-learn convex set to be the convex hull of a random subset of $Q$, and this convex set will not contain any other elements of $Q$ because no member of $Q$ is in the convex hull of the rest. In this way, unless a learner has seen a large fraction of the elements in $Q$ already, it has no way of predicting whether a previously unseen element in $Q$ belongs to the random convex set. We note that our argument is somewhat similar to well-known arguments proving impossibility of approximation of the volume of a convex set via a deterministic algorithm \cite{BaranyF86,Elekes86}.

\subsection{Comments on the model.}
\label{subsec: more comments}



\paragraph{What about cross-validation?} In case of realizable  learning (i.e. you are promised there is no noise) a common approach to verifying success is via checking prediction error rate on fresh data and making sure it is not too high. Does this idea allow one to construct a tester $\mathcal{T}$ for the distributional assumption of some agnostic learner $\mathcal{A}$? Such tester would (i) run $\mathcal{A}$ to obtain a predictor $\hat{f}$ (ii) test the success rate of $\hat{f}$ on fresh example-label pairs (iii) accept or reject based on the success rate. 

As was mentioned in the discussion of our intractability results, there cannot be a general low-overhead method of transforming standard agnostic learning algorithms into tester-learner pairs, because of our intractability results. Therefore, in particular, there cannot be such a method based on cross-validation.

Intuitively, the reason is the following. Suppose you run the learning algorithm, setting the closeness parameter
$\epsilon$ to $0.01$, then check the success of the predictor on fresh data and find that the generalization error is close to $0.25$. This could potentially be consistent with the two following situations: (1) there is a function in the concept class with close to zero generalization error, but the learning algorithm gave a poor predictor due to a violation of the distributional assumption (2) the distributional assumption holds, but every function in the concept class has generalization error of at least $0.24$. The \emph{composability} criterion tells you that in case (1) you should reject, but the \emph{completeness} criterion tells you that in case (2) you should accept. Overall, there is no way to tell from generalization error alone which of the two situations you are in, so there is no way to know if you should accept or reject.

\paragraph{Label-aware vs label-oblivious testers.} We say the tester $\mathcal{T}$ is \emph{label-aware} if it makes use of the labels given to it (and not only the examples). Otherwise, we call it \emph{label-oblivious}. We feel that label-obliviousness contributes to the interpretability of the overall guarantee. However, this condition is not strictly necessary for verifying success. Due to these considerations, our impossibility results are against more general label-aware testers, while  the tester given in this paper is label-oblivious.

\subsection{Related work.}
\label{subsection: related work}
\paragraph{Agnostic learning under distributional assumptions using low-degree polynomial regression.}
Since the introduction of the agnostic learning model \cite{Haussler92,KearnsSS94} there has been an explosion of work in agnostic learning. Making assumptions on the distribution on examples has been ubiquitous in this line of work. So has been the use of  low-degree polynomial regression as one of the main tools. 
 Previous to the work of \cite{kalai_agnostically_2005}, there existed an extensive body of work on using low-degree polynomial regression for learning under distributional assumptions, including \cite{linial_constant_1989,aiello_learning_1991,furst_improved_1991,mansour_onlog_1992,bshouty_fourier_1995,klivans_learning_2002}. The work of \cite{kalai_agnostically_2005} building on \cite{KearnsSS94} proposed to use low-degree polynomial $L^1$ regression to obtain \emph{agnostic} learning algorithms for halfspaces under distribution assumptions, as well as extended these previously studied low-degree regression algorithms into the agnostic setting. Further work used low degree polynomial $L^1$ regression to obtain agnostic learning algorithms for many more problems, again under various distributional assumptions \cite{odonnell_learning_2006,blais_polynomial_2008,KlivansOS08,gopalan_learning_2010,kane_gaussian_2010,wimmer_agnostically_2010,harsha_invariance_2010,diakonikolas_bounding_2010,cheraghchi_submodular_2012,AwasthiBL14,dachman-soled_approximate_2014,feldman_tight_2015,feldman_agnostic_2015,blais_learning_2015,canonne_testing_2017,feldman_tight_2017,diakonikolas_agnostic_2021}. 

\paragraph{Learning halfspaces.} 

See the work of \cite{diakonikolas_agnostic_2021} and references therein, for a historical discussion about the problem of learning halfspaces, as well as some up-to-date references regarding some problems connected to the one studied here.

\paragraph{Polynomial approximation theory.} 
Polynomial approximation theory has been used extensively as a tool for studying halfspaces. Among other work, see \cite{kalai_agnostically_2005,DiakonikolasGJSV09,KlivansLS09,Daniely15,DiakonikolasKTZ20,diakonikolas_agnostic_2021}.

\paragraph{Other works in testing distributions.}
There is a large body of literature on finite sample guarantees
for property testing of distributions.  Algorithms developed within this framework 
are given samples of an input distribution and aim to distinguish the case in which
the distribution has a specified property, from the case in which the distribution
is far  (in a reasonable distance metric) 
from any distribution with that property.   Properties of interest include whether
the distribution is uniform, independent, monotone, has high entropy or is supported by
a large number of distinct elements.   
We mention a few specific results that are closest to the results in this work:  Let $p$ 
be a distribution on a discrete domain of size $M$.
For a ``known" distribution $q$ (where the algorithm knows the value of $q$ on every element
of the domain, and does not need samples from it -- e.g., when $q$ is the uniform distribution), 
distinguishing whether $p$ is  the same as $q$
from the case where $p$ is $\epsilon$-far (in $L_1$ norm) from $q$ requires 
$\Theta(\sqrt{M}/\epsilon^2)$
samples
\cite{ GoldreichR2000, BatuFRSW00, BatuFFKRW01, Paninski08, DiakonikolasGPP16, DiakonikolasGKP21}.
For a more
in depth discussion of the history and results in this area, see the monograph by Canonne \cite{ClementSurvey}.


\paragraph{Other models of trusting agnostic learners.}  
The work of Goldwasser, Rothblum, Shafer and Yehudayoff considers the question of how an untrusted
prover can convince a learner that a hypothesis is approximately correct, and show that significantly
less data is needed than that required for agnostic learning
\cite{GoldwasserRSY20}.

\section{Preliminaries.}

\subsection{Standard definitions.}


The definition of agnostic learning is as follows: 
\begin{defn}
An algorithm $\mathcal{A}$ is an \emph{agnostic $(\epsilon,\delta)$-learning algorithm}
for function class $\mathcal{F}$ relative to the distribution $D$,
if given access to i.i.d. example-label pairs $(x,y)$ distributed according
to $D_{\text{pairs}}$, with the marginal distribution on the examples equal to $D$, the algorithm $\mathcal{A}$ with probability at least $1-\delta$
outputs a circuit computing a function $\hat{f}$, such that 
\[
\Pr_{\parr{x,y}\in_{R}D_{\text{pairs}}}[y\neq\hat{f}(x)]\leq\min_{f\in\mathcal{F}}\parr{\Pr_{\parr{x,y}\in_{R}D_{\text{pairs}}}[f(x)\neq y]}+\epsilon.
\]
\end{defn}
The quantity $\Pr_{\parr{x,y}\in_{R}D_{\text{pairs}}}[f(x)\neq y]$ is often
called the \emph{generalization error} of $\hat{f}$ (a.k.a. \emph{out-of-sample
error }or \emph{risk}). 

The following is standard theorem about agnostic learning from $\ell_1$-approximation. The proof is implicit in \cite{kalai_agnostically_2005} and this theorem has been implicitly used in much subsequent work (see Subsection \ref{subsection: related work} for references).
Let $U$ be some domain we are working
over.
\begin{thm}
\label{theorem: agnostic learning from L1 approximation}
Let $\{g_{1},\cdots g_{N}\}$
be a collection of real-valued functions over $U$ that can be evaluated in time $T$. 
Then, for
every $\epsilon>0$, there is a learning algorithm $\mathcal{A}$ for which the following is true. 
Let $D$ be
any distribution over $U$ and let $\mathcal{F}$ be any class of
Boolean functions over $U$, such that every element of $\mathcal{F}$
is $\epsilon$-approximated in $L^1$ norm relative to the distribution
$D$ by some element of $\spn{g_{1},\cdots,g_{N}}$. Then, $\mathcal{A}$
agnostically $\parr{\epsilon,\delta}$-learns $\mathcal{F}$ relative
to $D$. The algorithm $\mathcal{A}$ uses $\tilde{O}\parr{\frac{N}{\epsilon^{2}}\log\parr{\frac{1}{\delta}}}$
samples and uses run-time polynomial in this number of samples and
$T$. 
\end{thm}

We will also need the definition of $k$-wise independent distributions:
\begin{defn}
A distribution of a random variable $\vect x$ over $\{\pm 1\}^n$ is called $k$-wise independent (a.k.a. $k$-wise uniform) if for any size-$k$ subset $S$ of $\{1,\cdots, n\}$ the distribution of $\{x_i: ~ i\in S\}$ is uniform over $\{\pm 1\}^n$.
\end{defn}

\subsection{New definition: testing distributional assumptions of a learning
algorithm.}
\label{sec: definition of tester/learner}
\begin{defn}
\label{def: tester-learner pair}
Let $\mathcal{A}$ be an agnostic $\parr{\epsilon,\delta_{1}}$-learning algorithm
for function class $\mathcal{F}$ relative to the distribution $D$.
We say that an algorithm $\mathcal{T}$ is a 
\emph{
tester
for the distributional assumption} of $\mathcal{A}$ if 
\begin{enumerate}
\item \emph{(Composability) }Suppose a distribution $D_{\text{pairs}}$ on example-label pairs is such that, given access to i.i.d. labeled examples from it, the algorithm
$\mathcal{T}$ outputs ``Yes'' with probability at least $1/4$.
Then $\mathcal{A}$, given access to i.i.d. labeled examples from the same distribution $D_{\text{pairs}}$, will with probability at least $1-\delta_{1}$
output a circuit computing a function $\hat{f}$, such that 
\[
\Pr_{\parr{x,y}\in_{R}D_{\text{pairs}}}[y\neq\hat{f}(x)]\leq\min_{f\in\mathcal{F}}\parr{\Pr_{\parr{x,y}\in_{R}D_{\text{pairs}}}[f(x)\neq y]}+\epsilon.
\]
\item \emph{(Completeness) }
Suppose $D_{\text{pairs}}$ is such that the marginal distribution on examples equals to $D$. Then, given i.i.d. example-label pairs from $D_{\text{pairs}}$, tester $\mathcal{T}$ outputs ``Yes'' with probability
at least $3/4$.
\end{enumerate}
If this definition is satisfied, then we say that $(\mathcal{A}$,$\mathcal{T}$) form a tester-learner pair.
\end{defn}
Constants $1/4$ and $3/4$ in the definition above can without loss of generality be replaced with any other pair of constants $1-\delta_2$ and $1-\delta_3$ with $\delta_2\in(0,1)$ and $\delta_3 \in (\delta_2,1)$. See Appendix \ref{appendix: improving error probabilitites} for the proof via a standard repetition argument.





\section{\label{sec: beginning with algorithms.}An efficient tester-learner pair for learning halfspaces.}
\label{section: tester-learner under Gaussian}
We now describe our tester-learner pair for learning halfspaces under the Gaussian distribution. 
Roughly, the testing algorithm checks that the low-degree moments of the distribution on examples are close enough to those of the standard Gaussian distribution.
The learning algorithm uses a low-degree polynomial regression. As explained earlier, both of the algorithms ignore examples whose absolute value is too high, which allows them to obtain accurate estimates of distribution moments. \\

\noindent
{\bf Tester-learner pair for learning halfspaces:}
\begin{itemize}
\item Let $C_{1},\cdots,C_{4}$ be a collection of constants to be tuned
appropriately. Define $d:=2\left\lfloor \frac{1}{2\epsilon^{4}}\ln^{3}\parr{\frac{1}{\epsilon}}\right\rfloor $,
$\Delta:=\left\lfloor \frac{1}{\epsilon^{4}}\ln^{4}\parr{\frac{1}{\epsilon}}\right\rfloor $,
$t:=C_{1}\Delta\ln\Delta\sqrt{\log n}+\sqrt{2\ln\parr{\frac{C_{2}n}{\epsilon}}}$,
$N_{1}:=\left\lceil n^{C_{3}d}\right\rceil $ and $N_{2}:=\left\lceil t^{2\Delta}n^{C_{4}\Delta}\right\rceil $.
\item \textbf{Learning algorithm $\mathcal{A}$.} Given access to i.i.d.
labeled samples $(\vect x,y)\in\R^{n}\times\left\{ \pm1\right\} $
from an unknown distribution:
\begin{enumerate}
\item Obtain $N_{1}$ many labeled samples $(\vect x_{i},y_{i})$.
\item Discard all the samples $(\vect x_{i},y_{i})$ for which the absolute
value of some coordinate $\abs{\parr{\vect x_{i}}_{j}}$ is greater
than $t$.
\item Run the algorithm of Theorem \ref{theorem: agnostic learning from L1 approximation}
on the remaining samples, with accuracy parameter $\frac{\epsilon}{10}$,
allowed failure probability $\frac{1}{20},$ and taking the set of $\left\{ g_{i}\right\} $
to be the set of monomials of degree at most $d$, i.e. the set $\left\{ \prod_{j=1}^{n}x_{j}^{\alpha_{j}}:~\sum_j \alpha_j\leq d\right\} $
. This gives us a circuit computing
predictor $\hat{f}$. Form a new predictor $\hat{f}'$
that given $\vect x$ outputs (i) $\hat{f}(\vect x)$ if for all $j\in\pars n$, the value of
$\abs{\parr{\vect x_{i}}_{j}}$ is at most $t$. (ii) 1 if\footnote{This one's arbitrary. Can also output 0 in this case.}
for some $j\in\pars n$, the value of $\abs{\parr{\vect x_{i}}_{j}}$ exceeds $t$.
\end{enumerate}
\item \textbf{Testing algorithm $\mathcal{T}$}. Given access to i.i.d.
labeled samples $\vect x\in\R^{n}$ from an unknown distribution:
\begin{enumerate}
\item For each $j\in[n]$:
\begin{enumerate}
\item Estimate $\Pr\pars{\abs{x_{j}}>t}$ up to additive $\frac{\epsilon}{30n}$
with error probability $\frac{1}{100n}$.
\item If the estimate is at least $\frac{\epsilon}{10n}$, output \textbf{No
}and terminate.
\end{enumerate}
\item Draw $N_{2}$ fresh samples $\left\{ \vect x_{i}\right\} $, and discard
the ones for which the absolute value of some coordinate $\abs{\parr{\vect x_{i}}_{j}}$
is greater than $t$. 
\item For every monomial $\prod_{j=1}^{n}x_{j}^{\alpha_{j}}$ of degree
at most $\Delta$, compute its empirical expectation w.r.t. the samples
$\left\{ \vect x_{i}\right\} $. If for any of them resulting value
is not within $\frac{1}{2n^{\Delta}}$ of $\E_{\vect z\sim\N(0,I_{n\times n})}\pars{\prod_{j=1}^{n}x_{j}^{\alpha_{j}}}=\prod_{j=1}^{n}\parr{\parr{\alpha_{j}-1}!!\cdot\indicator_{\alpha_{j}\text{ is even}}}$,
output \textbf{No }and terminate.
\item Output \textbf{Yes.}
\end{enumerate}
\end{itemize}

The following theorem shows that the above algorithms indeed satisfy the criteria
for a tester-learner pair for learning halfspaces under the Gaussian distribution:
\begin{thm}
[\textbf{Tester-learner pair for learning halfspaces under Gaussian distribution}]
\label{thm: main testing learning theorem for LTFs} Suppose the values
$C_{1},\cdots,C_{4}$ present in algorithms $\mathcal{A}$ and $\mathcal{T}$
are chosen to be sufficiently large absolute constants, also assume
$n$ and $\frac{1}{\epsilon}$ are larger than some sufficiently large
absolute constant. Then, the algorithm $\mathcal{A}$ is an agnostic
$(O(\epsilon),0.1)$-learner for the function class of linear threshold
functions over $\R^{n}$ under distribution $\N(0,I_{n\times n})$
and the algorithm $\mathcal{T}$ is an 
assumption tester
for $\mathcal{A}$. The algorithms $\mathcal{A}$ and $\mathcal{T}$
both require only $n^{\tilde{O}\parr{\frac{1}{\epsilon^{4}}}}$ samples
and run-time. Additionally, The tester $\mathcal{T}$ is label-oblivious.   
\end{thm}

Note that an $(O(\epsilon),0.1)$-learner can be made an agnostic
$(\epsilon,\delta_{1})$-learner for any fixed constant $\delta_{1}$
and still require only $n^{\tilde{O}\parr{\frac{1}{\epsilon^{4}}}}$
samples and run-time via a standard repeat-and-check argument. The
tester $\mathcal{T}$ for the original learner will remain an 
assumption
tester for the new learner. 

The proof of correctness of the above tester-learner pair for halfspaces makes
use of the following lemmas, which will be proved in Section \ref{sec: proof of two main lemmas}.
Lemma \ref{main lemma: low degree moment lemma for distributions} states that as long
as the low-degree moments of a distribution are similar to the corresponding moments of the Gaussian
distribution, then the distribution is concentrated and anti-concentrated when projected onto any direction.
Lemma \ref{main lemma: low degree approximation lemma for halfspaces.} states that as long as  distribution $D$ satisfies the ``nice" properties of concentration and anti-concentration, then any halfspace can be approximated by a low-degree polynomial with respect to distribution $D$. Taken together, these lemmas will be used to show
that for any distribution $D$, if the moments of $D$ look similar to moments of
the Gaussian distribution, then halfspaces are well-approximated by low degree polynomials under $D$.

\begin{lem}[\textbf{Low degree moment lemma for distributions}.]
\label{main lemma: low degree moment lemma for distributions} Suppose
$D$ is a distribution over $\R^{n}$ and $\Delta$ is an even positive
integer, such that for every monomial $\prod_{i=1}^{n}x_{i}^{\alpha_{i}}$
of degree at most $\Delta$ we have 
\[
\abs{\E_{\vect x\sim D}\pars{\prod_{i=1}^{n}x_{i}^{\alpha_{i}}}-\E_{\vect x\sim\N(0,I_{n\times n})}\pars{\prod_{i=1}^{n}x_{i}^{\alpha_{i}}}}\leq\frac{1}{n^{\Delta}}.
\]
Further, assume that $\Delta\geq\frac{1}{\epsilon^{4}}\ln^{4}\parr{\frac{1}{\epsilon}}$.
Then, for every unit vector $\vect v$, the random variable $\vect v\cdot\vect x$
(with $\vect x\in_{R}D$) has the following properties
\begin{itemize}
\item \textbf{Concentration:} For any even positive integer $d\leq\Delta$, we have
$\parr{\E_{\vect x\in_{R}D}\pars{\abs{\vect v\cdot\vect x}^{d}}}^{1/d}\leq2\sqrt{d}$.
\item \textbf{Anti-concentration: }
for any real $y$, we have 
\[
\Pr_{\vect x\in_{R}D}\pars{\vect v\cdot\vect x\in[y,y+\epsilon]}\leq O\parr{\epsilon}.
\]
\end{itemize}
\end{lem}

\begin{lem}[\textbf{Low degree approximation lemma for halfspaces.}]
\label{main lemma: low degree approximation lemma for halfspaces.}
Suppose $D$ is a distribution on $\R^{n}$ and $\vect v\in\R^{n}$
is a unit vector, such that for some positive real parameters $\alpha,\gamma,\epsilon$
and a positive integer parameter $d_{0}$ we have
\begin{itemize}
\item \textbf{Anti-concentration: }for any real $y$, we have $\Pr_{\vect x\in_{R}D}\pars{\vect v\cdot\vect x\in[y,y+\epsilon]}\leq\alpha$,
\item \textbf{Concentration:} $\parr{\E_{\vect x\in_{R}D}\pars{\abs{\vect v\cdot\vect x}^{d_{0}}}}^{1/d_{0}}\leq\beta$,
for some $\beta\geq1$.
\end{itemize}
Also assume $d_{0}>\frac{5\beta}{\epsilon^{2}}$ and that $\epsilon$
is smaller than some sufficiently small absolute constant. Then, for
every $\theta\in\R$ and there is a polynomial $P(x)$ of degree at
most $\frac{2\beta}{\epsilon^{2}}+1$ such that
\[
E_{\vect x\in_{R}D}\pars{\abs{P(\vect v\cdot\vect x)-\sign(\vect v\cdot\vect x-\theta)}}=O\parr{\alpha+\epsilon+\frac{\parr{8\beta}^{\frac{2\beta}{\epsilon^{2}}+1}}{2^{d_{0}}}}.
\]
Each coefficient of the polynomial $P$ has magnitude of at most $O\parr{
2^{\frac{4\beta}{\epsilon^2}}
}$.
\end{lem}

\section{Technical preliminaries.}

\subsection{Polynomial approximation theory.}
We will need some standard facts about Chebychev polynomials and approximation of functions using them. See, for example, the text \cite{trefethen2019approximation} for comprehensive treatment of this topic. First, we define Chebychev polynomials and present relevant facts about them. On the interval $[-1,1]$ the $k$-th Chebychev polynomial can be defined
as\footnote{One needs to check that $\cos(k\alpha)$ is indeed a polynomial in
$\cos\alpha$, which follows by writing $\cos(k\alpha)=\frac{e^{ik\alpha}+e^{-ik\alpha}}{2}=\frac{1}{2}\parr{\parr{\cos\alpha+i\sin\alpha}^{k}+\parr{\cos\alpha-i\sin\alpha}^{k}}$,
expanding, observing that terms involving odd powers of $\sin\alpha$
cancel out, and using the identity $\sin^{2}\alpha=1-\cos^{2}\alpha$.} 
$
T_{k}(x):=\cos\parr{k\arccos(x)}.
$

For any $k\geq 0$, the polynomial $T_{k}(x)$ maps $[-1,1]$ to $[-1,1]$
(this follows immediately from the definition). Also, it is known that the Chebyshev polynomials satisfy
a recurrence relation 
\[
T_{k+1}(x)=2xT_{k}(x)-T_{k-1}(x),
\]
with the first two polynomials being $T_{0}(x)=1$ and $T_{1}(x)=x$.


To present a standard theorem from text \cite{trefethen2019approximation} about approximating functions  with Chebyshev polynomials, we will need the standard notions of Lipschitz continuity and of bounded variation functions.
A function $f$ is said to be Lipschitz continuous on $[-1,1]$ if
there is some $C$ so for any $x,y\in[-1,1]$ we have that $\abs{f(x)-f(y)}\leq C\abs{x-y}.$
For a differentiable function $f:[-w,w]\righ\R$, the \emph{total
variation of $f$} is the $L_{1}$ norm of it's derivative, i.e.
\[
\int_{-w}^{w}\abs{\frac{df(x)}{dx}}\d x.
\]
If $f$ has a single discontinuity at some point $a$ and is differentiable everywhere else, then the total variation of $f$ is defined as the sum of the following three terms (i) $\int_{-w}^{a}\abs{\frac{df(x)}{dx}}\d x$, (ii) the magnitude of the discontinuity at $a$ and (iii) $\int_{a}^{w}\abs{\frac{df(x)}{dx}}\d x$. Analogously, the definition extends to functions that are differentiable outside of finitely many discontinuities\footnote{It is also standard to consider more general functions, but we will not need that.}. We say ``$f$ is of bounded variation $V$'' if the total variation
of $f$ is at most $V$. 

We are now ready to state the following theorem about approximating functions using Chebyshev polynomials:
\begin{thm}[Consequence of Theorem 7.2 in the text \cite{trefethen2019approximation} (see also Theorem 3.1 on page 19 in the text \cite{trefethen2019approximation})]
\label{thm:approximating differentiable function}Let $f$ be Lipschitz
continuous on $[-1,1]$ and suppose the derivative $f^{'}$ is of
bounded variation $V$. Define for $k\geq0$ 
\[
a_{k}:=\frac{1+\indicator_{k>0}}{\pi}\int_{-1}^{1}\frac{f(x)T_{k}(x)}{\sqrt{1-x^{2}}}\d x.
\] Then, for any $d\geq0$ we have
\[
\max_{x\in[-1,1]}\abs{f(x)-\sum_{k=0}^{d}a_{k}T_{k}(x)}=O\parr{\frac{V}{d}}.
\]
\end{thm}
The partial sums $\sum_{k=0}^{d}a_{k}T_{k}$ are called Chebyshev
projections.

\section{Proving the two main lemmas (\ref{main lemma: low degree moment lemma for distributions},\ref{main lemma: low degree approximation lemma for halfspaces.}) via
polynomial approximation theory.}
\label{sec: proof of two main lemmas}

\subsection{Propositions useful for proving both main lemmas.}
Here we will present proposition that will be useful for proving both Lemma 
\ref{main lemma: low degree moment lemma for distributions} and \ref{main lemma: low degree approximation lemma for halfspaces.}. We start with an observation that bounds the magnitude of the coefficients of Chebyshev polynomials.
\begin{observation}
\label{obs: chebychev projection has small coefficients} 
Let $f:\R\rightarrow[-1,1]$
be a Lipschitz continuous function. Let $d\geq1$ be an integer, let $w \geq 1$ be a real number, and
let $f_{d}(x):=\sum_{k=0}^{d}a_{k}T_{k}(\frac{x}{w})$, where $a_{k}:=\frac{1+\indicator_{k>0}}{\pi}\int_{-1}^{1}\frac{f(wy)T_{k}(y)}{\sqrt{1-y^{2}}}\d y.$
Then, the largest coefficient from among all the monomials of $f_{d}(x)$
has value of at most $O\parr{d3^{d}}$.
\end{observation}

\begin{proof}
See Appendix \ref{appendix: proof of obs: chebychev projection has small coefficients}.
\end{proof}
Proving both lemmas, we will be approximating certain functions using Chebyshev polynomials re-scaled to the window $[-w,w]$. The following proposition lets us bound the error between function $f$ and its low-degree polynomial approximation, contributed by the region $(-\infty, w) \cup (w, +\infty)$.
\begin{prop}
\label{prop:approximation outside window, assuming moment bound}
Let $f$ be a Lipschitz continuous function $\R\righ[-1,1]$. Let
$d\geq1$ be an integer and $w\geq1$ be real-valued, and let $f_{d}(x):=\sum_{k=0}^{d}a_{k}T_{k}(\frac{x}{w}),$
where $a_{k}:=\frac{1+\indicator_{k>0}}{\pi}\int_{-1}^{1}\frac{f(wy)T_{k}(y)}{\sqrt{1-y^{2}}}\d y.$
Then, for any distribution $D$, it is the case that 
\[
\E_{x\in_{R}D}\pars{\abs{f(x)-f_{d}(x)}\indicator_{\abs x>w}}\leq O\parr{4^{d}\E_{x\in_{R}D}\pars{\abs x^{d}\indicator_{\abs x>w}}}.
\]
\end{prop}

\begin{proof}
See Appendix \ref{appendix: proof of prop:approximation outside window, assuming moment bound}.
 
\end{proof}
The following proposition, in turn, allows us to bound the expression we encounter in Proposition \ref{prop:approximation outside window, assuming moment bound} in terms of a bound on the moments of distribution $D$.
\begin{prop}
\label{prop: moment tail bound for bounded moment distributions}Let
$D$ be a distribution on $\R$ and $d_{0}\in\Z^{>0}$ such that 
\[
\parr{\E_{x\in_{R}D}\pars{\abs x^{d_{0}}}}^{1/d_{0}}\leq\beta.
\]
Then, for any $k\in\Z\cap\left[0,d_{0}/2\right]$ and $w\in\R^{+}$
we have
\[
\E_{x\in_{R}D}\pars{\abs x^{k}\indicator_{\abs{x}>w}}\leq2w^{k}\parr{\frac{\beta}{w}}^{d_{0}}
\]
\end{prop}

\begin{proof}
See Appendix \ref{appendix: proof of prop: moment tail bound for bounded moment distributions}.
\end{proof}

\subsection{Proof of low degree moment lemma for distributions(Lemma \ref{main lemma: low degree moment lemma for distributions}).}
\label{subsection: proof of low degree moment lemma for distributions}

Let us recall the setting of Lemma \ref{main lemma: low degree moment lemma for distributions}. $D$ is a distribution over $\R^{n}$ and
$\Delta$ is an even positive integer, such that for every monomial
$\prod_{i=1}^{n}x_{i}^{\alpha_{i}}$ of degree at most $\Delta$ we
have 
\[
\abs{\E_{\vect x\sim D}\pars{\prod_{i=1}^{n}x_{i}^{\alpha_{i}}}-\E_{\vect x\sim\N(0,I_{n\times n})}\pars{\prod_{i=1}^{n}x_{i}^{\alpha_{i}}}}\leq\frac{1}{n^{\Delta}}.
\]
Further, we have that $\Delta\geq\frac{1}{\epsilon^{4}}\ln^{4}\parr{\frac{1}{\epsilon}}$.
Then, we would like to show that for every unit vector $\vect v$,
the random variable $\vect v\cdot\vect x$ (with $\vect x\in_{R}D$)
has the following properties
\begin{itemize}
\item \textbf{Concentration:} For any even integer $d\leq\Delta$, we have
$\parr{\E_{\vect x\in_{R}D}\pars{\abs{\vect v\cdot\vect x}^{d}}}^{1/d}\leq2\sqrt{d}$.
\item \textbf{Anti-concentration: }for any real-valued
parameter $w\geq1$, for any real $y$, we have 
\[
\Pr_{\vect x\in_{R}D}\pars{\vect v\cdot\vect x\in[y,y+\epsilon]}\leq O\parr{\epsilon}.
\]
\end{itemize}
We start with the following observation saying that if moments of a distribution $D$ are similar to standard Gaussian, then the expectation of a polynomial of a form $(\vect v \cdot \vect x)^d$ for $D$ is similar to the same expectation under standard Gaussian.
\begin{observation}
\label{obs: similar to gaussian means projection is similar to gaussian}
Suppose $D$ is a distribution over $\R^{n}$ and $\Delta$ is a positive
integer, such that for every monomial $\prod_{i=1}^{n}x_{i}^{\alpha_{i}}$
of degree at most $\Delta$ we have $\abs{\E_{\vect x\sim D}\pars{\prod_{i=1}^{n}x_{i}^{\alpha_{i}}}-\E_{\vect x\sim\N(0,1)}\pars{\prod_{i=1}^{n}x_{i}^{\alpha_{i}}}}\leq\frac{1}{n^{\Delta}}$.
Then, for any unit vector $\vect v$ and integer $d\leq\Delta$ we
have 
\[
\abs{\E_{\vect x\in_{R}D}\pars{\parr{\vect v\cdot\vect x}^{d}}-\E_{\vect{x} \in_{R}\N(0,I_{n \times n})}\pars{\parr{\vect{v}\cdot \vect{x}}^{d}}}\leq\frac{n^{d}}{n^{\Delta}}.
\]
\end{observation}

\begin{proof}
See Appendix \ref{appendix: proof of obs: similar to gaussian means projection is similar to gaussian}.
\end{proof}
Let us now show the concentration property. 
Let $d$ be even. Recall that for even $d$ we have $\E_{\vect x\sim\N(0,I_{n\times n})}\pars{\parr{\vect v\cdot\vect x}^{d}}=\E_{x'\sim\N(0,1)}\pars{\parr{x'}^{d}}=(d-1)!!\leq d^{d/2}$.
This, together with Observation \ref{obs: similar to gaussian means projection is similar to gaussian}
implies
\[
\parr{\E_{\vect x\sim D}\pars{\parr{\vect v\cdot\vect x}^{d}}}^{1/d}\leq\parr{d^{d/2}+\frac{n^{d}}{n^{\Delta}}}^{1/d}=\sqrt{d}\parr{1+\frac{n^{d-\Delta}}{d^{d/2}}}^{1/d}\leq2\sqrt{d},
\]
which is the \emph{concentration} property we wanted to show.

Now, we proceed to the \emph{anti-concentration} property. Recall that for this property we need to bound $\Pr_{\vect x\in_{R}D}\pars{\vect v\cdot\vect x\in[y,y+\epsilon]}$. To this end, we first approximate $\indicator_{z\in\pars{y,y+\epsilon}}$ using the following function
\begin{equation}
g(z):=\begin{cases}
0 & \text{if \ensuremath{z\leq y-\epsilon}},\\
\frac{z-\left(y-\epsilon\right)}{\epsilon} & \text{if }\ensuremath{z\in\pars{y-\epsilon,y}},\\
1 & \text{if }\ensuremath{z\in\pars{y,y+\epsilon}},\\
\frac{\left(y+2\epsilon\right)-z}{\epsilon} & \text{if }\ensuremath{z\in\pars{y+\epsilon,y+2\epsilon},}\\
0 & \text{if \ensuremath{z\geq y+2\epsilon}}.
\end{cases}\label{eq: trapezoid defined}
\end{equation}
The key properties of $g$ are (i) $g(z)\geq\indicator_{z\in\pars{y,y+\epsilon}}$
(ii) $g(z)\in[0,1]$ (ii) $g(z)$ is Lipschitz continuous (iii) the
derivative $g'(z)$ is of bounded variation of $\frac{4}{\epsilon}$ (because the function has four discontinuities, each of magnitude $1/\epsilon$ and it stays constant in-between the discontinuities). 

Let $w\geq1$ be real-valued and $d$ be an integer in $\pars{1,\Delta/2}$,
to be chosen later and let $g_{d}(x):=\sum_{k=0}^{d}a_{k}T_{k}(\frac{x}{w}),$
where $a_{k}:=\frac{1+\indicator_{k>0}}{\pi}\int_{-1}^{1}\frac{g(wy)T_{k}(y)}{\sqrt{1-y^{2}}}\d y.$
Observation \ref{obs: N is anti-conc} and propositions \ref{prop: cutting deg} and \ref{prop: moving from D to N} are stated and proven below, and we use them no to get the following bound:
\begin{multline*}
\Pr_{\vect x\in_{R}D}\pars{\vect v\cdot\vect x\in[y,y+\epsilon]}\leq\E_{\vect x\in_{R}D}\pars{g(\vect v\cdot\vect x)}\leq
\\
\overbrace{
\E_{\vect x\in_{R}\N(0,I_{n\times n})}\pars{g(\vect v\cdot\vect x)}
}^{\text{\ensuremath{O(\epsilon)} by Observation \ref{obs: N is anti-conc} }}
+ 
\overbrace{
\E_{ \vect{x}\in_{R}\N(0,I_{n\times n})}\pars{\abs{g_{d}(\vect v \cdot \vect x)-g(\vect v \cdot \vect x)}}
}^{\text{O\ensuremath{\parr{4^{d}w^{d}\parr{\frac{2\sqrt{\Delta}}{w}}^{\Delta}+\frac{w}{\epsilon d}}} by Proposition \ref{prop: cutting deg} }}
+
\overbrace{
\abs{\E_{\vect x\in_{R}D}\pars{g_{d}(\vect v\cdot\vect x)}-
\E_{\vect x\in_{R}\N(0,I_{n\times n})}\pars{g_{d}(\vect v\cdot\vect x)}
}
}^{\text{O\ensuremath{\parr{4^{d}\frac{n^{d}}{n^{\Delta}}}} by Proposition \ref{prop: moving from D to N}}}
+\\+
\underbrace{\E_{\vect x\in_{R}D}\pars{\abs{g(\vect v\cdot\vect x)-g_{d}(\vect v\cdot\vect x)}}}_{\text{O\ensuremath{\parr{4^{d}w^{d}\parr{\frac{2\sqrt{\Delta}}{w}}^{\Delta}+\frac{w}{\epsilon d}}} by Proposition \ref{prop: cutting deg}}}=O\parr{\epsilon+4^{d}w^{d}\parr{\frac{2\sqrt{\Delta}}{w}}^{\Delta}+\frac{w}{\epsilon d}+4^{d}\frac{n^{d}}{n^{\Delta}}}.
\end{multline*}
Now, recall we assumed without loss of generality that $\Delta=\frac{1}{\epsilon^{4}}\ln^{4}\parr{\frac{1}{\epsilon}}$,
so taking\footnote{We also check that (taking $\epsilon$ small enough) $d$ is indeed in $\pars{1,\Delta/2}$, as was required earlier.} $d=\frac{1}{10\epsilon^{4}}\ln^{2}\parr{\frac{1}{\epsilon}}$
and $w=\frac{10}{\epsilon^{2}}\ln^{2}\parr{\frac{1}{\epsilon}}$
we get
\begin{multline*}
\Pr_{\vect x\in_{R}D}\pars{\vect v\cdot\vect x\in[y,y+\epsilon]}\leq O\parr{\epsilon+4^{d}w^{d}\parr{\frac{2\sqrt{\Delta}}{w}}^{\Delta}+\frac{w}{\epsilon d}+4^{d}\frac{n^{d}}{n^{\Delta}}}=\\
O\parr{\epsilon
+
\parr{\frac{40}{\epsilon^{2}}\ln^{2}\parr{\frac{1}{\epsilon}}}^{\frac{1}{10\epsilon^{4}}\ln^{2}\parr{\frac{1}{\epsilon}}}\parr{\frac{1}{5}}^{\frac{1}{\epsilon^{4}}\ln^{4}\parr{\frac{1}{\epsilon}}}
+
4^{\frac{1}{10\epsilon^{4}}\ln^{2}\parr{\frac{1}{\epsilon}}}\frac{1}{n^{
\frac{1}{\epsilon^{4}}
\ln^{4}\parr{\frac{1}{\epsilon}}
-\frac{1}{10\epsilon^4}\ln^2\parr{\frac{1}{\epsilon}}
}}}=O(\epsilon).
\end{multline*}
The only thing left to do is to prove the observations referenced
above.
\begin{observation}
\label{obs: N is anti-conc}For the function $g$ as defined in Equation
\ref{eq: trapezoid defined}, we have
\[
\E_{\vect x\in_{R}\N(0,I_{n\times n})}\pars{g(\vect v\cdot\vect x)}=O(\epsilon)
\]
\begin{proof}
The function $g$ has a range of $[0,1]$ and is supported on $\pars{y-\epsilon,y+3\epsilon}$. Also, $\vect v\cdot\vect x$ is distributed as a standard one-dimensional Gaussian.
Therefore, the probability that $\vect v\cdot\vect x$ lands in $\pars{y-\epsilon,y+3\epsilon}$,
is at most $O(\epsilon)$, which finishes the proof.
\end{proof}
\end{observation}

\begin{prop}
\label{prop: cutting deg}Suppose $D$ is a distribution over $\R^{n}$
and $\Delta$ is a positive integer, such that for every monomial
$\prod_{i=1}^{n}x_{i}^{\alpha_{i}}$ of degree at most $\Delta$ we
have $\abs{\E_{\vect x\sim D}\pars{\prod_{i=1}^{n}x_{i}^{\alpha_{i}}}-\E_{\vect x\sim\N(0,I_{n\times n})}\pars{\prod_{i=1}^{n}x_{i}^{\alpha_{i}}}}\leq\frac{1}{n^{\Delta}}$.
Let $d$ be an integer in $\pars{1,\Delta/2}$, let $w\geq 1$ be a real-valued parameter, and suppose $g:[-w,w]\righ[-1,1]$
is a Lipschitz function whose derivative $g'$ is of Bounded variation
$V$, and let $g_{d}(x):=\sum_{k=0}^{d}a_{k}T_{k}(\frac{x}{w})$, where
$a_{k}:=\frac{1+\indicator_{k>0}}{\pi}\int_{-1}^{1}\frac{g(wy)T_{k}(y)}{\sqrt{1-y^{2}}}\d y.$
Then, it is the case that 
\[
\E_{\vect x\in_{R}D}\pars{\abs{g(\vect v\cdot\vect x)-g_{d}(\vect v\cdot\vect x)}}\leq O\parr{4^{d}w^{d}\parr{\frac{2\sqrt{\Delta}}{w}}^{\Delta}+\frac{Vw}{d}}.
\]
\end{prop}

\begin{proof}
Proposition \ref{prop:approximation outside window, assuming moment bound}
and Proposition \ref{prop: moment tail bound for bounded moment distributions}
imply
\[
\E_{\vect x\sim D}\pars{\abs{g(\vect v\cdot\vect x)-g_{d}(\vect v\cdot\vect x)}\indicator_{\abs{\vect v\cdot\vect x}>w}}\leq O\parr{4^{d}\E_{\vect x\in_{R}D}\pars{\abs{\vect v\cdot\vect x}^{d}\indicator_{\abs{\vect v\cdot\vect x}>w}}}\leq4^{d}w^{d}\parr{\frac{2\sqrt{\Delta}}{w}}^{\Delta}\frac{\Delta}{\Delta-d}.
\]
To use Theorem \ref{thm:approximating differentiable function}, we need to bound the total variation of the function $\dv{g(wz)}{z}=wg'(wz)$. Inspecting the definition of total variation, we see that $g'(wz)$ has the same total variation as $g'(z)$, which is at most $V$. Therefore, the total variation of $\dv{g(wz)}{z}$ is at most $Vw$.
Thus, we have by Theorem \ref{thm:approximating differentiable function}
that 
\[
\E_{\vect x\sim D}\pars{\abs{g(\vect v\cdot\vect x)-g_{d}(\vect v\cdot\vect x)}\indicator_{\abs{\vect v\cdot\vect x}\leq w}}\leq\max_{z\in[-w,w]}\abs{g(z)-g_{d}(z)}\leq O\parr{\frac{Vw}{d}}.
\]
Summing the two equations above and recalling that $d\leq\Delta/2$,
our proposition follows.
\end{proof}
\begin{prop}
\label{prop: moving from D to N}Suppose $D$ is a distribution over
$\R^{n}$ and $\Delta$ is a positive integer, such that for every
monomial $\prod_{i=1}^{n}x_{i}^{\alpha_{i}}$ of degree at most $\Delta$
we have $\abs{\E_{\vect x\sim D}\pars{\prod_{i=1}^{n}x_{i}^{\alpha_{i}}}-\E_{\vect x\sim\N(0,1)}\pars{\prod_{i=1}^{n}x_{i}^{\alpha_{i}}}}\leq\frac{1}{n^{\Delta}}$.
Let $g:\R\righ[-1,1]$ be a Lipschitz continuous function, and $g_{d}(x):=\sum_{k=0}^{d}a_{k}T_{k}(\frac{x}{w}),$where
$a_{k}:=\frac{1+\indicator_{k>0}}{\pi}\int_{-1}^{1}\frac{f(wy)T_{k}(y)}{\sqrt{1-y^{2}}}\d y.$
Then
\[
\abs{\E_{\vect x\in_{R}D}\pars{g_{d}(\vect v\cdot\vect x)}-
\E_{ \vect{x}\in_{R}\N(0,I_{n\times n})}\pars{g_{d}(\vect v \cdot \vect x)}
}=O\parr{4^{d}\frac{n^{d}}{n^{\Delta}}}.
\]
\begin{proof}
Observation \ref{obs: chebychev projection has small coefficients}
implies that $g_{d}(z)$ is a degree $d$ polynomial, whose largest
coefficient is at most $d3^{d}$.
Using Observation
\ref{obs: similar to gaussian means projection is similar to gaussian} for each of these monomials, we get
\[
\abs{\E_{\vect x\in_{R}D}\pars{g_{d}\parr{\vect v\cdot\vect x}}-\E_{ \vect{x}\in_{R}\N(0,I_{n\times n})}\pars{g_{d}(\vect v \cdot \vect x)}}\leq O\parr{d^{2}3^{d}}\frac{n^{d}}{n^{\Delta}}=O\parr{4^{d}\frac{n^{d}}{n^{\Delta}}}.
\]
\end{proof}
\end{prop}

\subsection{Proof of low degree approximation lemma for halfspaces (Lemma \ref{main lemma: low degree approximation lemma for halfspaces.}).}

Let us recall what we need to show to prove Lemma \ref{main lemma: low degree approximation lemma for halfspaces.}.
Without loss of generality, we assume we are in one dimension. $D$
is a distribution on $\R$, such that for some positive real parameters
$\alpha,\gamma,\epsilon$ and a positive integer parameter $d_{0}$
we have
\begin{itemize}
\item \textbf{Anti-concentration:} for any real
$y$, we have $\Pr_{x\in_{R}D}\pars{x\in[y,y+\epsilon]}\leq\alpha$,
\item \textbf{Concentration:} $\parr{\E_{x\in_{R}D}\pars{\abs x^{d_{0}}}}^{1/d_{0}}\leq\beta$,
for some $\beta\geq1$.
\end{itemize}
Also we have $d_{0}>\frac{5\beta}{\epsilon^{2}}$ and that $\epsilon$
is smaller than some sufficiently small absolute constant. Then, for
every $\theta\in\R$ we would like to show there is a polynomial $P(x)$
of degree at most $\frac{2\beta}{\epsilon^{2}}+1$ such that
\[
E_{x\in_{R}D}\pars{\abs{P(x)-\sign(x-\theta)}}=O\parr{\alpha+\epsilon+\frac{\parr{8\beta}^{\frac{2\beta}{\epsilon^{2}}+1}}{2^{d_{0}}}}.
\]

Let $w>1$ and $d\in Z^{+}$ be parameters, values of which will be set later. We will approximate the sign function with a polynomial in the following two steps:
\begin{itemize}
\item Approximate $\sign(x-\theta)$ by a continuous function
\[
f(x):=\begin{cases}
1 & \text{if }\frac{x-\theta}{\epsilon}>1,\\
-1 & \text{if }\frac{x-\theta}{\epsilon}<-1,\\
\frac{x-\theta}{\epsilon} & \text{otherwise.}
\end{cases}
\]
\item For a parameter $d$, approximate $f(x)$ by 
\[
f_{d}(x):=\sum_{k=0}^{d}a_{k}T_{k}(\frac{x}{w}),
\]
where 
\[
a_{k}:=\frac{1+\indicator_{k>0}}{\pi}\int_{-1}^{1}\frac{f(wy)T_{k}(y)}{\sqrt{1-y^{2}}}\d y.
\]
\end{itemize}
First, we observe that $f$ is a good approximator for $\sign(x-\theta)$ with respect to $D$.
\begin{prop}
\label{prop: anticoncentration part} If $D$ is a distribution over
$\R$ such that for every $x_{0}\in\R$ we have $\Pr_{x\in_{R}D}\pars{x\in\pars{x_{0},x_{0}+\epsilon}}\leq\alpha$,
then (with $f(x)$ defined as above) we have 
\[
\E_{x\in_{R}D}\pars{\abs{f(x)-\sign(x-\theta)}}\leq2\alpha.
\]
\end{prop}

\begin{proof}
The two functions differ only on $\pars{\theta-\epsilon,\theta+\epsilon}$,
with the absolute value of difference being at most $1$. Since the
distribution $D$ cannot have probability mass more than $2\alpha$
in this interval, the proposition follows.
\end{proof}
Secondly, we show that $f_d$ is a good approximator to $f$ with respect to $D$, within the region $[-w,w]$.
\begin{prop}
\label{prop: approximation within window} For any distribution $D$,
we have 
\[
\E_{x\in_{R}D}\pars{\abs{f(x)-f_{d}(x)}\indicator_{\abs x\leq w}}\leq O\parr{\frac{w}{\epsilon d}}
\]
\end{prop}

\begin{proof}
Using Theorem \ref{thm:approximating differentiable function} we
have 
\[
\E_{x\in_{R}D}\pars{\abs{f(x)-f_{d}(x)}\indicator_{\abs x\leq w}}\leq\max_{x\in[-w,w]}\abs{f(x)-f_{d}(x)}=\max_{y\in[-1,1]}\abs{f\parr{wy}-f_{d}\parr{wy}}=O\parr{\frac{w}{\epsilon d}}.
\]
\end{proof}

Now, we put all the relevant propositions together to show the lemma.
Using Propositions \ref{prop:approximation outside window, assuming moment bound} and \ref{prop: moment tail bound for bounded moment distributions}, we see that if we have $d\in\Z\cap\pars{1,d_{0}/2}$ then
\[
\E_{x\in_{R}D}\pars{\abs{f(x)-f_{d}(x)}\indicator_{\abs x>w}}\leq O\parr{4^{d}\E_{x\in_{R}D}\pars{\abs x^{d}\indicator_{\abs x>w}}}
\leq
O\parr{4^{d}2w^{d}\parr{\frac{\beta}{w}}^{d_{0}}}
\]
Together with Proposition \ref{prop: approximation within window}, this implies that
\[
\E_{x\in_{R}D}\pars{\abs{f(x)-f_{d}(x)}}\leq
O\parr{4^{d}2w^{d}\parr{\frac{\beta}{w}}^{d_{0}}}
+
O\parr{\frac{w}{\epsilon d}}
\]
This, in turn, together with Proposition \ref{prop: anticoncentration part} implies that 
\[
E_{x\in_{R}D}\pars{\abs{f_{d}(x)-\sign(x-\theta)}}=O\parr{\alpha+\frac{w}{\epsilon d}+4^{d}w^{d}\parr{\frac{\beta}{w}}^{d_{0}}}.
\]
Taking\footnote{Recall that to do all this we needed that $d$ is in $[1,d_0/2]$. Recall that by an assumption of the lemma we are proving we have $d_{0}>\frac{5\beta}{\epsilon^{2}}$ and $\beta \geq 1$. Therefore,
for $\epsilon$ smaller than some sufficiently small absolute constant
we indeed have $\left\lceil \frac{2\beta}{\epsilon^{2}}\right\rceil \in\pars{1,d_{0}/2}$.} $w=2\beta$ and $d=\left\lceil \frac{2\beta}{\epsilon^{2}}\right\rceil $
we get 
\[
E_{x\in_{R}D}\pars{\abs{f_{d}(x)-\sign(x-\theta)}}=O\parr{\alpha+\epsilon+\frac{\parr{8\beta}^{\left\lceil \frac{2\beta}{\epsilon^{2}}\right\rceil }}{2^{d_{0}}}}=O\parr{\alpha+\epsilon+\frac{\parr{8\beta}^{\frac{2\beta}{\epsilon^{2}}+1}}{2^{d_{0}}}}.
\]

Finally, we note that by \cref{obs: chebychev projection has small coefficients} we have that each coefficient of the polynomial $f_d$ has a magnitude of at most $O(d3^d)=O\parr{
4^{\frac{2\beta}{\epsilon^2}}
}
$.
This completes the proof of the low degree approximation lemma for
halfspaces (Lemma \ref{main lemma: low degree approximation lemma for halfspaces.}).



\section{Proof of Main Theorem via two main lemmas.}
\label{section: proof of main theorem}
\subsection{Truncated Gaussian has moments similar to Gaussian}
Recall that our tester truncates the samples and checks that low-degree moments are close to the corresponding moments of a Gaussian. If the distribution is indeed Gaussian, the following proposition shows that this truncation step does not distort the moments too much.
\begin{prop}
\label{prop: truncation does not change moments much for Gaussian}Let
$\prod_{i=1}^{n}x_{i}^{\alpha_{i}}$ be a monomial of degree at most
$\Delta$ and $t$ a real number in $\left[2\sqrt{\Delta}+1,+\infty\right)$.
Then we have 
\[
\abs{\E_{\vect x\sim\N(0,I_{n\times n})}\pars{\prod_{i=1}^{n}x_{i}^{\alpha_{i}}\conditional\forall i:\,\abs{x_{i}}\leq t}-\E_{\vect x\sim\N(0,I_{n\times n})}\pars{\prod_{i=1}^{n}x_{i}^{\alpha_{i}}}}\leq O\parr{2^{\Delta}\Delta^{\frac{\Delta+2}{2}}t^{\Delta}e^{-\frac{t^{2}}{2}}}.
\]
\end{prop}

\begin{proof}
If any of the $\alpha_{i}$ is odd, both expectations are zero, so
the proposition follows trivially. So, without loss of generality,
assume that each~$\alpha_{i}$ is even. Also, without loss of generality,
we can also assume that $n\geq\Delta$ and the $\alpha_{i}$ can be
non-zero only for $i\in\left\{ 1,\cdots\Delta\right\} $. We
prove the following observation separately:
\begin{observation}
\label{obs: moment tail bound for Gaussian} For $d\geq0$, if $w\geq2\sqrt{d}+1$,
then it is the case that 
\[
\E_{x\in_{R}\N(0,1)}\pars{x^{d}\indicator_{\abs x>w}}\leq O\parr{w^{d}e^{-\frac{w^{2}}{2}}}
\]
\end{observation}

\begin{proof}
We have $\int_{w}^{+\infty}x^{d}e^{-\frac{x^{2}}{2}}\d x=\int_{w}^{+\infty}e^{-\parr{\frac{x^{2}}{2}-d\ln x}}\d x$.
For $x\geq w$, we have
\[
\frac{\d}{\d x}\parr{\frac{x^{2}}{2}-d\ln x}=x-\frac{d}{x}\geq w-\frac{d}{w},
\]
which means
\[
\parr{\frac{x^{2}}{2}-d\ln x}\geq\frac{w^{2}}{2}-d\ln\parr w+\parr{w-\frac{d}{w}}\parr{x-w}.
\]
Thus, we have 
\[
\int_{w}^{+\infty}x^{d}e^{-\frac{x^{2}}{2}}\d x\leq e^{-\frac{w^{2}}{2}+d\ln\parr w}\int_{w}^{+\infty}e^{-\parr{w-\frac{d}{w}}\parr{x-w}}\d x=\frac{w^{d}e^{-\frac{w^{2}}{2}}}{\parr{w-\frac{d}{w}}}\leq O\parr{w^{d}e^{-\frac{w^{2}}{2}}}.
\]
\end{proof}
Now, we consider the one-dimensional case of our proposition.
\begin{observation}
\label{obs: in one dimension trunction does not change moments much}Let
$d$ be a positive integer and $t$ be a real number, such that $t$
is in $\left[2\sqrt{d}+1,+\infty\right)$, then

\[
\abs{\E_{x\sim\N(0,1)}\pars{x^{d}\conditional\abs x\leq t}-\E_{x\sim\N(0,1)}\pars{x^{d}}}\leq O\parr{t^{d}e^{-\frac{t^{2}}{2}}}.
\]
\begin{proof}
If $d$ is odd, both expectations are zero, so without loss of generality
assume that $d$ is even. We have 
\begin{multline*}
\abs{\E_{x\sim\N(0,1)}\pars{x^{d}\conditional\abs x\leq t}-\E_{x\sim\N(0,1)}\pars{x^{d}}}=\\
\abs{\frac{\E_{x\sim\N(0,1)}\pars{x^{d}\indicator_{\abs x\leq t}}}{\Pr_{x\sim\N(0,1)}\pars{\abs x\leq t}}-\E_{x\sim\N(0,1)}\pars{x^{d}\indicator_{\abs x\leq t}}-\E_{x\sim\N(0,1)}\pars{x^{d}\indicator_{\abs x>t}}}=\\
\abs{\frac{\E_{x\sim\N(0,1)}\pars{x^{d}\indicator_{\abs x\leq t}}\Pr_{x\sim\N(0,1)}\pars{\abs x>t}}{\Pr_{x\sim\N(0,1)}\pars{\abs x\leq t}}-\E_{x\sim\N(0,1)}\pars{x^{d}\indicator_{\abs x>t}}}\leq\\
\overbrace{\leq O\parr{\abs{\E_{x\sim\N(0,1)}\pars{x^{d}\indicator_{\abs x\leq t}}\Pr_{x\sim\N(0,1)}\pars{\abs x>t}}}+\abs{\E_{x\sim\N(0,1)}\pars{x^{d}\indicator_{\abs x>t}}}}^{\text{Using (i) triangle inequality (ii) \ensuremath{\Pr_{x\sim\N(0,1)}\pars{\abs x\leq t}}\ensuremath{\ensuremath{\geq\Omega}(1)} because \ensuremath{t\geq1}. }}\leq\\
\underbrace{\leq O\parr{d^{d/2}\abs{\Pr_{x\sim\N(0,1)}\pars{\abs x>t}}}+\abs{\E_{x\sim\N(0,1)}\pars{x^{d}\indicator_{\abs x>t}}}}_{\text{Since \ensuremath{\E_{x\sim\N(0,1)}\pars{x^{d}}=(d-1)!!.}}}\underbrace{\leq O\parr{d^{d/2}e^{-\frac{t^{2}}{2}}+t^{d}e^{-\frac{t^{2}}{2}}}}_{\text{Using Observation \ref{obs: moment tail bound for Gaussian}}.}\underbrace{\leq O\parr{t^{d}e^{-\frac{t^{2}}{2}}}}_{\text{Because \ensuremath{t>\sqrt{d}.}}}.
\end{multline*}
\end{proof}
\end{observation}

We proceed to reduce the high-dimensional case to the one-dimensional version we have just shown. 
\begin{multline*}
\abs{
\E_{\vect x\sim\N(0,I_{n\times n})}\pars{\prod_{i=1}^{n}x_{i}^{\alpha_{i}}}
-
\E_{\vect x\sim\N(0,I_{n\times n})}\pars{\prod_{i=1}^{n}x_{i}^{\alpha_{i}}\conditional\forall i:\,\abs{x_{i}}\leq t}
}
=\\
\abs{\prod_{i=1}^{\Delta}\E_{x\sim\N(0,1)}\pars{x^{\alpha_{i}}}-\prod_{i=1}^{\Delta}\E_{x\sim\N(0,1)}\pars{x^{\alpha_{i}}\conditional\abs x\leq t}}\leq\\
\sum_{j=1}^{\Delta}\abs{\prod_{i=1}^{j-1}\E_{x\sim\N(0,1)}\pars{x^{\alpha_{i}}\conditional\abs x\leq t}\prod_{i=j}^{\Delta}\E_{x\sim\N(0,1)}\pars{x^{\alpha_{i}}}-\prod_{i=1}^{j}\E_{x\sim\N(0,1)}\pars{x^{\alpha_{i}}\conditional\abs x\leq t}\prod_{i=j+1}^{\Delta}\E_{x\sim\N(0,1)}\pars{x^{\alpha_{i}}}}=\\
\sum_{j=1}^{\Delta}\abs{\prod_{i=1}^{j-1}\E_{x\sim\N(0,1)}\pars{x^{\alpha_{i}}\conditional\abs x\leq t}\prod_{i=j+1}^{\Delta}\E_{x\sim\N(0,1)}\pars{x^{\alpha_{i}}}\parr{\E_{x\sim\N(0,1)}\pars{x^{\alpha_{j}}}-\E_{x\sim\N(0,1)}\pars{x^{\alpha_{j}}\conditional\abs x\leq t}}}.
\end{multline*}
Now, we have $\E_{x\sim\N(0,1)}\pars{x^{\alpha_{i}}\conditional\abs x\leq t}=\frac{\E_{x\sim\N(0,1)}\pars{x^{\alpha_{i}}\indicator_{\abs x\leq t}}}{\Pr_{x\sim\N(0,1)}\pars{\abs x\leq t}}\leq2\E_{x\sim\N(0,1)}\pars{x^{\alpha_{i}}}$,
since $\Pr_{x\sim\N(0,1)}\pars{\abs x\leq t}\ensuremath{\geq0.5}$
for $t\geq1$. Using this, Observation \ref{obs: in one dimension trunction does not change moments much}
and the fact that $\E_{x\sim\N(0,1)}\pars{x^{\alpha_{i}}}=(\alpha_{j}-1)!!\leq\alpha_{j}^{\alpha_{j}/2}$
with the inequality above, we have 
\begin{multline*}
\abs{\E_{\vect x\sim\N(0,I_{n\times n})}\pars{\prod_{i=1}^{n}x_{i}^{\alpha_{i}}\conditional\forall i:\,\abs{x_{i}}\leq t}-\E_{\vect x\sim\N(0,I_{n\times n})}\pars{\prod_{i=1}^{n}x_{i}^{\alpha_{i}}}}\leq\\
2^{\Delta}\prod_{i=1}^{\Delta}\E_{x\sim\N(0,1)}\pars{x^{\alpha_{i}}}\sum_{j=1}^{\Delta}\abs{\parr{\E_{x\sim\N(0,1)}\pars{x^{\alpha_{j}}}-\E_{x\sim\N(0,1)}\pars{x^{\alpha_{j}}\conditional\abs x\leq t}}}\leq\\
O\parr{2^{\Delta}\prod_{j=1}^{\Delta}\alpha_{j}^{\alpha_{j}/2}\parr{\sum_{j=1}^{\Delta}t^{\alpha_{j}}e^{-\frac{t^{2}}{2}}}}\leq O\parr{2^{\Delta}\Delta^{\Delta/2}\parr{\Delta\cdot t^{\Delta}e^{-\frac{t^{2}}{2}}}}=O\parr{2^{\Delta}\Delta^{\frac{\Delta+2}{2}}t^{\Delta}e^{-\frac{t^{2}}{2}}}
\end{multline*}
This completes the proof of Proposition \ref{prop: truncation does not change moments much for Gaussian}.
\end{proof}

\subsection{Finishing the proof of Theorem \ref{thm: main testing learning theorem for LTFs}.}

In this subsection we finish the proof of Theorem \ref{thm: main testing learning theorem for LTFs}, using the low degree moment lemma for distributions (Lemma
\ref{main lemma: low degree moment lemma for distributions}) and the low degree approximation lemma for halfspaces (Lemma
\ref{main lemma: low degree approximation lemma for halfspaces.}). The main thing left to do is to address issues relating to truncation of samples in the learning and testing algorithms.

We now restate the theorem. We are given that the values
$C_{1},\cdots,C_{4}$ present in algorithms $\mathcal{A}$ and $\mathcal{T}$
(in the beginning of Section \vref{sec: beginning with algorithms.})
are chosen to be sufficiently large absolute constants, and also $n$
and $\frac{1}{\epsilon}$ are larger than some sufficiently large
absolute constant. Then, we need to show that the algorithm $\mathcal{A}$
is an agnostic $(O(\epsilon),0.1)$-learner for the function class
of linear threshold functions over $\R^{n}$ under distribution $\N(0,I_{n\times n})$
and the algorithm $\mathcal{T}$ is an assumption tester
for $\mathcal{A}$. We also need to show that $\mathcal{A}$ and $\mathcal{T}$
require only $n^{\tilde{O}\parr{\frac{1}{\epsilon^{4}}}}$ samples
and run-time.

Bounds on the run-time and sample complexity of our algorithms follow
directly from our choice of parameters.
\begin{itemize}
\item The learner $\mathcal{A}$ draws $N_{1}:=n^{\tilde{O}\parr{\frac{1}{\epsilon^{4}}}}$
samples, then performs a computation running in time polynomial in
(i) $N_{1}$ (ii) the number of monomials $\prod_{j=1}^{n}x_{j}^{\alpha_{j}}$
of degree at most $d$, which is $O\parr{n^{d}}$ (this includes the
run-time consumed by the algorithm of Theorem \ref{theorem: agnostic learning from L1 approximation}).
Overall, the learner $\mathcal{A}$ uses $n^{\tilde{O}\parr{\frac{1}{\epsilon^{4}}}}$
samples and run-time.
\item The tester $\mathcal{T}$ first performs estimations of values
$\Pr\pars{\abs{x_{j}}>t}$ up to additive $\frac{\epsilon}{30n}$
with error probability $\frac{1}{100n}$, which in total require $\textit{poly}\parr{\frac{n}{\epsilon}}$
samples and run-time. Then, the tester $\mathcal{T}$ obtains
$N_{2}:=\left\lceil t^{\Delta}n^{C_{4}\Delta}\right\rceil $ samples (where
$\Delta:=\left\lfloor \frac{1}{\epsilon^{4}}\ln^{4}\parr{\frac{1}{\epsilon}}\right\rfloor $
and $t:=C_{1}\Delta\ln\Delta\sqrt{\log n}+\sqrt{2\ln\parr{\frac{C_{2}n}{\epsilon}}}$)
 and performs a polynomial time computation with them. We see
that $t=O\parr{\textit{poly}\parr{\frac{1}{\epsilon},n}}$ and therefore
$N_{2}=n^{\tilde{O}\parr{\frac{1}{\epsilon^{4}}}}$. Finally, the
tester $\mathcal{T}$ runs a computation running in time
polynomial in (i) $N_{2}$ and (ii) the number of monomials $\prod_{j=1}^{n}x_{j}^{\alpha_{j}}$
of degree at most $\Delta$, which is $O\parr{n^{\Delta}}$. Overall,
we get that the run-time and sample complexity of $\mathcal{T}$
is $n^{\tilde{O}\parr{\frac{1}{\epsilon^{4}}}}$.
\end{itemize}
\begin{prop}
The following proposition uses the low degree approximation lemma for halfspaces (Lemma
\ref{main lemma: low degree approximation lemma for halfspaces.}) to argue that, under certain regularity conditions on the distribution $D$, the learning algorithm satisfies the agnostic learning guarantee.
\label{prop: tester acceptance implies learner works}Suppose the
$C_{1},\cdots,C_{4}$ are chosen to be sufficiently large absolute
constants, $n$ and $\frac{1}{\epsilon}$ are larger than some sufficiently
large absolute constant. Suppose $D$ is a distribution over $\R^{n}$
such that it the following properties hold 
\end{prop}

\begin{itemize}
\item \textbf{\textit{Good tail: }}\textit{We have $\Pr_{\vect x\in_{R}D}\pars{\exists i\in[n]:\:\abs{x_{i}}>t}\leq\frac{\epsilon}{5}$.}
\item \textbf{Concentration along} \textbf{any direction for truncated distribution:
}For any unit vector $\vect v$ we have 
\[
\parr{\E_{\vect x\in_{R}D}\pars{\abs{\vect v\cdot\vect x}^{d}\conditional\forall i\in[n]:\:\abs{x_{i}}\leq t}}^{1/d}\leq2\sqrt{d}.
\]
\item \textbf{A}\textbf{\textit{nti-concentration along }}\textbf{any direction
for truncated distribution}\textbf{\textit{: }}For any unit vector
$\vect v$ and\textbf{\textit{ }}for any real\textit{ $y$, }we have\textit{
\[
\Pr_{\vect x\in_{R}D}\pars{\vect v\cdot\vect x\in[y,y+\epsilon]\conditional\forall i\in[n]:\:\abs{x_{i}}\leq t}\leq O\left(\epsilon\right).
\]
}
\end{itemize}
Then, the algorithm $\mathcal{A}$ is an agnostic $(O\left(\epsilon\right),0.1)$-learner
for the function class of linear threshold functions over $\R^{n}$
under distribution $D$ with failure probability at most $\frac{1}{20}$.
\begin{proof}
Let $D_{\text{truncated}}$ be the distribution of $\vect x$ drawn from $D$ conditioned on $\abs{x_i}\leq t$ for all $i$. 
We see that the premises of this proposition imply that the distribution $D_{\text{truncated}}$ satisfies the premises of the low degree approximation lemma for halfspaces(Lemma
\ref{main lemma: low degree approximation lemma for halfspaces.}) with parameters $d_0=d$, $\alpha=O(\epsilon)$ and $\beta=2\sqrt{d}$. Taking $\epsilon$ smaller than some absolute constant ensures that the condition
$d>\frac{5\beta}{\epsilon^{2}}=\frac{10\sqrt{d}}{\epsilon^{2}}$ is also satisfied.

The low degree approximation lemma for halfspaces(Lemma
\ref{main lemma: low degree approximation lemma for halfspaces.}) then allows us to conclude that for every $\theta\in\R$ and for any $w\geq1$ there is
a polynomial $P(x)$ of degree at most $d$ such that
\[
E_{\vect x\in_{R}D_{\text{truncated}}}\pars{\abs{\sign(\vect v\cdot\vect x-\theta)-P(\vect v\cdot\vect x)}
}=O\parr{\epsilon+\frac{\parr{16\sqrt{d}}^{\frac{4\sqrt{d}}{\epsilon^{2}}+1}}{2^{d}}}.
\]
 Recalling that $d:=2\left\lfloor \frac{1}{2\epsilon^{4}}\ln^{3}\parr{\frac{1}{\epsilon}}\right\rfloor $
so we get that 
\[
E_{\vect x\in_{R}D_{\text{truncated}}}\pars{\abs{\sign(\vect v\cdot\vect x-\theta)-P(\vect v\cdot\vect x)}
}=O\parr{\epsilon+\frac{\parr{O\left(\frac{1}{\epsilon^{2}}\ln^{1.5}\parr{\frac{1}{\epsilon}}\right)}^{O\left(\frac{1}{\epsilon^{4}}\ln^{1.5}\parr{\frac{1}{\epsilon}}\right)}}{2^{\Omega\left(\frac{1}{\epsilon^{4}}\ln^{3}\parr{\frac{1}{\epsilon}}\right)}}}.
\]
For $\epsilon$ smaller than some sufficiently small absolute constant,
the above is $O(\epsilon)$.

Thus, we have that for any linear threshold function $\sign\parr{\vect v\cdot\vect x-\theta}$
there is a degree $d$ multivariate polynomial $Q$ for which 
\[
\E_{\vect x\in_{R}D_{\text{truncated}}}\pars{\abs{\sign\parr{\vect v\cdot\vect x-\theta}-Q(\vect x)}}\leq O(\epsilon)
\]
In other words, under $D_{\text{truncated}}$, any linear threshold function $\sign\parr{\vect v\cdot\vect x-\theta}$ is $O(\epsilon)$-approximated in $L^1$ by something in the span of set of monomials of degree at most $d$, i.e. the set $\left\{ \prod_{j=1}^{n}x_{j}^{\alpha_{j}}:~\sum_j \alpha_j\leq d\right\} $.
Now, Theorem \ref{theorem: agnostic learning from L1 approximation}.
tells us that with probability at least $1-\frac{1}{20}$ the predictor
$\widehat{f}$ given in step 3 has an error of at most $O(\epsilon)$
more than $\sign(\vect v\cdot\vect x-\theta)$ for samples $\vect x\in_{R}D_{\text{truncated}}$. Overall, recalling the definition of $D_{\text{truncated}}$ we
have 
\begin{multline*}
\Pr_{\vect x,y\in_{R}D_{\text{pairs}}}\pars{\widehat{f}'(\vect x)\neq y}\leq\Pr_{\vect x\in_{R}D}\pars{\exists i\in[n]:\:\abs{x_{i}}>t}+\Pr_{\vect x,y\in_{R}D_{\text{pairs}}}\pars{\widehat{f}(\vect x)\neq y\conditional\forall i\in[n]:\:\abs{x_{i}}\leq t}\leq\\
\Pr_{\vect x,y\in_{R}D_{\text{pairs}}}\pars{\sign(\vect v\cdot\vect x-\theta)\neq y\conditional\forall i\in[n]:\:\abs{x_{i}}\leq t}+O\parr{\epsilon},
\end{multline*}
which completes the proof.
\end{proof}
Now, the following proposition, using low degree moment lemma for distributions (Lemma
\ref{main lemma: low degree moment lemma for distributions}), tells us that the tester we use (1) is likely accept if the Gaussian assumption indeed holds (2) is likely to reject if the regularity conditions for Proposition \ref{prop: tester acceptance implies learner works} do not hold.
\begin{prop}
\label{prop: tester completeness + tester rejection condition}Suppose
the $C_{1},\cdots,C_{4}$ are chosen to be sufficiently large absolute
constants, $n$ and $\frac{1}{\epsilon}$ are larger than some sufficiently
large absolute constant. Then, there is some absolute constant $B$,
so the tester $\mathcal{T}$ has the following properties:
\begin{enumerate}
\item If $\mathcal{T}$ is given samples from $\N(0,I_{n\times n})$, it
outputs \textbf{Yes }with probability at least $0.9$.
\item 
The tester $\mathcal{T}$ rejects with probability greater than $0.9$ any $D$  for which at least one of the following holds: 
\begin{enumerate}
\item \textbf{Bad tail: }We have $\Pr_{\vect x\in_{R}D}\pars{\exists i\in[n]:\:\abs{x_{i}}>t}>\frac{\epsilon}{5}$.
\item \textbf{Failure of concentration along} \textbf{some direction for
truncated distribution: }there is a unit vector $\vect v$ such that
\[
\parr{\E_{\vect x\in_{R}D}\pars{\abs{\vect v\cdot\vect x}^{d}\conditional\forall i\in[n]:\:\abs{x_{i}}\leq t}}^{1/d}>2\sqrt{d}.
\]
\item \textbf{Failure of anti-concentration along some direction for truncated
distribution: }there is a unit vector $\vect v$ and real $y$, for
which
\[
\Pr_{\vect x\in_{R}D}\pars{\vect v\cdot\vect x\in[y,y+\epsilon]\conditional\forall i\in[n]:\:\abs{x_{i}}\leq t}>B\epsilon.
\]
\end{enumerate}

\end{enumerate}
\end{prop}

\begin{proof}
First, assume that $\mathcal{T}$ is getting samples from $\N(0,I_{n\times n})$
and let us prove that $\mathcal{T}$outputs \textbf{\textit{Yes }}\textit{with
probability at least $0.9$. }

Since $t\geq1$, by we have\footnote{Proof: $\int_{t}^{+\infty}e^{-\frac{x^{2}}{2}}\d x\leq e^{-\frac{t^{2}}{2}}\int_{t}^{+\infty}e^{-\frac{t\left(x-t\right)}{2}}\d x\leq\frac{2e^{-\frac{t^{2}}{2}}}{t}\leq O\parr{e^{-\frac{t^{2}}{2}}}.$}
$\Pr_{z\in\N(0,1)}\pars{\abs z>t}\leq O\parr{e^{-\frac{t^{2}}{2}}}$.
As $t\geq\sqrt{2\ln\parr{\frac{C_{2}n}{\epsilon}}}$, taking $C_{2}$
large enough we get $\Pr_{z\in\N(0,1)}\pars{\abs z>t}\leq\frac{\epsilon}{30n}$.
Therefore, $\N(0,I_{n\times n})$ passes step 1 of tester $\mathcal{T}$
with probability at least $1-\frac{1}{100}.$

Also, $\Pr_{z\in\N(0,1)}\pars{\abs z>t}\leq\frac{\epsilon}{30n}$
implies that $\Pr_{\vect x\in\N(0,I_{n\times n})}\pars{\forall i\in[n]:\:\abs{x_{i}}\leq t}\geq1-\frac{\epsilon}{30}$.
Together with a very loose application of the Hoeffding bound, we
see that for sufficiently large $C_{4}$ with probability at least
$1-\frac{1}{100}$ only at most half of the samples are discarded
in the step 2 of $\mathcal{T}$. We henceforth assume this indeed
was the case. The remaining samples themselves are i.i.d. and distributed
according to $\N(0,I_{n\times n})$ conditioned on all coordinates
being in $\pars{-t,t}$.

Since all remaining samples have the size of their coordinates bounded
by $t$, the value of a given monomial $\prod_{j=1}^{n}x_{j}^{\alpha_{j}}$
of degree at most $\Delta$ evaluated on any of them is in $\pars{-t^{\Delta},t^{\Delta}}.$
Therefore, the Hoeffding bound implies that for sufficiently large
$C_{4}$ with probability at least $1-\frac{1}{100n^{\Delta}}$ the
empirical average of $\prod_{j=1}^{n}x_{j}^{\alpha_{j}}$ on the (at
least $\frac{N_{2}}{2}$ many) remaining samples is within $\frac{1}{10n^{\Delta}}$
of 
\[
\E_{\vect x\sim\N(0,I_{n\times n})}\pars{\prod_{i=1}^{n}x_{i}^{\alpha_{i}}\conditional\forall i:\,\abs{x_{i}}\leq t}.
\]
 For sufficiently large $C_{1}$, we verify the premise of Proposition
\ref{prop: truncation does not change moments much for Gaussian}
that $t\in\left[2\sqrt{\Delta}+1,+\infty\right)$ and therefore have
\[
\abs{\E_{\vect x\sim\N(0,I_{n\times n})}\pars{\prod_{i=1}^{n}x_{i}^{\alpha_{i}}\conditional\forall i:\,\abs{x_{i}}\leq t}-\E_{\vect x\sim\N(0,I_{n\times n})}\pars{\prod_{i=1}^{n}x_{i}^{\alpha_{i}}}}\leq O\parr{2^{\Delta}\Delta^{\frac{\Delta+2}{2}}t^{\Delta}e^{-\frac{t^{2}}{2}}}.
\]
Now, we have $\frac{d}{dt}\parr{\Delta\log t-\frac{t^{2}}{2}}=\frac{\Delta}{t}-t$
which is negative when $t>\sqrt{\Delta}$. As $t\geq C_{1}\Delta\left(\ln\Delta\sqrt{\log n}\right)>\sqrt{\Delta}$,
we have 
\[
t^{\Delta}e^{-\frac{t^{2}}{2}}\leq\parr{C_{1}\Delta\left(\ln\Delta\sqrt{\log n}\right)}^{\Delta}\exp\parr{-\frac{\parr{C_{1}\Delta\left(\ln\Delta\sqrt{\log n}\right)}^{2}}{2}},
\]
which together with the preceding inequality implies
\begin{multline*}
\abs{\E_{\vect x\sim\N(0,I_{n\times n})}\pars{\prod_{i=1}^{n}x_{i}^{\alpha_{i}}\conditional\forall i:\,\abs{x_{i}}\leq t}-\E_{\vect x\sim\N(0,I_{n\times n})}\pars{\prod_{i=1}^{n}x_{i}^{\alpha_{i}}}}\leq\\
O\parr{2^{\Delta}\Delta^{\frac{\Delta+2}{2}}\parr{C_{1}\Delta\left(\ln\Delta\sqrt{\log n}\right)}^{\Delta}\exp\parr{-\frac{\parr{C_{1}\Delta\left(\ln\Delta\sqrt{\log n}\right)}^{2}}{2}}}
\end{multline*}
for sufficiently large $C_{1}$ the above is less than $\frac{1}{10n^{\Delta}}$.
Therefore, in the whole, we have that the empirical average of $\prod_{j=1}^{n}x_{j}^{\alpha_{j}}$
in step 3 of $\mathcal{T}$ is with probability at least $1-\frac{1}{100n^{\Delta}}$
within $\frac{1}{10n^{\Delta}}$ of $\E_{\vect x\sim\N(0,I_{n\times n})}\pars{\prod_{i=1}^{n}x_{i}^{\alpha_{i}}}$.
Taking a union bound over all monomials $\prod_{j=1}^{n}x_{j}^{\alpha_{j}}$
of degree at most $\Delta$, we see that the step 3 of the tester
$\mathcal{T}$ also passes with probability at least $1-\frac{1}{100}$
when it is run on $\N(0,I_{n\times n})$.

Overall, we conclude that the probability $\mathcal{T}$ outputs \textbf{No
}when given samples from $\N(0,I_{n\times n})$ is at most $\frac{3}{100}<0.1$
as promised.

Now, we shall show that $\mathcal{T}$will likely output \textbf{No
}if any of the conditions given in the proposition hold. 

If Condition (a) holds, we have $\Pr_{\vect x\in_{R}D}\pars{\exists i\in[n]:\:\abs{x_{i}}>t}>\frac{\epsilon}{5}$,
then there is some coordinate $i$ for which $\Pr_{\vect x\in_{R}D}\pars{\abs{x_{i}}>t}>\frac{\epsilon}{5n}$.
This coordinate will lead to $\mathcal{T}$ outputting \textbf{No
}in step $1$ with probability at least $1-\frac{1}{100}$. 

Now, suppose condition (a) doesn't hold so we $\Pr_{\vect x\in_{R}D}\pars{\exists i\in[n]:\:\abs{x_{i}}>t}\leq\frac{\epsilon}{5}$
but condition (b) or (c) does hold. We would like to show that $\mathcal{T}$ will
still likely output \textbf{No.} With a very loose application of
the Hoeffding bound, for sufficiently large $C_{4}$ with probability
at least $1-\frac{1}{100}$ only at most half of the samples are discarded
in the step 2 of $\mathcal{T}$, which we also assume henceforth.
Using the Hoeffding bound again, we see that for sufficiently large
$C_{4}$ with probability at least $1-\frac{1}{100}$ the empirical
expectation of all monomials $\prod_{j=1}^{n}x_{j}^{\alpha_{j}}$
of degree at most $\Delta$ is within $\frac{1}{10n^{\Delta}}$ of
\[
\E_{\vect x\in_{R}D}\pars{\prod_{i=1}^{n}x_{i}^{\alpha_{i}}\conditional\forall i\in[n]:\:\abs{x_{i}}\leq t}.
\]

In other words, with probability at least $1-\frac{1}{100}$ the tester
$\mathcal{T}$ will output \textbf{No }in step 3, unless we have for
all monomials $\prod_{j=1}^{n}x_{j}^{\alpha_{j}}$ that 
\[
\abs{\E_{\vect x\in_{R}D}\pars{\prod_{i=1}^{n}x_{i}^{\alpha_{i}}\conditional\forall i\in[n]:\:\abs{x_{i}}\leq t}-\E_{\vect z\sim\N(0,I_{n\times n})}\pars{\prod_{j=1}^{n}x_{j}^{\alpha_{j}}}}\leq\frac{1}{2n^{\Delta}}+\frac{1}{10n^{\Delta}}=\frac{3}{5n^{\Delta}}.
\]
So, to finish the proof, it is enough to show that the inequality
above cannot hold if Condition (b) or Condition (c) holds. This follows
from the low degree moment lemma for distributions(Lemma \ref{main lemma: low degree moment lemma for distributions}),
for a sufficiently large choice of $B$, thereby finishing the proof\footnote{To be explicit: if condition (a) doesn't hold but condition (b) or
(c) does hold via union bound the probability that $\mathcal{T}$
will fail to output \textbf{No} is at most $\frac{1}{100}+\frac{1}{100}<0.1$
as required. }.
\end{proof}
Finally, we can use the two propositions above to finish the proof
of Theorem \ref{thm: main testing learning theorem for LTFs}. Bounds
on run-time have been shown earlier, so now we need to show correctness.
That requires us to show the following two conditions:
\begin{enumerate}
\item \textbf{(Composability) }If, given access to i.i.d. labeled samples
$(x,y)$ distributed according to $D_{\text{pairs}}$, the algorithm
$\mathcal{T}$ outputs ``Yes'' with probability at least $1/4$,
then $\mathcal{A}$ will with probability at least $0.9$ output a
circuit computing a function $\hat{f}$, such that 
\[
\Pr_{\parr{x,y}\in_{R}D_{\text{pairs}}}[y\neq\hat{f}(x)]\leq\min_{f\in\text{halfspaces}}\parr{\Pr_{\parr{x,y}\in_{R}D_{\text{pairs}}}[f(x)\neq y]}+O\parr{\epsilon}.
\]
\item \textbf{(Completeness) }Given access to i.i.d. labeled samples $(x,y)$
distributed according to $D_{\text{pairs}}$, with $x$ itself distributed
as a Gaussian over $R^{n}$, tester $\mathcal{T}$ outputs ``Yes''
with probability at least $3/4$.
\item $\mathcal{A}$ is an agnostic learner for halfspaces over $\R^{n}$
under the Gaussian distribution.
\end{enumerate}
Note that Condition 3 follows from the first two. The completeness condition (i.e. Condition 2) immediately
follows from Proposition \ref{prop: tester completeness + tester rejection condition}.
The composability condition (i.e.
Condition 1) follows from Proposition \ref{prop: tester completeness + tester rejection condition}
 and Proposition \ref{prop: tester acceptance implies learner works}
in following way. If $\mathcal{T}$ outputs ``No'' with probability
less than $3/4$ then conditions (a), (b) and (c) in Proposition \ref{prop: tester completeness + tester rejection condition}
should all be violated. This allows us to use Proposition
\ref{prop: tester acceptance implies learner works} to conclude that $\mathcal{A}$ is an agnostic $(O\left(\epsilon\right),0.1)$-learner
for the function class of linear threshold functions over $\R^{n}$
under distribution $D$, where $D$ is the marginal distribution of
$x$ when $(x,y)$ distributed according to $D_{\text{pairs}}$. This
implies the composability condition (i.e.
Condition 1 above) and finishes the proof of Theorem \ref{thm: main testing learning theorem for LTFs}.

\section{
Tester-learner pairs for agnostically learning halfspaces under the
uniform distribution over Boolean cube.}
\label{section: tester-learner under cube}
\subsection{The tester-learner pair.}
\noindent
{\bf Tester-learner pair for learning halfspaces over $\{0,1\}^n$:}
\begin{itemize}
\item Let $C_{1}$ be a sufficiently large constant to be tuned
appropriately. Also define $k:=\frac{1}{50\epsilon^4} \ln^4 \frac{1}{\epsilon}$.
\item \textbf{Learning algorithm $\mathcal{A}_{\text{Boolean}}$.} Given access to i.i.d.
labeled samples $(\vect x,y)\in\{\pm 1\}^{n}\times\left\{ \pm1\right\} $
from an unknown distribution:
\begin{itemize}
\item Use the algorithm of Theorem \ref{theorem: agnostic learning from L1 approximation} (that came from \cite{kalai_agnostically_2005}), with error parameter $C_1 \epsilon$, 
allowed failure probability $\frac{1}{10},$ and taking the set of $\left\{ g_{i}\right\} $
to be the set of monomials of degree at most $\frac{20}{\epsilon^4} \ln^2 \frac{1}{\epsilon}$, i.e. the set $\left\{ \prod_{j=1}^{n}x_{j}^{\alpha_{j}}:~\sum_j \alpha_j\leq \frac{20}{\epsilon^4} \ln^2 \frac{1}{\epsilon}\right\}$ (with all $\alpha_j\in \{0,1\}$ because $x_k$ are in $\{\pm 1\}$)
.
\end{itemize}
\item \textbf{Testing algorithm $\mathcal{T}_{\text{Boolean}}$}. Given access to i.i.d.
labeled examples $\vect x\in \{\pm 1\}^{n}$ from an unknown distribution:
\begin{enumerate}
\item Use a tester from literature (see \cite{ODonnellZ18,AlonAKMRX07,AlonGM03}) for testing $k$-wise independent distributions against distributions that are $n^{-\frac{42}{\epsilon^4}\ln^2\parr{\frac{1}{\epsilon}}}$-far from $k$-wise independent.
\item Output the same response as the one given by the $k$-wise independence tester. 
\end{enumerate}
\end{itemize}

\begin{thm}
[\textbf{Tester-learner pair for learning halfspaces under uniform distribution on $\{\pm 1\}^n$}]
\label{thm: main testing learning theorem for LTFs on cube} Suppose the value
$C$ present in algorithm $\mathcal{A}_{\text{Boolean}}$ 
is chosen to be a sufficiently large absolute constant, also assume
$n$ and $\frac{1}{\epsilon}$ are larger than some sufficiently large
absolute constants. Then, the algorithm $\mathcal{A}_{\text{Boolean}}$ is an agnostic
$(O(\epsilon),0.1)$-learner for the function class of linear threshold
functions over $\{\pm 1\}^{n}$ under the uniform distribution 
and the algorithm $\mathcal{T}_{\text{Boolean}}$ is an 
assumption tester
for $\mathcal{A}_{\text{Boolean}}$. The algorithms $\mathcal{A}_{\text{Boolean}}$ and $\mathcal{T}_{\text{Boolean}}$
both require only $n^{\tilde{O}\parr{\frac{1}{\epsilon^{4}}}}$ samples
and run-time. Additionally, the tester $\mathcal{T}_{\text{Boolean}}$ is label-oblivious.   
\end{thm}
The testers from the literature for $k$-wise independence take $n^{O(k)}/\eta^2$ samples and run-time to distinguish a $k$-wise independent distribution and a distribution that is $\eta$-far from $k$-wise independent (see \cite{ODonnellZ18,AlonAKMRX07,AlonGM03}). Thus, the run-time of tester $\mathcal{T}_{\text{Boolean}}$ is $n^{\tilde{O}(1/\epsilon^4)}$. The same run-time bound of  $n^{\tilde{O}(1/\epsilon^4)}$ for $\mathcal{A}_{\text{Boolean}}$ follows from \Cref{theorem: agnostic learning from L1 approximation}.

The only thing remaining to prove is that the algorithm $\mathcal{A}_{\text{Boolean}}$ is indeed a $(O(\epsilon),0.1)$-agnostic learning algorithm for the class of halfspaces on $\{\pm 1\}^n$ with respect to distributions $D$ that are $n^{-\frac{42}{\epsilon^4}\ln^2\parr{\frac{1}{\epsilon}}}$-close to $k$-wise independent. By \Cref{theorem: agnostic learning from L1 approximation} (that came from \cite{kalai_agnostically_2005}), this follows from the following proposition:
\begin{restatable*}[low-degree approximation]{prop}{lowDegreeApprox}
\label{prop: low-degree approximation for close to k-wise independent on cube}
Let $\sign{(\vect{v} \cdot \vect{x}-\theta)}$ be an arbitrary halfspace, $\vect{v}$ be normalized to be a unit vector, 
and let $k:=\frac{1}{50\epsilon^4} \ln^4 \frac{1}{\epsilon}$. Also let $D$ be a distribution that is $n^{-\frac{42}{\epsilon^4}\ln^2\parr{\frac{1}{\epsilon}}}$-close in TV distance to $k$-wise independent. Then, there is a polynomial $P$ of degree $\frac{20}{\epsilon^4} \ln^2 \frac{1}{\epsilon}$ for which
\[
\E_{\vect{x} \sim D }
\pars{
\abs{
P(\vect{x})-
\sign{(\vect{v} \cdot \vect{x}-\theta)}
}
}
=
O(\epsilon)
\]
\end{restatable*}
The remaining subsections are dedicated to proving \Cref{prop: low-degree approximation for close to k-wise independent on cube} which finishes the proof of \Cref{thm: main testing learning theorem for LTFs on cube}.

\subsection{Proving that halfspaces are well-approximated by low-degree polynomials under distributions close to $k$-wise independent.}
\subsubsection{Basic facts.}
We now present some basic facts and definitions. 
\begin{defn}[From \cite{DiakonikolasGJSV09}]
We say that the halfspace $\sign{(\vect{v} \cdot \vect{x} -\theta)}$ is $\epsilon$-regular if for any $i$ we have $\abs{v_i}/\norm{\vect{v}}\leq \epsilon$.
\end{defn}

The following is a standard corollary of the Berry-Esseen theorem (see for example Corollary 2.2 of \cite{DiakonikolasGJSV09Jounal}).
\begin{prop}
\label{prop: anti-concentration when regular and i.i.d. uniform}
Suppose the halfspace $\sign{(\vect{v} \cdot \vect{x} -\theta)}$ is $\epsilon$-regular, then for any interval $[a,b]\subset \R$ we have
\[
\Pr_{\vect{x} \sim \{\pm 1\}^n }
\pars{
\frac{
\vect{v} \cdot \vect{x}
}
{\norm{\vect{v}}}
\in 
[a, b]
}
\leq
\abs{b-a}+2\epsilon.
\]
\end{prop}

We will also need the fact about the concentration properties of a $k$-wise independent distribution on $\{\pm 1\}^n$, when it is projected to an arbitrary direction.
\begin{prop}
\label{prop: moment bound (for Boolean cube setting)}
Suppose $D$ is a $k$-wise independent distribution over $\{\pm 1\}^n$. Then, for any unit vector $\vect{v} \in \R^n$ and even integer $d \in  [2,k]$, we have
\[
\parr{\E_{\vect x\sim D}\pars{\parr{\vect v\cdot\vect x}^{d}}}^{1/d}\leq2\sqrt{d},
\]
\end{prop}
\begin{proof}
Since $d\leq k$ and $D$ is $k$-wise independent, we have 
\[
\E_{\vect x\sim D}\pars{\parr{\vect v\cdot\vect x}^{d}}
=
\E_{\vect x\sim \{\pm1\}^n}\pars{\parr{\vect v\cdot\vect x}^{d}}.
\]
The standard Hoeffding bound tells us that for any $t\in \R$ 
\[
\Pr_{\vect x\sim \{\pm1\}^n}
\pars{
\abs{\vect v\cdot\vect x}\geq t
}
\leq
2e^{-t^2/2}.
\]
Therefore

\begin{multline*}
\E_{\vect x\sim \{\pm1\}^n}\pars{\parr{\vect v\cdot\vect x}^{d}}
=
\E_{\vect x\sim \{\pm1\}^n}\pars{
\int_{\tau=0}^{\infty} d\tau^{d-1}  \indicator_{\abs{\vect v\cdot\vect x}>\tau} \d t
} =
\\
\underbrace{
=\int_{0}^{\infty}d\tau^{d-1}\Pr_{\vect x\sim \{\pm1\}^n}\pars{|\vect v\cdot\vect x|>\tau}\d\tau
}_{\text{Via linearity of expectation and Tonelli's theorem.}}\leq
2\int_{0}^{\infty}d\tau^{d-1} e^{-\tau^2/2}\d\tau=\\
\sqrt{2\pi} \cdot d \cdot \E_{\tau \sim N(0,1)}\pars{\abs{\tau}^{d-1}}
=\sqrt{2\pi} \cdot d \cdot \sqrt{\frac{2}{\pi}} (d-2)!!
= 2 d!!\leq 2 d^{d/2}.
\end{multline*}
This directly implies the statement we were seeking to prove.
\end{proof}
\subsubsection{Re-using the polynomial from \Cref{subsection: proof of low degree moment lemma for distributions}.}
We will use the polynomial constructed in \Cref{subsection: proof of low degree moment lemma for distributions}, which we designed to approximate well the function $\indicator_{[y,y+\epsilon]}$. We now summarize its properties 
\begin{prop}
\label{prop: nice poly exists (for anti-concentration)}
For every $y\in \R$, $\epsilon\in (0,1]$, define
\begin{equation*}
g(z):=\begin{cases}
0 & \text{if \ensuremath{z\leq y-\epsilon}},\\
\frac{z-\left(y-\epsilon\right)}{\epsilon} & \text{if }\ensuremath{z\in\pars{y-\epsilon,y}},\\
1 & \text{if }\ensuremath{z\in\pars{y,y+\epsilon}},\\
\frac{\left(y+2\epsilon\right)-z}{\epsilon} & \text{if }\ensuremath{z\in\pars{y+\epsilon,y+2\epsilon},}\\
0 & \text{if \ensuremath{z\geq y+2\epsilon}}.
\end{cases}
\end{equation*}
Then, for any $w\in \R_{>1}$,
there exists a polynomial $P_0$ of degree $d=O(w/\epsilon^2)$, such that for any $x\in [-w,w]$ we have $\abs{g(x)-P_0(x)}\leq \epsilon$.
Additionally, each coefficient of $P_0$ has a magnitude of at most $d3^d$.  
\end{prop}

\subsubsection{Proof of \Cref{prop: low-degree approximation for close to k-wise independent on cube}}
First, we show that $k$-wise independent distributions are anti-concentrated when projected onto regular vectors. 
\begin{prop}

\label{prop: anti-concentration for k-wise independent with relgular vector}
Suppose the halfspace $\sign{(\vect{v} \cdot \vect{x}-\theta)}$ is $\epsilon$-regular, $\vect{v}$ is normalized to be a unit vector,
and let $k:=\frac{1}{100\epsilon^4} \ln^4 \frac{1}{\epsilon}$. Then, for any $k$-wise independent distribution $D$ we have for every $y\in \R$ that
\[
\Pr_{\vect{x} \sim D }
\pars{
\vect{v} \cdot \vect{x}
\in 
[y, y+\epsilon]
}
=
O(\epsilon)
\]
\end{prop}

\begin{proof}
We take $w:=\frac{1}{\epsilon^2} \ln^2 \frac{1}{\epsilon}$, and WLOG assume that $\epsilon$ is small enough that $w>1$. Let $P_0$ be as in \Cref{prop: nice poly exists (for anti-concentration)}. 
First, we would like to bound
$
\abs{
\E_{\vect x\sim D}\pars{P_0 \parr{
\vect{v} \cdot \vect{x}
}
\indicator_{\abs{\vect{v} \cdot \vect{x}} >w}}
}
$.
To do this, first we observe that by combining \cref{prop: moment bound (for Boolean cube setting)} and \cref{prop: moment tail bound for bounded moment distributions}
we have
\[
\max_{i\in \{0,\cdots, d\}}
\E_{\vect x\in_{R}D}\pars{\abs{
\vect{v} \cdot \vect{x}
}
^{i}\indicator_{\abs{x}>w}}\leq2w^{d}\parr{\frac{2\sqrt{k}}{w}}^{k}.
\]
Each coefficient of $P_0$ is bounded by $d3^d$, this means that 
\begin{equation}
\label{eq: 1 for regular halfspaces}
\abs{
\E_{\vect x\sim D}\pars{P_0
\parr{
\vect{v} \cdot \vect{x}
}
\indicator_{z >w}}
}
\leq
2d^23^d
w^{d}\parr{\frac{2\sqrt{k}}{w}}^{k}
\leq
O
\parr{
4^d
w^{d}\parr{\frac{2\sqrt{k}}{w}}^{k}
}.
\end{equation}

Repeating the exact same argument above for the uniform distribution over $\{\pm 1\}^n$ (in place of $D$) we also get
\begin{equation}
\label{eq: 2 for regular halfspaces}
\abs{
\E_{\vect x\sim \{\pm 1\}^n}\pars{P_0
\parr{
\vect{v} \cdot \vect{x}
}
\indicator_{z >w}}
}
\leq
O
\parr{
4^d
w^{d}\parr{\frac{2\sqrt{k}}{w}}^{k}
}.
\end{equation}
Now, we consider the region inside $[-w,w]$. We have
\begin{equation}
\label{eq: 3 for regular halfspaces}
5\epsilon
\overbrace{
\geq
\Pr_{\vect{x} \sim \{\pm 1\}^n }
\pars{
\vect{v} \cdot \vect{x}
\in 
[y-\epsilon,y+2\epsilon]
}}
^{\text{By \Cref{prop: anti-concentration when regular and i.i.d. uniform}.}}
\overbrace{
\geq
\E_{\vect x\sim \{\pm 1\}^n}\pars{P_0
\parr{
\vect{v} \cdot \vect{x}
}
\indicator_{z \leq w}}
-\epsilon
}^{\text{Because on $[-w,w]$ we have $P_0(z)\leq \indicator_{[y-\epsilon,y+2\epsilon]} + \epsilon$}}
\end{equation}
Similarly, we also have 
\begin{equation}
\label{eq: 4 for regular halfspaces}
\E_{\vect x\sim D}\pars{P_0
\parr{
\vect{v} \cdot \vect{x}
}
\indicator_{z  \leq w}}
\overbrace{\geq
\Pr_{\vect{x} \sim D }
\pars{
\vect{v} \cdot \vect{x}
\in 
[y, y+\epsilon]
}
-
\epsilon
}^{\text{Because on $[-w,w]$ we have $P_0(z)\geq \indicator_{[y, y+\epsilon]} - \epsilon$}}
.
\end{equation}
Taking together \cref{eq: 1 for regular halfspaces}, \cref{eq: 2 for regular halfspaces}, \cref{eq: 3 for regular halfspaces} and \cref{eq: 4 for regular halfspaces} we get 
\[
\Pr_{\vect{x} \sim D }
\pars{
\vect{v} \cdot \vect{x}
\in 
[y, y+\epsilon]
}
\leq 
O(\epsilon)
+O
\parr{
4^d
w^{d}\parr{\frac{2\sqrt{k}}{w}}^{k}
}.
\]
Substituting $k=\frac{1}{100\epsilon^4} \ln^4 \frac{1}{\epsilon}$, $d=O(w/\epsilon^2)$ and $w=\frac{1}{\epsilon^2} \ln^2 \frac{1}{\epsilon}$ we now get
\[
\Pr_{\vect{x} \sim D }
\pars{
\vect{v} \cdot \vect{x}
\in 
[y, y+\epsilon]
}
\leq 
O(\epsilon)
+O
\parr{
4^d
w^{d-k}\parr{2\sqrt{k}}^{k}
}
=
O(\epsilon)
\]
\end{proof}
Now, we use the proposition we just proved to show that, with respect to $k$-wise independent distributions, low-degree polynomials approximate well halfspaces whose normal vectors are regular.
\begin{prop}
\label{prop: approximation with polynomials of regular halfspace under k-wise}
Suppose the halfspace $\sign{(\vect{v} \cdot \vect{x}-\theta)}$ is $\epsilon$-regular, $\vect{v}$ is normalized to be a unit vector, 
and let $k:=\frac{1}{100\epsilon^4} \ln^4 \frac{1}{\epsilon}$. Then, for any $k$-wise independent distribution $D$ we have a polynomial $P$ of degree $d:=\frac{1}{4\epsilon^4}\ln^2\parr{\frac{1}{\epsilon}}$ for which
\[
\E_{\vect{x} \sim D }
\pars{
\abs{
P( \vect{x})-
\sign{(\vect{v} \cdot \vect{x}-\theta)}
}
}
=
O(\epsilon)
\]
Additionally, each coefficient of polynomial $P$ is bounded by $(4n)^d$ in absolute value.
\end{prop}
\begin{proof}
We combine \Cref{prop: moment bound (for Boolean cube setting)} and \Cref{prop: anti-concentration for k-wise independent with relgular vector} with \Cref{main lemma: low degree approximation lemma for halfspaces.}. In \Cref{main lemma: low degree approximation lemma for halfspaces.}, we have $\alpha=O(\epsilon)$, $d_0=k$ and $\beta=2\sqrt{d_0}$. Overall, from the conclusion of \Cref{main lemma: low degree approximation lemma for halfspaces.} it follows that for some polynomial $P(\vect{x})=Q(\vect{v}\cdot \vect{x})$ it is indeed the case that 
\[
\E_{\vect{x} \sim D }
\pars{
\abs{
P(\vect{x})-
\sign{(\vect{v} \cdot \vect{x}-\theta)}
}
}
=
O(\epsilon).
\]
The degree of the polynomial $P$ is $\frac{2\beta}{\epsilon^2}+1$ which is at most $\frac{1}{4\epsilon^4}\ln^2\parr{\frac{1}{\epsilon}}$ for sufficiently small $\epsilon$.

Now, we need to bound the (multivariable) coefficients of $P$. To do this, fix a specific multivariable term and track how much it can grow as we open the parentheses for $Q(\vect{v}\cdot \vect{x})$. 
As all coordinates of unit vector $\vect{v}$ are bounded by $1$, every time we open the parentheses for a term of form $c_i (\vect{v}\cdot \vect{x})^i$, it can contribute at most $\abs{c_i} n^i$ to the absolute value of any specific coefficient of $P$. As we know that every single-variable coefficient $c_i$ of $Q$ is bounded by $d3^d$, we get an overall bound of $d(3n)^d\leq (s4n)^d$ on each multivariate coefficient of $P$. 
\end{proof}

Consequently, we use ideas similar to the ones in \cite{DiakonikolasGJSV09Jounal} in order to reduce the case of general halfspaces to the case of halfspaces whose normal vectors are regular.
\begin{prop}
\label{prop: approximation for exactly k-wise independent}
Let $\sign{(\vect{v} \cdot \vect{x}-\theta)}$ be an arbitrary halfspace, $\vect{v}$ be normalized to be a unit vector, 
and let $k:=\frac{1}{50\epsilon^4} \ln^4 \frac{1}{\epsilon}$. Then, for any $k$-wise independent distribution $D$ we have a polynomial $P$ of degree $\frac{20}{\epsilon^4} \ln^2 \frac{1}{\epsilon}$ for which
\[
\E_{\vect{x} \sim D }
\pars{
\abs{
P(\vect{x})-
\sign{(\vect{v} \cdot \vect{x}-\theta)}
}
}
=
O(\epsilon)
\]
Additionally, each coefficient of the polynomial $P$ has a magnitude of at most $n^{\frac{20}{\epsilon^4}\ln^2\parr{\frac{1}{\epsilon}}}$.
\end{prop}
\begin{proof}
Without loss of generality, we assume that the values of $\vect{v}$ are in decreasing order (i.e. $v_i\geq v_{i+1}$). We use the notation $\sigma_i =\sqrt{\sum_{j>i}v_i^2}$.
The \textbf{critical index} $\ell(\epsilon)$ is defined as the smallest $i$ for which $v_i \leq \epsilon \sigma_i$. We set $\ell_0=\frac{8 \log^2 (10/\epsilon)}{\epsilon^2}$ and consider two cases: (i) $\ell(\epsilon)\leq \ell_0$ and (ii) $\ell(\epsilon) > \ell_0$.

Suppose $\ell(\epsilon)\leq \ell_0$, then write the vector $\vect{v}$ as the concatenation of two vectors $\vect{v}_{\text{head}}$ in $\R^{\ell(\epsilon)}$ and $\vect{v}_{\text{head}}$ in $\R^{n-\ell(\epsilon)}$. 
Analogously a vector $\vect{x}$ in $\{\pm 1 \}^n$ can be broken down into $\vect{x}_{\text{head}}$ in $\{\pm 1 \}^{\ell(\epsilon)}$ and $\vect{x}_{\text{head}}$ in $\{\pm 1 \}^{n-\ell(\epsilon)}$. 
For any fixed value of $\vect{x}_{\text{head}}$, the condition $\ell(\epsilon)\leq \ell_0$ directly implies that the halfspace $\sign{(\vect{v}_{\text{head}} \cdot \vect{x}_{\text{head}} + \vect{v}_{\text{tail}} \cdot \vect{x}_{\text{tail}}  -\theta)}$ is a regular halfspace. 
Since $D$ is a $k$-wise independent distribution, when one conditions on a specific value of $\vect{x}_{\text{head}}$, the resulting distribution over $\vect{x}_{\text{tail}}$ is $k-\ell_0$-wise independent. 
Therefore by 
\Cref{prop: approximation with polynomials of regular halfspace under k-wise} there is some polynomial $P^{\vect{x}_{\text{head}}}(\vect{v}_{\text{tail}} \cdot \vect{x}_{\text{tail}})$ of degree $\frac{1}{4\epsilon^4}\ln^2\parr{\frac{1}{\epsilon}}$ for which we have:
\[
\E_{\vect{x} \sim D }
\pars{
\abs{
P^{\vect{x}_{\text{head}}}(\vect{v}_{\text{tail}} \cdot \vect{x}_{\text{tail}})-
\sign{(\vect{v}_{\text{head}} \cdot \vect{x}_{\text{head}} + \vect{v}_{\text{tail}} \cdot \vect{x}_{\text{tail}}  -\theta)}
}
\:
\bigg \vert
\vect{x}_{\text{head}}
}
=
O(\epsilon).
\]
This means, that if we take our polynomial $P$ to map $\vect{x}=(\vect{x}_{\text{head}},\vect{x}_{\text{tail}})$ to $\sum_{\vect{x}_{0}\in\{\pm 1 \}^{\ell(\epsilon)}}(\indicator_{\vect{x}_{\text{head}}=\vect{x}_{0}} \cdot P^{\vect{x}_{0}}(\vect{x}_{\text{tail}}))$ then we will overall have:
\[
\E_{\vect{x} \sim D }
\pars{
\abs{
P(\vect{x})-
\sign{(\vect{v} \cdot \vect{x}-\theta)}
}
}
=
O(\epsilon).
\]
Since the indicators $\indicator_{\vect{x}_{\text{head}}=\vect{x}_{0}}$ have degree of at most $\ell_0$, the polynomial $P$ has a degree of at most $\ell_0+\frac{1}{4\epsilon^4}\ln^2\parr{\frac{1}{\epsilon}}$, which is at most $\frac{20}{\epsilon^4} \ln^2 \frac{1}{\epsilon}$ for sufficiently small $\epsilon$ as required. 

Let us bound the coefficients of $P$. For each fixed $\vect{x}_{0}$, we know that the coefficients of $P^{\vect{x}_{0}}(\vect{x}_{\text{tail}}))$  are bounded by $(4n)^{\frac{1}{4\epsilon^4}\ln^2\parr{\frac{1}{\epsilon}}}$. Each coefficient of $\indicator_{\vect{x}_{\text{head}}=\vect{x}_{0}}$ is bounded by $\frac{1}{2^l}$ (this follows by explicitly writing out this polynomial). Overall, (since the variables in $P^{\vect{x}_{0}}(\vect{x}_{\text{tail}}))$ and $\indicator_{\vect{x}_{\text{head}}=\vect{x}_{0}}$ are disjoint) we see that each coefficient of $\indicator_{\vect{x}_{\text{head}}=\vect{x}_{0}} \cdot P^{\vect{x}_{0}}(\vect{x}_{\text{tail}})$ is bounded in absolute value by $(4n)^{\frac{1}{4\epsilon^4}\ln^2\parr{\frac{1}{\epsilon}}}/2^l$. Summing this over all $\vect{x}_{0}$, we see that every coefficient of $P$ is then at most $(4n)^{\frac{1}{4\epsilon^4}\ln^2\parr{\frac{1}{\epsilon}}}$ in absolute value. (This is at most $n^{\frac{20}{\epsilon^4}\ln^2\parr{\frac{1}{\epsilon}}}$ for sufficiently small $\epsilon$).

This concludes our consideration of the case $\ell(\epsilon)\leq \ell_0$, and the rest of the proof examines the case $\ell(\epsilon)> \ell_0$.

Suppose we have $\ell(\epsilon) > \ell_0$. 
Similar to before, we break the vector $\vect{v}$ into $\vect{v}_{\text{head}}$ in $\R^{\ell_0}$ and $\vect{v}_{\text{tail}}$ in $\R^{n-\ell_0}$ and the vector $\vect{x}$ in $\{\pm 1 \}^n$ into $\vect{x}_{\text{head}}$ in $\{\pm 1 \}^{\text{tail}}$ and $\vect{x}_{\ell_0}$ in $\{\pm 1 \}^{n-\ell_0}$. 
The polynomial we shall use to approximate the halfspace will now depend entirely on $\vect{x}_{\text{head}}$. Specifically, it will make the natural best guess at $\sign{(\vect{v}_{\text{head}} \cdot \vect{x}_{\text{head}} + \vect{v}_{\text{tail}} \cdot \vect{x}_{\text{tail}}  -\theta)}$ given only $\vect{x}_{\text{head}}$, i.e. we have $P$ mapping $\vect{x}=(\vect{x}_{\text{head}},\vect{x}_{\text{tail}})$ to $\sum_{\vect{x}_{0}\in\{\pm 1 \}^{\ell_0}}(\indicator_{\vect{x}_{\text{head}}=\vect{x}_{0}} \cdot \sign{(\vect{v}_{\text{head}} \cdot \vect{x}_{0}  -\theta)})$. 
Since the indicators $\indicator_{\vect{x}_{\text{head}}=\vect{x}_{0}}$ have degree of $\ell_0$, the polynomial also has a degree of at most $\ell_0$. Each of the indicators $\indicator_{\vect{x}_{\text{head}}=\vect{x}_{0}}$ has coefficients equal to $\frac{1}{2^{\ell_0}}$ in absolute value, and there at most $n^{\ell_0}$ of these indicator polynomials. Therefore, each coefficient of $P$ is can be bounded by $n^{\ell_0}$ in absolute value.

We now want to argue that $P$ has a small error. We will use the following proposition that is implicit in the proof of Theorem 5.4 of \cite{DiakonikolasGJSV09Jounal}.
\begin{prop}
For $\ell_0=\frac{8 \log^2 (10/\epsilon)}{\epsilon^2}$, suppose $D$ is a $(\ell_0+2)$-wise independent distribution over $\{\pm1\}^n$,  $\sign{(\vect{v} \cdot \vect{x}-\theta)}$ is a halfspace with critical index $\ell(\epsilon)> \ell_0$. Also suppose $\vect{v}$ is a unit vector and its coordinates of $\vect{v}$ are in descending order, and break $\vect{v}$ into $\vect{v}_{\text{head}}$ in $\R^{\ell_0}$ and $\vect{v}_{\text{head}}$ in $\R^{n-\ell_0}$ and the vector $\vect{x}$ in $\{\pm 1 \}^n$ into $\vect{x}_{\text{head}}$ in $\{\pm 1 \}^{\ell_0}$ and $\vect{x}_{\ell_0}$ in $\{\pm 1 \}^{n-\ell_0}$. Then we have
\[
\Pr_{\vect{x}\sim D}
\pars{
\sign{(\vect{v}\cdot \vect{x}   -\theta)}
\neq
\sign{(\vect{v}_{\text{head}} \cdot \vect{x}_{\text{head}}  -\theta)}
}
=
O(\epsilon)
\]
\end{prop}
 Now, when $\sign{(\vect{v}\cdot \vect{x} \cdot \vect{x}_{\text{tail}}  -\theta)}
=
\sign{(\vect{v}_{\text{head}} \cdot \vect{x}_{\text{head}}  -\theta)}$ our polynomial has error zero, and when $\sign{(\vect{v}\cdot \vect{x} \cdot \vect{x}_{\text{tail}}  -\theta)}
\neq
\sign{(\vect{v}_{\text{head}} \cdot \vect{x}_{\text{head}}  -\theta)}
$ our polynomial has an error of $2$. Overall, this means that indeed
\[
\E_{\vect{x} \sim D }
\pars{
\abs{
P(\vect{x})-
\sign{(\vect{v} \cdot \vect{x}-\theta)}
}
}
=
O(\epsilon).
\]
\end{proof}
Finally, we move from distributions that are $k$-wise independent to distributions that are merely close to $k$-wise independent, which concludes this line of reasoning.
\lowDegreeApprox

\begin{proof}
Let $D'$ be the closest in TV distance $k$-wise independent distribution to $D$. We have 
\[
 d_{\text{TV}}(D,D')
 \leq
 n^{-\frac{42}{\epsilon^4}\ln^2\parr{\frac{1}{\epsilon}}}.
\]By \cref{prop: approximation for exactly k-wise independent}, we have a a polynomial $P$ of degree $\frac{20}{\epsilon^4} \ln^2 \frac{1}{\epsilon}$ for which
\begin{equation}
\label{equation: D' is good}
\E_{\vect{x} \sim D' }
\pars{
\abs{
P(\vect{x})-
\sign{(\vect{v} \cdot \vect{x}-\theta)}
}
}
=
O(\epsilon)
\end{equation}

To move from $D$ to $D'$ we use the following observation that follows immediately from the definition of TV distance 
\begin{observation}
Let $\phi$ be some function $\{\pm 1\}^n \righ \R$ and suppose $\phi$ is bounded everywhere by $B$ in absolute value. Let $D$ and $D'$ be two probability distributions over $\{\pm 1\}^n$. Then 
\[
\abs{
\E_{\vect{x}\sim D}[\phi]
-
\E_{\vect{x}\sim D'}[\phi]
}
\leq
B \cdot d_{\text{TV}}(D,D')
\]
\end{observation}

The polynomial $P(\vect{x})$ has at most $n^{\frac{20}{\epsilon^4}\ln^2\parr{\frac{1}{\epsilon}}}$ terms each of which has a coefficient of magnitude at most $n^{\frac{20}{\epsilon^4}\ln^2\parr{\frac{1}{\epsilon}}}$. As each of the terms always evaluates to $\pm 1$ anywhere on $\{\pm 1\}^n$, the absolute value of $P$ is bounded by $n^{\frac{40}{\epsilon^4}\ln^2\parr{\frac{1}{\epsilon}}}$. For all sufficiently small $\epsilon$ we therefore have that $\abs{P(\vect{x})-\sign{(\vect{v} \cdot \vect{x}-\theta)}}\leq n^{\frac{41}{\epsilon^4}\ln^2\parr{\frac{1}{\epsilon}}}$. This, together with the observation above gives us that 
\begin{multline*}
\abs{
\E_{\vect{x}\sim D}[\abs{P(\vect{x})-\sign{(\vect{v} \cdot \vect{x}-\theta)}}]
-
\E_{\vect{x}\sim D'}[\abs{P(\vect{x})-\sign{(\vect{v} \cdot \vect{x}-\theta)}}]
}
\leq\\
n^{\frac{41}{\epsilon^4}\ln^2\parr{\frac{1}{\epsilon}}}
d_{\text{TV}}(D,D')
\leq
n^{\frac{41}{\epsilon^4}\ln^2\parr{\frac{1}{\epsilon}}}
n^{-\frac{42}{\epsilon^4}\ln^2\parr{\frac{1}{\epsilon}}}=O(\epsilon)
\end{multline*}
Combining this with \Cref{equation: D' is good} we finish the proof.
\end{proof}

\section{Lower bounds on testable agnostic learning complexity.}
\label{section: hardness results}
In this section we present sample lower bounds for tester-learner pairs for (i) learning convex sets under Gaussian distribution in $\R^n$ (ii) learning monotone functions under uniform distribution over $\{0,1\}^n$.

\subsection{Theorem statements.}

The following theorem implies that there is no tester-learner pair for agnostic learning convex sets under the standard Gaussian distribution with combined sample complexity of $2^{o(n)}$.
\begin{thm}
\label{thm: hardness for convex strong} For all sufficiently large
$n$, the following is true. Suppose $\mathcal{A}$ is an algorithm
that given sample-label pairs $\left\{ (\vect{x_{i}},y_{i})\right\} \subset\R^{n}\times\left\{ \pm1\right\} $
outputs a function $\hat{f}:\R^{n}\righ\left\{ \pm1\right\} $. Also,
suppose $\mathcal{T}$ is a tester that given access to i.i.d. labeled
points $\left\{ (\vect{x_{i}},y_{i})\right\} \subset\R^{n}\times\left\{ \pm1\right\} $
outputs ``Yes'' or ``No''. Suppose whenever the points $\left\{ \vect{x_{i}}\right\} $
are themselves distributed i.i.d. from $\N\parr{0,I_{n\times n}}$,
tester $\mathcal{T}$ outputs ``Yes'' with probability at least
$1-\delta_{2}$. Also suppose the combined sample complexity of $\mathcal{A}$
and $\mathcal{T}$ is at most $N:=2^{0.01n}$. Then, there is a distribution
$D_{\text{pairs}}$ on $\R^{n}\times\left\{ \pm1\right\} $ such that 
\begin{itemize}
\item There is a function $f_{0}:\R^{n}\righ\left\{ \pm1\right\}$, for
which $\left\{ \vect{x}:\,f_{0}(x)=1\right\}$ is a convex set and $$\Pr_{(\vect{x},y)\sim D_{\text{pairs}}}\pars{f_{0}(\vect{x})=y}=1$$
In other words, it predicts the label perfectly.
\item The tester $\mathcal{T}$, given samples from $D_{\text{pairs}}$,
accepts with probability at least $1-\delta_{2}-\frac{1}{2^{\Omega(n)}}$.
\item The learner $\mathcal{A}$, given samples from $D_{\text{pairs}}$,
outputs a predictor $\hat{f}$ whose expected advantage over random
guessing is at most $\frac{1}{2^{\Omega\parr n}}$.
\end{itemize}
\end{thm}
The following theorem implies that there is no tester-learner pair for agnostic learning monotone functions under uniform distribution over $\{0,1\}^n$ with combined sample complexity of $2^{o(n)}$. Recall that a function $f_0:\{0,1\}^n \righ \{\pm 1\}$ is \emph{monotone} if $f_0(\vect{x}_1)\geq f_0(\vect{x}_2)$ whenever each coordinate of $\vect{x}_1$ is at least as large as the corresponding coordinate of $\vect{x}_2$.
\begin{thm}
\label{thm: hardness for monotone strong} For all sufficiently large
$n$, the following is true. Suppose $\mathcal{A}$ is an algorithm
that given sample-label pairs $\left\{ (\vect{x_{i}},y_{i})\right\} \subset\left\{ 0,1\right\} ^{n}\times\left\{ \pm1\right\} $
outputs a function $\hat{f}:\left\{ 0,1\right\} ^{n}\righ\left\{ \pm1\right\} $.
Also, suppose $\mathcal{T}$ is a tester that given access to i.i.d.
labeled points $\left\{ (\vect{x_{i}},y_{i})\right\} \subset\left\{ 0,1\right\} ^{n}\times\left\{ \pm1\right\} $
outputs ``Yes'' or ``No''. Suppose whenever the points $\left\{ \vect{x_{i}}\right\} $
are themselves distributed i.i.d. uniformly over $\left\{ 0,1\right\} ^{n}$,
tester $\mathcal{T}$ outputs ``Yes'' with probability at least
$1-\delta_{2}$. Also suppose the combined sample complexity of $\mathcal{A}$
and $\mathcal{T}$ is at most $N:=2^{0.01n}$. Then, there is a distribution
$D_{\text{pairs}}$ on $\left\{ 0,1\right\} ^{n}\times\left\{ \pm1\right\} $
such that 
\begin{itemize}
\item There is a monotone $f_{0}:\left\{ 0,1\right\} ^{n}\righ\left\{ \pm1\right\} $
for which $\Pr_{(\vect{x},y)\sim D_{\text{pairs}}}\pars{f_{0}(\vect{x})=y}=1$.
In other words, it predicts the label perfectly.
\item The tester $\mathcal{T}$, given samples from $D_{\text{pairs}}$,
accepts with probability at least $1-\delta_{2}-\frac{1}{2^{\Omega(n)}}$.
\item The learner $\mathcal{A}$, given samples from $D_{\text{pairs}}$,
outputs a predictor $\hat{f}$ whose expected advantage over random
guessing is at most $\frac{1}{2^{\Omega\parr{\frac{n}{\log^{2}n}}}}$.
\end{itemize}
\end{thm}

\subsection{Technical lemmas about behavior of testing and learning algorithms.}
In this section we show lemmas that are helpful to show inability of testing and learning algorithms to perform well under certain circumstances. Roughly, the following lemma says that one can ``fool'' a tester for a specific distribution $D$ by replacing it by a uniform sample from a set $S$ of sufficiently large size, where each element in $S$ is a uniform sample from $D$.
\begin{lem}
\label{lem: tester fooling lemma}Let $D$ be some fixed distribution
over $U$. Suppose that a tester $\mathcal{T}$ outputs ``Yes''
with probability at least $1-\delta_{2}$ whenever given access to
i.i.d. labeled samples $(x,y)\in U\times\left\{ \pm1\right\} $ distributed
according to $D_{\text{pairs}}$, such that $x$ itself is distributed
according to $D$. Furthermore, suppose the number of samples consumed
by $\mathcal{T}$ is at most $N$. Fix some function $g:U\righ\left\{ \pm1\right\} $
and let $S$ be a random multiset of $M$ i.i.d. elements drawn from
$D$. Then, with probability at least $1-\Delta$ over the choice of $S$
we have 
\[
\Pr_{\substack{x_{1},\cdots,x_{N}\sim S\\
\text{randomness of \ensuremath{\mathcal{T}}}
}
}\pars{\mathcal{T}\parr{\parr{x_{1},g(x_{1})},\cdots,\parr{x_{N},g(x_{N})}}=\text{``Yes"}}\geq1-\delta_{2}-\frac{N^{2}}{M}-\frac{N}{\sqrt{\Delta M}}.
\]
 
\end{lem}

\begin{proof}
Let the elements of the multiset $S$ be $\parr{z_{1},\cdots,z_{M}}$,
which recall are i.i.d. from $D$. Let $\parr{\ensuremath{z_{i_{1}},\cdots,z_{i_{N}}}}$
be sampled i.i.d. from $S$. We have
\begin{multline*}
\Pr\left[\left.\text{\ensuremath{\mathcal{T}} accepts given \ensuremath{\parr{\ensuremath{\parr{z_{i_{1}},g\parr{z_{i_{1}}}},\cdots,\parr{z_{i_{N}},g\parr{z_{i_{N}}}}}}} }\right|S=\parr{\ensuremath{z_{1},\cdots,z_{M}}}\right]\geq\\
\Pr\left[\left.\text{\ensuremath{\mathcal{T}} accepts given \ensuremath{\parr{\ensuremath{\ensuremath{\parr{z_{i_{1}},g\parr{z_{i_{1}}}},\cdots,\parr{z_{i_{N}},g\parr{z_{i_{N}}}}}}}} }\right|S=\parr{\ensuremath{z_{1},\cdots,z_{M}}},\forall\mathit{j}_{1}\neq j_{2}:\,i_{j_{1}}\neq i_{j_{2}}\right]\cdot\\
\cdot\underbrace{\Pr\left[\forall j_{1}\neq j_{2}:\,i_{j_{1}}\neq i_{j_{2}}\right]}_{\geq1-\frac{N^{2}}{M}\text{, via birthay-paradox argument}}\geq\\
\parr{1-\frac{N^{2}}{M}}\Pr\left[\left.\text{\ensuremath{\mathcal{T}} accepts given \ensuremath{\parr{\ensuremath{\parr{z_{i_{1}},g\parr{z_{i_{1}}}},\cdots,\parr{z_{i_{N}},g\parr{z_{i_{N}}}}}}} }\right|S=\parr{\ensuremath{z_{1},\cdots,z_{M}}},\forall j_{1}\neq j_{2}:\,i_{j_{1}}\neq i_{j_{2}}\right]
\end{multline*}
In expectation, for the above probability we have

\begin{multline*}
\Pr\left[\text{\ensuremath{\mathcal{T}} accepts given \ensuremath{\parr{\ensuremath{\parr{z_{i_{1}},g\parr{z_{i_{1}}}},\cdots,\parr{z_{i_{N}},g\parr{z_{i_{N}}}}}}} }\right]\geq\\
\geq\parr{1-\frac{N^{2}}{M}}\parr{1-\delta_{2}}\geq1-\delta_{2}-\frac{N^{2}}{M},
\end{multline*}
because the conditioning on $i_{j_{1}}$ and $i_{j_{2}}$ being all
distinct results in feeding $\ensuremath{\mathcal{T}}$ with i.i.d.
uniform sample-label pairs for which we know the acceptance probability
is at least $1-\delta_{2}$, as given in the premise of the claim.
Having bound the expectation of this probability, let us now bound
its variance. Define

\begin{multline*}
p_{\text{average}}:=\E_{S}\left[\Pr\left[\left.\text{\ensuremath{\mathcal{T}} accepts given \ensuremath{\parr{\ensuremath{\parr{z_{i_{1}},g\parr{z_{i_{1}}}},\cdots,\parr{z_{i_{N}},g\parr{z_{i_{N}}}}}}} }\right|S=\parr{\ensuremath{z_{1},\cdots,z_{M}}},\forall j_{1}\neq j_{2}:\,i_{j_{1}}\neq i_{j_{2}}\right]\right]
\end{multline*}
We have 
\begin{multline*}
\E_{S}\left[\left(\Pr\left[\left.\text{\ensuremath{\mathcal{T}} accepts given \ensuremath{\parr{\ensuremath{\parr{z_{i_{1}},g\parr{z_{i_{1}}}},\cdots,\parr{z_{i_{N}},g\parr{z_{i_{N}}}}}}} }\right|S=\parr{\ensuremath{z_{1},\cdots,z_{M}}},\forall j_{1}\neq j_{2}:\,i_{j_{1}}\neq i_{j_{2}}\right]-p_{\text{average}}\right)^{2}\right]=\\
\E\left[\left(\Pr\left[\text{\ensuremath{\mathcal{T}} accepts given \ensuremath{\left\{ \ensuremath{\parr{z_{k_{1}},g\parr{z_{k_{1}}}},\cdots,\parr{z_{k_{N}},g\parr{z_{k_{N}}}}}\right\} } }\right]-p_{\text{average}}\right)\right]\cdot\\
\cdot\left(\Pr\left[\text{\ensuremath{\mathcal{T}} accepts given \ensuremath{\left\{ \ensuremath{\parr{z_{l_{1}},g\parr{z_{l_{1}}}},\cdots,\parr{z_{l_{N}},g\parr{z_{l_{N}}}}}\right\} } }\right]-p_{\text{average}}\right)
\end{multline*}
where $\left\{ k_{1},\cdots,k_{N_{\text{tester}}}\right\} $ and $\left\{ \ell_{1},\cdots,\ell_{N_{\text{tester}}}\right\} $
are picked as i.i.d. uniform subsets of $\left\{ 1,\cdots,N_{\text{support}}\right\} $,
with $N_{\text{tester}}$ elements each. 

Now, if it happens that $\left\{ k_{1},\cdots,k_{N_{\text{tester}}}\right\} $
and $\left\{ \ell_{1},\cdots,\ell_{N_{\text{tester}}}\right\} $ are
disjoint, then $\ensuremath{\left\{ \ensuremath{z_{k_{1}},\cdots,z_{k_{N_{\text{}}}}}\right\} }$
are independent from $\ensuremath{\left\{ \ensuremath{z_{\ell_{1}},\cdots,z_{\ell_{N}}}\right\} }$,
and we check that the expectation above is then zero. Overall, this
means that the expression above is upper-bounded by the probability
that $\left\{ k_{1},\cdots,k_{N_{\text{tester}}}\right\} $ and $\left\{ \ell_{1},\cdots,\ell_{N_{\text{tester}}}\right\} $
have a non-zero intersection. Using a standard birthday-paradox argument,
this is at most $\frac{N^{2}}{M}$.

Overall, over the choice of $S$, the quantity 
\[
\Pr\left[\left.\text{\ensuremath{\mathcal{T}} accepts given \ensuremath{\parr{\ensuremath{\parr{z_{i_{1}},g\parr{z_{i_{1}}}},\cdots,\parr{z_{i_{N}},g\parr{z_{i_{N}}}}}}} }\right|S=\parr{\ensuremath{z_{1},\cdots,z_{M}}},\forall j_{1}\neq j_{2}:\,i_{j_{1}}\neq i_{j_{2}}\right]
\]
 has an expectation of at least $1-\delta_{2}-\frac{N^{2}}{M}$ and
standard deviation of at most $\frac{N}{\sqrt{M}}$, so by Chebyshev's
inequality it is at least $1-\delta_{2}-\frac{N^{2}}{M}-\frac{N}{\sqrt{\Delta M}}$
with probability at least $1-\Delta$. This means that with probability
at least $1-\Delta$ we have
\[
\Pr\left[\text{\ensuremath{\mathcal{T}} accepts given \ensuremath{\parr{\ensuremath{\parr{z_{i_{1}},g\parr{z_{i_{1}}}},\cdots,\parr{z_{i_{N}},g\parr{z_{i_{N}}}}}}} }\right]\geq1-\delta_{2}-\frac{N^{2}}{M}-\frac{N}{\sqrt{\Delta M}}.
\]
\end{proof}
The following lemma says that if a function is ``random enough'', then a learning algorithm will not be able to get a non-trivially small error given few example-label pairs.
\begin{lem}
\label{lem: learner fooling lemma} Let $\mathcal{A}$ be an algorithm
that takes $N$ samples $\left\{ (x_{i},y_{i})\right\} $
with $x_{i}\in\left\{ 1,\cdots,M\right\} $ and $y_{i}\in\left\{ \pm 1\right\} $
and outputs a predictor $\hat{f}:\left\{ 1,\cdots,M\right\} \righ\left\{ \pm 1\right\} $.
Let $g:\left\{ 1,\cdots,M\right\} \righ\left\{ \pm 1\right\} $ be a
random function, such that (i) $g$ has some predetermined (and possibly given to algorithm $\mathcal{A}$) values
on some fixed subset of $\left\{ 1,\cdots,M\right\} ,$ which comprises
an at most $\phi$ fraction of $\left\{ 1,\cdots,M\right\} $
(ii) $g$ is i.i.d. uniformly random in $\left\{ \pm 1\right\} $ on
the rest of $\left\{ 1,\cdots,M\right\} $. Upon receiving $N$
labeled samples $\left\{ (x_{i},g(x_{i}))\right\} $ with $\left\{ x_{i}\right\} $
distributed i.i.d. uniformly on $\left\{ 1,\cdots,M\right\} $, let
the algorithm $\mathcal{A}$ output a predictor $\hat{f}$. Then,
for sufficiently large $M$ we have
\[
\E_{g,\,\left\{ x_{i}\right\} ,\,\text{randomness of \ensuremath{\mathcal{A}}}}\left[\abs{\Pr_{x\in_{R}\left\{ 1,\cdots,M\right\} }\left[\hat{f}(x)\neq g(x)\right]-\frac{1}{2}}\right]\leq\frac{3}{2}\parr{\phi+\frac{N}{M}}+5\sqrt{\frac{\ln M}{M}}.
\]
\end{lem}

\begin{proof}
Write $\left\{ 1,\cdots,M\right\} $ as a union of two disjoint sets
$S$ and $\overline{S}$, where $S$ contains (i) the $\phi M$ or
fewer elements of $\left\{ 1,\cdots,M\right\} $ on which $g$ is
predetermined and (ii) the $N$ or fewer elements of $\left\{ 1,\cdots,M\right\} $
that the learner $\mathcal{A}$ encountered among the labeled samples
$\left\{ (x_{i},g(x_{i}))\right\} $. So, we have $\abs S\leq N+\phi M$.
We can write
\[
\Pr_{x\in_{R}\left\{ 1,\cdots,M\right\} }\left[\hat{f}(x)\neq g(x)\right]=\E_{x\in_{R}\left\{ 1,\cdots,M\right\} }\left[\indicator_{\hat{f}(x)\neq g(x)}\indicator_{x\in S}\right]+\E_{x\in_{R}\left\{ 1,\cdots,M\right\} }\left[\indicator_{\hat{f}(x)\neq g(x)}\indicator_{x\notin S}\right],
\]
which means
\begin{multline*}
\abs{\Pr_{x\in_{R}\left\{ 1,\cdots,M\right\} }\left[\hat{f}(x)\neq g(x)\right]-\frac{1}{2}}\leq\\
\abs{\E_{x\in_{R}\left\{ 1,\cdots,M\right\} }\left[\indicator_{\hat{f}(x)\neq g(x)}\indicator_{x\in S}\right]+\E_{x\in_{R}\left\{ 1,\cdots,M\right\} }\left[\indicator_{\hat{f}(x)\neq g(x)}\indicator_{x\in\overline{S}}\right]-\frac{\abs{\overline{S}}}{2M}-\frac{\abs S}{2M}}\leq\\
\abs{\E_{x\in_{R}\left\{ 1,\cdots,M\right\} }\left[\indicator_{\hat{f}(x)\neq g(x)}\indicator_{x\notin S}\right]-\frac{\abs{\overline{S}}}{2M}}+\frac{3\abs S}{2M}\leq\\
\abs{\E_{x\in_{R}\left\{ 1,\cdots,M\right\} }\left[\indicator_{\hat{f}(x)\neq g(x)}\indicator_{x\in\overline{S}}\right]-\frac{\abs{\overline{S}}}{2M}}+\frac{3}{2}\parr{\phi+\frac{N}{M}}=\\
\frac{\abs{\overline{S}}}{M}\abs{\E_{x\in_{R}\overline{S}}\left[\indicator_{\hat{f}(x)\neq g(x)}\right]-\frac{1}{2}}+\frac{3}{2}\parr{\phi+\frac{N}{M}}
\end{multline*}
Note that $\hat{f}$ depends only on (i) $S$, (ii) values of $g$
on $S$ and (iii) the internal randomness of $\mathcal{A}.$ This
means that even conditioned on $\hat{f}(x)$, the values of $g$ on
$\overline{S}$ are i.i.d. In other words, $\E_{x\in_{R}\overline{S}}\left[\indicator_{\hat{f}(x)\neq g(x)}\right]$
is distributed as the average of $\abs {\overline{S}}$ i.i.d. random variables,
each of which is uniformly random in $\left\{ 0,1\right\} $. A Hoeffding
bound argument then implies that for any $\epsilon\in[0,1]$
\[
\E_{g,\,\left\{ x_{i}\right\} ,\,\text{randomness of \ensuremath{\mathcal{A}}}}\left[\abs{\E_{x\in_{R}\overline{S}}\left[\indicator_{\hat{f}(x)\neq g(x)}\right]-\frac{1}{2}}\right]\leq\epsilon+2e^{-2\epsilon^{2}\abs{\overline{S}}},
\]
and taking $\epsilon=\sqrt{\frac{\ln\abs{\overline{S}}}{2\abs{\overline{S}}}},$we
get 
\[
\E_{g,\,\left\{ x_{i}\right\} ,\,\text{randomness of \ensuremath{\mathcal{A}}}}\left[\abs{\E_{x\in_{R}\overline{S}}\left[\indicator_{\hat{f}(x)\neq g(x)}\right]-\frac{1}{2}}\right]\leq
\underbrace{\sqrt{\frac{\ln\abs{\overline{S}}}{2\abs{\overline{S}}}}+\frac{2}{\abs{\overline{S}}}
\leq5\sqrt{\frac{\ln M}{M}}}_{\text{Since $M\geq \abs{\overline S}\geq M/2$}}.
\]
Overall, we get 
\[
\E_{g,\,\left\{ x_{i}\right\} ,\,\text{randomness of \ensuremath{\mathcal{A}}}}\left[\abs{\Pr_{x\in_{R}\left\{ 1,\cdots,M\right\} }\left[\hat{f}(x)\neq g(x)\right]-\frac{1}{2}}\right]\leq
5\sqrt{\frac{\ln M}{M}}+\frac{3}{2}\parr{\phi+\frac{N}{M}}
\]
\end{proof}

\subsection{Propositions to be used in proving Theorem \ref{thm: hardness for convex strong}.}

We will need a result about concentration the norm of an $n$-dimensional standard Gaussian. Roughly speaking, the norm is tightly concentrated within a  $O\left(n^{1/4}\right)$-neighborhood of $\sqrt{n}$. More precisely, we use
the following special case of Lemma 8.1 in \cite{birge_alternative_2001}
(this reference contains a complete short proof):
\begin{lem}
\label{lem: concentration for norm of Gaussian}Let $\vect{X}$ be a standard
$n$-dimensional Gaussian, then for any $\alpha>0$ we have

\[
\Pr\pars{\abs{\vect{X}}^{2}\geq n+2\sqrt{n\ln\parr{\frac{2}{\alpha}}}+2\ln\parr{\frac{2}{\alpha}}}\leq\frac{\alpha}{2},
\]
and

\[
\Pr\pars{\abs{\vect{X}}^{2}\leq n-2\sqrt{n\ln\parr{\frac{2}{\alpha}}}}\leq\frac{\alpha}{2}.
\]
\end{lem}

The following claim tells us that two independent Gaussian vectors are unlikely to be very close to each other.
\begin{claim}
\label{claim: random points far}Let $\vect{X}_{1}$ and $\vect{X}_{2}$ be i.i.d.
$n$-dimensional standard Gaussians. For all sufficiently large $n$,
and for any $r>0$ we have 
\[
\Pr\pars{\abs{\vect{X}_{1}-\vect{X}_{2}}\leq r}\leq8^{n}\parr{\frac{r^{2}}{n}}^{n/2}.
\]
\end{claim}

\begin{proof}
Probability density of a Gaussian is everywhere at most $\parr{\frac{1}{\sqrt{2\pi}}}^{n}$,
and the volume of a ball around $\vect{X}_{1}$ of radius $r$ is $\frac{\pi^{n/2}}{\Gamma\parr{\frac{n}{2}+1}}r^{n}$.
Stirling's approximation formula tells that for sufficiently large
$n$ we have $\Gamma\parr{\frac{n}{2}+1}\geq\sqrt{n}\parr{\frac{n}{2e}}^{n/2}$.
Therefore, for sufficiently large $n$
\[
\frac{\pi^{n/2}}{\Gamma\parr{\frac{n}{2}+1}}r^{n}\leq\frac{1}{\sqrt{n}}\parr{2e\pi}^{n/2}\parr{\frac{r^{2}}{n}}^{n/2}\leq18^{n}\parr{\frac{r^{2}}{n}}^{n/2}.
\]
Overall, the probability that $\abs{\vect{X}_{2}-\vect{X}_{1}}\leq r$ is then at
most $\parr{\frac{1}{\sqrt{2\pi}}}^{n}18^{n}\parr{\frac{r^{2}}{n}}^{n/2}\leq8^{n}\parr{\frac{r^2}{n}}^{n/2}$,
which finishes the proof.
\end{proof}
We will also need the following geometric observations for proving Theorem
\ref{thm: hardness for convex strong}.
In the following, we will use $\conv{\cdot,\cdots,\cdot}$ to denote
the convex hull of some number of objects. We will also use $\B_{r}$
to denote the ball $\left\{ x:\,\abs x\leq r\right\} $ in $\R^{n}$.  

\begin{claim}
\label{claim: if far don't see each other}Let $\vect{X}_{1}$ and $\vect{X}_{2}$
be points in $\R^{n}$ satisfying $\abs{\vect{X}_{1}},\abs{\vect{X}_{2}}\in[a,b]$ 
for some $a>0$ and $b>a$. Then, we have that if $\abs{\vect{X}_{2}-\vect{X}_{1}}$ is
greater than $2\sqrt{b^{2}-a^{2}}$, then the line segment connecting
$\vect{X}_{1}$ and $\vect{X}_{2}$ intersects $\B_{a}$.
\end{claim}

\begin{proof}
We show the claim by arguing that if $\abs{\vect{X}_{1}},\abs{\vect{X}_{2}}\in[a,b]$
and the distance between the line segment connecting $\ensuremath{\vect{X}_{1}}$
and $\ensuremath{\vect{X}_{2}}$ and origin is at least $a$, then $\abs{\vect{X}_{2}-\vect{X}_{1}}$
is at most $2\sqrt{b^{2}-a^{2}}$. If $\abs{\vect{X}_{1}}\notin\left\{ a,b\right\} $,
then one can add a small multiple of $\vect{X}_{1}-\vect{X}_{2}$ to $\vect{X}_{1}$ and
this will increase the distance $\abs{\vect{X}_{2}-\vect{X}_{1}}$, while keeping
the conditions satisfied. If $\abs{\vect{X}_{2}}\notin\left\{ a,b\right\} $,
analogous argument applies. Therefore, without loss of generality
$\abs{\vect{X}_{1}},\abs{\vect{X}_{2}}\in\left\{ a,b\right\} $. If both $\abs{\vect{X}_{1}}$
and $\abs{\vect{X}_{2}}$ equal to $a$, the segment will get closer than
$a$ to origin, unless $\vect{X}_{1}=\vect{X}_{2}$ and $\abs{\vect{X}_{2}-\vect{X}_{1}}=0$.
If both $\abs{\vect{X}_{1}}$ and $\abs{\vect{X}_{2}}$ equal to $b$, then their
distance is at most $2\sqrt{b^{2}-a^{2}}$. Finally, we need to consider
the case $\abs{\vect{X}_{1}}=a$ and $\abs{\vect{X}_{2}}=b$ (the case $\abs{\vect{X}_{1}}=b$
and $\abs{\vect{X}_{2}}=a$ is analogous). If $\vect{X}_{1}\cdot(\vect{X}_{2}-\vect{X}_{1})<0$,
then for any sufficiently small $\kappa$ we have $\abs{\vect{X}_{1}+\kappa(\vect{X}_{2}-\vect{X}_{1})}^{2}=\abs{\vect{X}_{1}}^{2}+\kappa \vect{X}_{1}\cdot(\vect{X}_{2}-\vect{X}_{1})+\kappa^{2}\abs{(\vect{X}_{2}-\vect{X}_{1})}<\abs{\vect{X}_{1}}^{2}$
, which means that $\text{dist}(\text{line segment connecting \ensuremath{\vect{X}_{1}} and \ensuremath{\vect{X}_{2}}, origin})<a$
contradicting one of the conditions. Therefore, $\vect{X}_{1}\cdot(\vect{X}_{2}-\vect{X}_{1})\geq0$.
We have 
\[
b^{2}=\abs{\vect{X}_{2}}^{2}=\abs{\vect{X}_{1}+\parr{\vect{X}_{2}-\vect{X}_{1}}}^{2}=\abs{\vect{X}_{1}}^{2}+\abs{\vect{X}_{2}-\vect{X}_{1}}^{2}+\vect{X}_{1}\cdot\parr{\vect{X}_{2}-\vect{X}_{1}}\geq a^{2}+\abs{\vect{X}_{2}-\vect{X}_{1}}^{2}.
\]
Therefore, $\abs{\vect{X}_{2}-\vect{X}_{1}}\leq\sqrt{b^{2}-a^{2}}$ in this case.
Overall across the cases, $\abs{\vect{X}_{2}-\vect{X}_{1}}$is at most $2\sqrt{b^{2}-a^{2}}$.
\end{proof}

The following claim says that if the line segment between two points $x_1$ and $x_2$ intersects the ball $\B_{a}$, then (i) the convex hull of $x_1$ and $\B_a$ (ii) the convex hull of $x_2$ and $\B_a$ have no non-trivial intersection.
\begin{claim}
\label{claim: if cant see each other cones disjoint}For any $a>0$,
let $x_{1}$ and $x_{2}$ be points in $\R^{n}$ and suppose $x_{1},x_{2}\notin\B_{a}$.
Then, if the line segment between $x_{1}$ and $x_{2}$ intersects
$\B_{a}$, then $\conv{x_{i_{1}},\B_{a}}\cap\conv{x_{i_{2}},\B_{a}}=\B_{a}$.
\end{claim}

\begin{proof}
We argue that $\conv{\vect{X}_{1},\B_{a}}\cap\conv{\vect{X}_{2},\B_{a}}\neq\B_{a}$
implies that the distance between the line segment connecting $\ensuremath{\vect{X}_{1}}$
and $\ensuremath{\vect{X}_{2}}$ and origin is greater than $a$. Indeed,
let $Z$ be a point in $\conv{\vect{X}_{1},\B_{a}}\cap\conv{\vect{X}_{2},\B_{a}}$
and not in $\B_{a}$. Then, since $\B_{a}$ is convex, the separating
hyperplane theorem tells us that there is a hyperplane separating
$Z$ from $\B_{a}$. Now, $\vect{X}_{1}$ cannot be on the same side of the
hyperplane as $\B_{a}$, because this would mean that the hyperplane
separates $Z$ from $\conv{\vect{X}_{1},\B_{a}}$. So, $\vect{X}_{1}$ has to be
on the same side of the hyperplane as $Z$ or be on the hyperplane
itself. The same argument tells us that $\vect{X}_{2}$ has to be on the
same side of the hyperplane as $Z$ or be on the hyperplane itself.
Overall, $\B_{a}$ is on one side of the hyperplane while any point
on line segment connecting $\ensuremath{\vect{X}_{1}}$ and $\ensuremath{\vect{X}_{2}}$
is either on the other side or on the hyperplane itself. Since $\B_{a}$
is closed, the distance between $\B_{a}$ and the hyperplane is positive.
This means $\text{dist}(\text{line segment connecting \ensuremath{\vect{X}_{1}} and \ensuremath{\vect{X}_{2}}, origin})>a$. 
\end{proof}
\begin{claim}
\label{claim: if can't see each other convex hull is just ball and cones}For
any $a>0$, let $\left\{ \vect{x}_{i}\right\} _{i=1}^{M}$ be a collection
of points in $\R^{n}$ and suppose $\vect{x}_{i}\notin\B_{a}$ for all $i$.
Also, suppose that for any distinct $i_{1}$ and $i_{2}$ the line
segment between $\vect{x}_{1}$ and $\vect{x}_{2}$ intersects $\B_{a}$. Then,
\[
\conv{\vect{x}_{1},\cdots,\vect{x}_{M},\B_{a}}=\conv{\vect{x}_{1},\B_{a}}\cup\conv{\vect{x}_{2},\B_{a}}\cup\cdots\cup\conv{\vect{x}_{M},\B_{a}}.
\]
\end{claim}

\begin{proof}
The line segment from $\vect{x}_{i_{1}}$ to $\vect{x}_{i_{2}}$ can be decomposed
into three contiguous nonempty disjoint regions, (i) the one in $\conv{\vect{x}_{i_{1}},\B_{a}}\setminus\B_{a}$
(ii) the one in $\B_{a}$ (iii) the one in $\conv{\vect{x}_{i_{2}},\B_{a}}\setminus\B_{a}$.
This implies the following. Let $\beta \vect{x}_{i_{1}}+(1-\beta)\vect{x}_{i_{2}}$,
with $\beta$ in $[0,1]$, be an element of this line segment. If
$\beta \vect{x}_{i_{1}}+(1-\beta)\vect{x}_{i_{2}}$ is in regions (i) or (ii) then
we can write $\beta \vect{x}_{i_{1}}+(1-\beta)\vect{x}_{i_{2}}=\gamma \vect{x}_{i_{1}}+(1-\gamma)\vect{q}$
for some $\gamma\in[0,1]$ and some $\vect{q}\in\B_{a}$. If $\beta \vect{x}_{i_{1}}+(1-\beta)\vect{x}_{i_{2}}$
is in regions (ii) or (iii) then we can write $\beta \vect{x}_{i_{1}}+(1-\beta)\vect{x}_{i_{2}}=\gamma \vect{x}_{i_{2}}+(1-\gamma)\vect{q}$
for some $\gamma\in[0,1]$ and some $\vect{q}\in\B_{a}$.

Now, clearly $\bigcup_{k}\conv{\vect{x}_{k},\B_{a}}\subseteq\conv{\vect{x}_{1},\cdots,\vect{x}_{M},\B_{a}}$,
so we only need to show the inclusion in other direction. Let $\vect{x}$
be in $\conv{\vect{x}_{1},\cdots,\vect{x}_{M},\B_{a}}$, which means that 
\begin{equation}
\vect{x}=\beta_{1}^{0}\vect{x}_{1}+\cdots\beta_{M}^{0}\vect{x}_{M}+(1-\sum_{k}\beta_{k}^{0})\vect{r}^{0}\label{eq: an element in convex hull}
\end{equation}
for some $\vect{r}^{0}\in\B_{a}$, $\beta_{k}^{0}\in[0,1]$ and satisfying
$1-\sum_{k}\beta_{k}^{0}\in[0,1]$. Take any distinct $i$ and $j$ with $\beta_{i}^{0}\neq0$
and $\beta_{j}^{0}\neq0$, then we use our earlier observation to
get that one of the cases below holds.
\[
\frac{\beta_{i}^{0}}{\beta_{i}^{0}+\beta_{j}^{0}}\vect{x}_{1}+\frac{\beta_{i}^{0}}{\beta_{i}^{0}+\beta_{i}^{0}}\vect{x}_{2}=\begin{cases}
\gamma \vect{x}_{i}+(1-\gamma)\vect{q} & \text{for some \ensuremath{\gamma\in[0,1]} and some \ensuremath{\vect{q}\in\B_{a}}, or}\\
\gamma \vect{x}_{j}+(1-\gamma)\vect{q} & \text{for some \ensuremath{\gamma\in[0,1]} and some \ensuremath{\vect{q}\in\B_{a}}}.
\end{cases}
\]
Regardless which of these cases holds, we can substitute it back in
Equation \ref{eq: an element in convex hull} and get a new expression
\[
\vect{x}=\beta{}_{1}^{1}\vect{x}_{1}+\cdots\beta{}_{M}^{1}\vect{x}_{M}+(1-\sum_{k}\beta{}_{k}^{1})\vect{r}^{1},
\]
where $\beta_{i}^{1}=0$ or $\beta_{j}^{1}=0$ and we still have $\beta_{k}\in\pars{0,1}$
for any $k$. Also, we still have $(1-\sum_{k}\beta{}_{k}^{1})\in[0,1]$
and we have $\vect{r}^{1}=\frac{(\beta_{i}^{0}+\beta_{j}^{0})(1-\gamma)\vect{q}+(1-\sum_{k}\beta_{k}^{0})\vect{r}^{0}}{(1-\sum_{i}\beta{}_{k}^{1})}$.
We check that 
\[
(\beta_{i}^{0}+\beta_{j}^{0})(1-\gamma)+(1-\sum_{k}\beta_{k}^{0})=1-\sum_{k\notin\left\{ i,j\right\} }\beta_{k}^{0}-\gamma(\beta_{i}^{0}+\beta_{j}^{0})=1-\sum_{k}\beta{}_{k}^{1},
\]
which means that $\vect{r}^{1}$ is a convex combination of $\vect{q}$ and $\vect{r}^{0}$,
and since $\vect{q},\vect{r}^{0}\in\B_{a}$ this means that $\vect{r}^{1}$ is also in
$\B_{a}$.

Now, further observe that the argument above has the following extra property:
$\beta_{k}^{0}=0$ for some $k\notin\left\{ i,j\right\} $, we also
have $\beta_{k}^{1}=0$. Therefore, if we use the argument above iteratively
to obtain values $\parr{\left\{ \beta_{k}^{2}\right\} ,\vect{r}^{2}}$, $\parr{\left\{ \beta_{k}^{3}\right\} ,\vect{r}^{3}}$
and so on, at every iteration the number of non-zero $\beta$ coefficients
decreases. We can keep iterating as long as there is a pair $\beta_{i'}^{\ell}$
and $\beta_{j'}^{\ell}$ both of which are nonzero, and we will terminate
in $M$ iterations or less. Thus, as we terminate we have 
\[
\vect{x}=\beta_{i_{0}}^{M}\vect{x}_{i_{0}}+\left(1-\beta_{i_{0}}^{M}\right)\vect{r}^{M}
\]
with $\beta_{i_{0}}^{M}\in[0,1]$ and $\vect{r}^{M}\in\B_{a}$. This means
that $\vect{x}\in\conv{\vect{x}_{i_{0}},\B_{a}}\subseteq\bigcup_{k}\conv{\vect{x}_{k},\B_{a}}$
finishing the proof.
\end{proof}

\subsection{Proofs of main hardness theorems (theorems \ref{thm: hardness for convex strong}
and \ref{thm: hardness for monotone strong}).}

\begin{proof}[Proof of Theorem \ref{thm: hardness for convex strong}]
 Let $\delta$, $\Delta$, $\alpha$ and $M$ be real-valued parameters
to be chosen later. By Lemma \ref{lem: concentration for norm of Gaussian}
we have $\Pr_{\vect x\in_{R}\mathcal{N}(0,I_{n\times n})}\left[\abs{ \vect{x}^{2}}\notin\pars{a,b}\right]\leq\alpha$,
where we denote $b=\sqrt{n+2\sqrt{n\ln\parr{\frac{2}{\alpha}}}+2\ln\parr{\frac{2}{\alpha}}}$
and $a=\sqrt{n-2\sqrt{n\ln\parr{\frac{2}{\alpha}}}}$.

We want to set our parameters in such a way that there is a distribution
$D'$ over $\R^{n}$ and a function $g:\R^{n}\righ\left\{ \pm 1\right\} $
with the following properties:
\begin{enumerate}
\item $D'$ is uniform over $M$ distinct elements $\{\vect{z}_1,\cdots,\vect{z}_M\}$ of $\R^{n}$.
\item A sample $\vect{x}$ from $D'$ with probability at least $1-2\alpha$ has
$\abs{\vect{x}}\in\left[a,b\right]$.
\item Suppose $\vect{x}_{1}$ and $\vect{x}_{2}$ belong to the support of $D'$ and both
$\abs{\vect{x}_{1}}$ and $\abs{\vect{x}_{2}}$ are in $\left[a,b\right]$. Then
$\abs{\vect{x}_{1}-\vect{x}_{2}}>2\sqrt{b^{2}-a^{2}}$ and the line segment connecting
$\vect{x}_{1}$ and $\vect{x}_{2}$ intersects $\B_{a}$.
\item Given $N$ samples of the form $(\vect{x}_{j},g(\vect{x}_{j}))$ with each $\vect{x}_{j}$
i.i.d. from $D'$, the tester $\mathcal{T}$ accepts with probability
at least $1-\delta$.
\item Given $N$ samples of the form $(\vect{x}_{j},g(\vect{x}_{j}))$ with each $\vect{x}_{j}$
i.i.d. from $D'$, the learner $\mathcal{A}$ outputs a predictor
$\hat{f}$ for which 
\[
\E_{\left\{ \vect{x}_{i}\right\} ,\,\text{randomness of \ensuremath{\mathcal{A}}}}\left[\abs{\Pr_{\vect{x}\in_{R}D'}\left[\hat{f}(\vect{x})\neq g(\vect{x})\right]-\frac{1}{2}}\right]\leq12\alpha+6\frac{N}{M}+24\sqrt{\frac{\ln M}{M}}.
\]
\end{enumerate}
Let $D'$ be uniform over a multiset $S:=\left\{ \vect{z}_{1},\cdots,\vect{z}_{M}\right\} $
of elements drawn i.i.d. uniformly from $\mathcal{N}(0,I_{n\times n})$.
Let $g$ be a random function over $\R^{n}$ picked as follows:
\begin{itemize}
\item If $\left|\vect{x}\right|>b,$then $g(\vect{x})=0$.
\item If $\left|\vect{x}\right|<a$, then $g(\vect{x})=1$.
\item If $\abs{\vect{x}}\in[a,b]$, then $g(\vect{x})$ is chosen randomly in $\left\{ \pm 1\right\} $
subject to the following conditions.
\begin{itemize}
\item For every $\vect{x}$, we have $\Pr\left[g(\vect{x})=0\right]=\Pr\left[g(\vect{x})=1\right]=\frac{1}{2}$.
\item For any collection of $\left\{ \vect{x}_{i}\right\} $, such that any two
distinct $\vect{x}_{i}$ and $\vect{x}_{j}$ are further away\footnote{The exact value of $\sqrt{b^{2}-a^{2}}$ here does not matter. We
could have taken it to be anything smaller than $2\sqrt{b^{2}-a^{2}}$.} from each other than $\sqrt{b^{2}-a^{2}}$, then $\left\{ g(\vect{x}_{i})\right\} $
is a collection of i.i.d. random variables uniform on $\left\{ \pm 1\right\} $. 
\end{itemize}
\end{itemize}
One way to give an explicit construction of random function $g$ satisfying
conditions above is to break the region $\left\{ \vect{x}\in\R^{d}:\abs{ \vect{x}}\in[a,b]\right\} $
into finitely many disjoint parts of diameter at most $\sqrt{b^{2}-a^{2}}$
and have $g$ be i.i.d. uniformly random in $\left\{ \pm 1\right\} $
on each of these parts. Then, we have 
\begin{enumerate}
\item Condition 1 is satisfied with probability $1$, because $\mathcal{N}(0,I_{n\times n})$
has continuous density. 
\item By Lemma \ref{lem: concentration for norm of Gaussian}, for $\vect{x}$
drawn from $\mathcal{N}(0,I_{n\times n})$, the probability that $\abs{ \vect{x}}\notin\left[a,b\right]$
is at most $\alpha$. Then, another application of the standard Hoeffding
bound shows that out of $\left\{ \vect{z}_{1},\cdots,\vect{z}_{M}\right\} $, the
fraction with norm outside of $\left[a,b\right]$
is at most $2\alpha$ with probability at most $e^{-2\alpha^{2}M}$.
In other words, Condition 2 is satisfied with probability at least
$1-e^{-2\alpha^{2}M}$.
\item Claim \ref{claim: random points far} tells us that distinct
$i$ and $j$ the probability that $\abs{\vect{z}_{i}},\abs{\vect{z}_{j}}\in\left[a,b\right]$
and $\abs{\vect{z}_{i}-\vect{z}_{j}}\leq2\sqrt{b^{2}-a^{2}}$ is at most $8^{n}\parr{\frac{b^{2}-a^{2}}{n}}^{n/2}$.
Claim \ref{claim: if far don't see each other} then tells us that if $\abs{\vect{z}_{i}-\vect{z}_{j}}>2\sqrt{b^{2}-a^{2}}$, then
the line segment connecting $\vect{z}_{i}$ and $\vect{z}_{j}$ intersects $\B_{a}$.
Taking a union bound over all such distinct pairs $(\vect{z}_{i},\vect{z}_{j})$,
the probability of the Condition 3 being violated is at most $1-8^{n}\parr{\frac{4\sqrt{n\ln\parr{\frac{2}{\alpha}}}+2\ln\parr{\frac{2}{\alpha}}}{n}}^{n/2}M^{2}$. 
\item Via Lemma \ref{lem: tester fooling lemma} we see that with probability
at least $1-\Delta$ we have that given $N$ samples of the form $(\vect{x}_{j},g(\vect{x}_{j}))$
with each $\vect{x}_{j}$ i.i.d. from $D'$, the tester $\mathcal{T}$ accepts
with probability at least $1-\delta_{2}-\frac{N^{2}}{M}-\frac{N}{\sqrt{\Delta M}}$.
So, to satisfy Condition 4, we need that $\delta-\delta_{2}>\frac{N^{2}}{M}$
and $\Delta=\frac{N^{2}}{M}\frac{1}{\parr{\delta-\delta_{2}-\frac{N^{2}}{M}}^{2}}$.
\item For any distinct $\vect{z}_{i}$ and $\vect{z}_{j}$ satisfying $\abs{\vect{z}_{i}},\abs{\vect{z}_{j}}\in[a,b]$,
Condition 3 tells us that $\abs{\vect{z}_{i}-\vect{z}_{j}}>2\sqrt{b^{2}-a^{2}}$.
The way random function $g$ was constructed then implies that the
random variables $\left\{ g(\vect{z}_{i}),\,i:\,\abs{\vect{z}_{i}}\in[a,b]\right\} $
is a collection of i.i.d. random variables uniform in $\left\{ \pm 1\right\} $.
We can therefore use Lemma \ref{lem: learner fooling lemma} as long
as $2\alpha<\frac{1}{4}$ and $N\leq\frac{M}{4}$ and 
\[
\E_{g,\,\left\{ \vect{x}_{i}\right\} ,\,\text{randomness of \ensuremath{\mathcal{A}}}}\left[\abs{\Pr_{\vect{x}\sim D'}\left[\hat{f}(\vect{x})\neq g(\vect{x})\right]-\frac{1}{2}}\right]\leq\frac{3}{2}\parr{2\alpha+\frac{N}{M}}+5\sqrt{\frac{\ln M}{M}}.
\]
Therefore, with probability at least $1-\frac{1}{4}$ over the choice
of $g$ we have 
\[
\E_{\left\{ \vect{x}_{i}\right\} ,\,\text{randomness of \ensuremath{\mathcal{A}}}}\left[\abs{\Pr_{\vect{x}\sim D'}\left[\hat{f}(\vect{x})\neq g(\vect{x})\right]-\frac{1}{2}}\right]\leq12\alpha+6\frac{N}{M}+20\sqrt{\frac{\ln M}{M}}.
\]
\end{enumerate}
Overall, the probability that all five of the conditions hold is non-zero
as long as $\delta-\delta_{2}>\frac{N^{2}}{M}$
and
\begin{equation}
e^{-2\alpha^{2}M}+8^{n}\parr{\frac{4\sqrt{n\ln\parr{\frac{2}{\alpha}}}+2\ln\parr{\frac{2}{\alpha}}}{n}}^{n/2}M^{2}+\frac{N^{2}}{M}\frac{1}{\parr{\delta-\delta_{2}-\frac{N^{2}}{M}}^{2}}+\frac{1}{4}<1.\label{eq: constraint on parameters for convex hardness}
\end{equation}

From now on we fix $g$ and $D'$ assuming the five conditions above
hold (we will check that the Equation \ref{eq: constraint on parameters for convex hardness} indeed holds when we pick our parameters). We claim that there is a function $f_{0}:\R^{n}\righ\left\{ \pm 1\right\} $
such that (i) $\left\{ \vect{x}:\,f_{0}(\vect{x})=1\right\} $ is a convex set (ii)
$\Pr_{\vect{x}\sim D'}\pars{f_{0}(\vect{x})=g(\vect{x})}=1$ (even though the function
$g$ itself is very likely not indicator of a convex body). Recall
that $D'$ was uniform from $S:=\left\{ \vect{z}_{1},\cdots,\vect{z}_{M}\right\} $
so we define $f_{0}$ to be $1$ on $\conv{\B_{a},\left\{ \vect{z}_{i}:g(\vect{z}_{i})=1\right\} }$
and $0$ otherwise. Property (i) is immediate from the definition
of $f_{0}$. To show property (ii), recall that $D'$ is supported
on $\left\{ \vect{z}_{i}\right\} $, so we need to show that $f_{0}(\vect{z}_{i})=g(\vect{z}_{i})$
for every $i$. 
\begin{itemize}
    \item If $g(\vect{z}_{i})=1$, from definition of $f_{0}$ it is
immediate that $f_{0}(\vect{z}_{i})=g(\vect{z}_{i})$.
    \item If $g(\vect{z}_{i})=0$, we argue
as follows. By Claim \ref{claim: if cant see each other cones disjoint}
we know that for any $j\neq i$ we have $\conv{\vect{z}_{i},\B_{a}}\cap\conv{z_{j},\B_{a}}=\B_{a}$
which in particular implies $\vect{z}_{i}\notin\conv{\vect{z}_{j},\B_{a}}$. So,
$\vect{z}_{i}$ is not in $\bigcup_{j:\,g(\vect{z}_{j})=1}\conv{\vect{z}_{j},\B_{a}}$,
but $\bigcup_{j:\,g(\vect{z}_{j})=1}\conv{\vect{z}_{j},\B_{a}}=\conv{\left\{ \vect{z}_{j}:g(\vect{z}_{j})=1\right\} ,\B_{a}}$
by Claim \ref{claim: if can't see each other convex hull is just ball and cones},
so $\vect{z}_{i}\notin\conv{\left\{ \vect{z}_{j}:g(\vect{z}_{j})=1\right\} ,\B_{a}}$ and
therefore $f_{0}(\vect{z}_{i})=0$ as required.
\end{itemize}

Finally, we get to picking the parameters. Recall that $N=2^{0.01n}$.
 We take $M=2^{0.1n}$ and $\delta=\delta_{2}+\frac{N^{2}}{M}+100\frac{N}{\sqrt{M}}=\delta_{2}+100\frac{N}{2^{0.05n}}+\frac{N^{2}}{2^{0.1n}}$,
which allows us to conclude that the tester $\mathcal{T}$, given
samples $(\vect{x},g(\vect{x}))$ with $\vect{x}\sim D'$, accepts with probability at
least $1-\delta_{2}-100\frac{N}{2^{0.05n}}-\frac{N^{2}}{2^{0.1n}}=1-\delta_{2}-\frac{1}{2^{\Omega(n)}}$. We proceed to making sure Equation \ref{eq: constraint on parameters for convex hardness} is satisfied:
\begin{itemize}
    \item We see that $\frac{N^{2}}{M}\frac{1}{\parr{\delta-\delta_{2}-\frac{N^{2}}{M}}^{2}}=\frac{N^{2}}{M}\frac{1}{10000\frac{N^{2}}{M}}=\frac{1}{10000}$.
    \item By taking $\alpha=\frac{1}{2}e^{-\frac{n}{160000}}$, we make sure that now $e^{-2\alpha^{2}M}+8^{n}\parr{\frac{4\sqrt{n\ln\parr{\frac{2}{\alpha}}}+2\ln\parr{\frac{2}{\alpha}}}{n}}^{n/2}M^{2}=2^{-\Omega(n)}$,
so taking $n$ sufficiently large we can make this expression as small
as we want. 
\end{itemize}
Thus, Equation \ref{eq: constraint on parameters for convex hardness} indeed
holds for sufficiently large $n$ for our choice of the parameters.
We see that our choice of parameters also satisfies the required
condition $\delta-\delta_{2}>\frac{N^{2}}{M}$.
Condition 5 tells that the expected advantage of the predictor $\hat{f}$
is at most 
\[
12\alpha+6\frac{N}{M}+20\sqrt{\frac{\ln M}{M}}=6e^{-\frac{n}{160000}}+6\frac{2^{0.01n}}{2^{0.1n}}+\frac{20\sqrt{0.1n}}{2^{0.005n}}=2^{-\Omega\parr n}.
\]
\end{proof}
Now, let's prove our theorem about hardness of testable agnostic learning
of monotone functions.

\begin{proof}[Proof of Theorem \ref{thm: hardness for monotone strong}]
Let $\delta$, $\Delta$, $\alpha$ and $M$ be real-valued parameters
to be chosen later. Observe that we have $\Pr_{\vect x\in\left\{ 0,1\right\} ^{n}}\left[\abs{\vect x}\notin\pars{\frac{n}{2}-\sqrt{\frac{n}{2}\ln\frac{2}{\alpha}}}\right]\leq\alpha$,
and denote $h_{\alpha}:=\sqrt{\frac{n}{2}\ln\frac{2}{\alpha}}$. We
want to set our parameters in such a way that there is a distribution
$D'$ over $\left\{ 0,1\right\} ^{n}$ and a function $g:\left\{ 0,1\right\} ^{n}\righ\left\{ \pm1\right\} $
with the following properties:
\begin{enumerate}
\item $D'$ is uniform over $M$ distinct elements $\{\vect{z}_1,\cdots,\vect{z}_M\}$ of $\left\{ 0,1\right\} ^{n}$.
\item A sample $\vect{x}$ from $D'$ with probability at least $1-2\alpha$ has
hamming weight in $\left[\frac{n}{2}-h_{\alpha},\frac{n}{2}+h_{\alpha}\right]$.
\item Suppose $\vect{x}_{1}$ and $\vect{x}_{2}$ belong to the support of $D'$ and both
$\vect{x}_{1}$ and $\vect{x}_{2}$ have hamming weight in $\left[\frac{n}{2}-h_{\alpha},\frac{n}{2}+h_{\alpha}\right]$.
Then $\vect{x}_{1}$ and $\vect{x}_{2}$ are incomparable (i.e. neither one dominates
the other one bit-wise). 
\item Given $N$ samples of the form $(\vect{x}_{j},g(\vect{x}_{j}))$ with each $\vect{x}_{j}$
i.i.d. from $D'$, the tester $\mathcal{T}$ accepts with probability
at least $1-\delta$.
\item Given $N$ samples of the form $(\vect{x}_{j},g(\vect{x}_{j}))$ with each $\vect{x}_{j}$
i.i.d. from $D'$, the learner $\mathcal{A}$ outputs a predictor
$\hat{f}$ for which 
\[
\E_{\left\{ \vect{x}_{i}\right\} ,\,\text{randomness of \ensuremath{\mathcal{A}}}}\left[\abs{\Pr_{\vect{x}\in_{R}D'}\left[\hat{f}(\vect{x})\neq g(\vect{x})\right]-\frac{1}{2}}\right]\leq12\alpha+6\frac{N}{M}+20\sqrt{\frac{\ln M}{M}}.
\]
\end{enumerate}
We use the probabilistic method to show the existence of such $D'$
and $g.$ Let $D'$ be uniform over a multiset $S:=\left\{ \vect{z}_{1},\cdots,\vect{z}_{M}\right\} $
of elements drawn i.i.d. uniformly from $\left\{ 0,1\right\} ^{n}$.
Let $g$ be a random function over $\left\{ 0,1\right\} ^{n}$ picked
as
\[
g(\vect x)=\begin{cases}
1 & \text{if \ensuremath{\left|\vect x\right|>n/2+h_{\alpha},}}\\
-1 & \left|\vect{x}\right|<n/2-h_{\alpha},\\
\text{i.i.d. uniformly random in \ensuremath{\left\{ \pm1\right\} }} & \text{otherwise.}
\end{cases}
\]
 Then, we have 
\begin{enumerate}
\item Condition 1 is satisfied with probability at least $1-\frac{M^{2}}{2^{n}}$
by a standard birthday-paradox argument. 
\item By the standard Hoeffding bound, a uniform sample from $\left\{ 0,1\right\} ^{n}$
falls outside of $\left[\frac{n}{2}-h_{\alpha},\frac{n}{2}+h_{\alpha}\right]$
with probability at most $2e^{-\frac{2h_{\alpha}^{2}}{n}}=\alpha$.
Then, another application of the standard Hoeffding bound shows that
out of $\left\{ z_{1},\cdots,z_{M}\right\} $, the fraction with Hamming
weight outside of $\left[\frac{n}{2}-h_{\alpha},\frac{n}{2}+h_{\alpha}\right]$
is at most $2\alpha$ with probability at most $e^{-2\alpha^{2}M}$.
In other words, Condition 2 is satisfied with probability at least
$1-e^{-2\alpha^{2}M}$.
\item For distinct $i_{1}$ and $i_{2}$, we bound the probability probability
of the event that (i) $\vect{z}_{i_{1}}$ and $\vect{z}_{i_{2}}$ have Hamming weight
in $\left[\frac{n}{2}-h_{\alpha},\frac{n}{2}+h_{\alpha}\right]$ and
(ii) $\vect{z}_{i_{1}}$ dominates $\vect{z}_{i_{2}}$ bit-wise. Suppose $\vect{z}_{i_{1}}$ indeed
has Hamming weight in $\left[\frac{n}{2}-h_{\alpha},\frac{n}{2}+h_{\alpha}\right]$,
then there are only at most $n^{2h_{\alpha}}$ possible candidates
for $\vect{z}_{i_{2}}$ that will make the event to take place. Thus, the
probability of this event is at most $\frac{n^{2h_{\alpha}}}{2^{n}}=\frac{n^{\sqrt{2n\ln\frac{2}{\alpha}}}}{2^{n}}$.
Taking a union bound over all distinct pairs $(\vect{z}_{i_{1}},\vect{z}_{i_{2}})$,
the probability of the Condition 3 being violated is at most $1-\frac{n^{\sqrt{2n\ln\frac{2}{\alpha}}}}{2^{n}}M^{2}$.
\item Via Lemma \ref{lem: tester fooling lemma} we see that with probability
at least $1-\Delta$ we have that given $N$ samples of the form $(\vect{x}_{j},g(\vect{x}_{j}))$
with each $x_{j}$ i.i.d. from $D'$, the tester $\mathcal{T}$ accepts
with probability at least $1-\delta_{2}-\frac{N^{2}}{M}-\frac{N}{\sqrt{\Delta M}}$.
So, to satisfy Condition 4, we need that $\delta-\delta_{2}>\frac{N^{2}}{M}$
and $\Delta=\frac{N^{2}}{M}\frac{1}{\parr{\delta-\delta_{2}-\frac{N^{2}}{M}}^{2}}$.
\item Via Lemma \ref{lem: learner fooling lemma} we have 
\[
\E_{g,\,\left\{ \vect{x}_{i}\right\} ,\,\text{randomness of \ensuremath{\mathcal{A}}}}\left[\abs{\Pr_{\vect{x}\sim D'}\left[\hat{f}(\vect{x})\neq g(\vect{x})\right]-\frac{1}{2}}\right]\leq\frac{3}{2}\parr{2\alpha+\frac{N}{M}}+6\sqrt{\frac{\ln M}{M}}.
\]
Therefore, with probability at least $1-\frac{1}{4}$ over the choice
of $g$ we have 
\[
\E_{\left\{ \vect{x}_{i}\right\} ,\,\text{randomness of \ensuremath{\mathcal{A}}}}\left[\abs{\Pr_{\vect{x}\sim D'}\left[\hat{f}(\vect{x})\neq g(\vect{x})\right]-\frac{1}{2}}\right]\leq12\alpha+6\frac{N}{M}+24\sqrt{\frac{\ln M}{M}}.
\]
\end{enumerate}
Overall, the probability that all five of the conditions hold is non-zero
as long as $\delta-\delta_{2}>\frac{N^{2}}{M}$
and
\begin{equation}
\frac{M^{2}}{2^{n}}+e^{-2\alpha^{2}M}+\frac{n^{\sqrt{2n\ln\frac{2}{\alpha}}}}{2^{n}}M^{2}+\frac{N^{2}}{M}\frac{1}{\parr{\delta-\delta_{2}-\frac{N^{2}}{M}}^{2}}+\frac{1}{4}<1.\label{eq: constraint on parameters for monotone hardness}
\end{equation}

From now on, we assume that the five conditions above hold (we will check that the Equation \ref{eq: constraint on parameters for monotone hardness} indeed holds when we pick our parameters).
We claim that there
is a monotone $f_{0}:\left\{ 0,1\right\} ^{n}\righ\left\{ \pm 1\right\} $
for which $\Pr_{\vect{x}\sim D'}\pars{f_{0}(\vect{x})=g(\vect{x})}=1$ (even though the
function $g$ itself is very likely not monotone). Recall that $D'$
was uniform from $S:=\left\{ \vect{z}_{1},\cdots,\vect{z}_{M}\right\} $ so we write
\[
f_{0}(\vect{x})=\begin{cases}
1 & \text{if \ensuremath{\left|\vect{x}\right|>n/2+h_{\alpha},}}\\
-1 & \left|\vect{x}\right|<n/2-h_{\alpha},\\
g(\vect{x}) & \text{if \ensuremath{|\vect{x}|\in\left[n/2-h_{\epsilon},n/2+h_{\epsilon}\right]} and \ensuremath{\vect{x}} is in the support of \ensuremath{D'},}\\
-1 & \text{if \ensuremath{|\vect{x}|\in\left[n/2-h_{\epsilon},n/2+h_{\epsilon}\right]} and \ensuremath{\vect{x}} is dominated by some \ensuremath{\vect{y}} in the support of \ensuremath{D'} s.t. \ensuremath{g(\vect{y})=-1},}\\
1 & \text{if \ensuremath{|\vect{x}|\in\left[n/2-h_{\epsilon},n/2+h_{\epsilon}\right]} and \ensuremath{\vect{x}} dominates some \ensuremath{\vect{y}} in the support of \ensuremath{D'} s.t. \ensuremath{g(\vect{y})=+1},}\\
-1 & \text{otherwise.}
\end{cases}
\]
The definition above is not self-contradictory, because Condition
3 says if $\vect{x}_{1}$ and $\vect{x}_{2}$ belong to the support of $D'$ and
both $\vect{x}_{1}$ and $\vect{x}_{2}$ have hamming weight in $\left[\frac{n}{2}-h_{\alpha},\frac{n}{2}+h_{\alpha}\right]$,
then $\vect{x}_{1}$ and $\vect{x}_{2}$ are incomparable. We see that $f_{0}(\vect{x})$
is indeed monotone and agrees with $g$ on the support of $D$.

Finally, we get to picking the parameters. Recall that $N=2^{0.01n}$.
We take $M=2^{0.1n}$ and $\delta=\delta_{2}+\frac{N^{2}}{M}+100\frac{N}{\sqrt{M}}=\delta_{2}+100\frac{N}{2^{0.05n}}+\frac{N^{2}}{2^{0.1n}}$,
which allows us to conclude that the tester $\mathcal{T}$, given
samples $(\vect{x},g(\vect{x}))$ with $\vect{x}\sim D'$, accepts with probability at
least $1-\delta_{2}-100\frac{N}{2^{0.05n}}-\frac{N^{2}}{2^{0.1n}}=1-\delta_{2}-\frac{1}{2^{\Omega(n)}}$.
We proceed to making sure Equation \ref{eq: constraint on parameters for monotone hardness} is satisfied:
\begin{itemize}
    \item We see that $\frac{N^{2}}{M}\frac{1}{\parr{\delta-\delta_{2}-\frac{N^{2}}{M}}^{2}}=\frac{N^{2}}{M}\frac{1}{10000\frac{N^{2}}{M}}=\frac{1}{10000}$.
    \item Taking $\alpha=\frac{1}{2}e^{-0.1\frac{n}{\log^{2}n}}$, we see that now $\frac{M^{2}}{2^{n}}+e^{-2\alpha^{2}M}+\frac{n^{\sqrt{2n\ln\frac{2}{\alpha}}}}{2^{n}}M^{2}=2^{-\Omega(n)}$,
so taking $n$ sufficiently large we can make this expression as small
as we want.
\end{itemize}
Thus, Equation \ref{eq: constraint on parameters for monotone hardness}
holds for sufficiently large $n$ for our choice of the parameters.
We see that our choice of parameters also satisfies the required
condition $\delta-\delta_{2}>\frac{N^{2}}{M}$.
Condition 5 tells that the expected advantage of the predictor $\hat{f}$
is at most 
\[
12\alpha+6\frac{N}{M}+20\sqrt{\frac{\ln M}{M}}=6e^{-0.1\frac{n}{\log^{2}n}}+6\frac{2^{0.01n}}{2^{0.1n}}+\frac{20\sqrt{0.1n}}{2^{0.005n}}=2^{-\Omega\parr{\frac{n}{\log^{2}n}}}.
\]
\end{proof}

\section{Acknowledgements.}
We wish to thank Jonathan Kelner and Pravesh Kothari for their useful comments and references. We also thank anonymous referees for their comments.
\bibliographystyle{alpha}
\bibliography{references}

\appendix
\section{Omitted proofs.}
\subsection{Improvement of error probabilities for a tester-learner pair via repetition.}
\label{appendix: improving error probabilitites}
If the constants $1/4$ and $3/4$ in the Definition \ref{def: tester-learner pair} are replaced by some other constants $1-\delta_2$ and $1-\delta_3$ with $\delta_3 \in (\delta_2,1)$, then we say that $\mathcal{T}$ is a 
\emph{$\parr{\delta_{2},\delta_{3}}$-tester
for the distributional assumption} of $\mathcal{A}$. The following proposition tells us that that taking $\delta_2=1/4$ and $\delta_3=3/4$ is without loss of generality.
\begin{prop}
\label{prop: improving tester via repetition}Let $\delta_{1},\delta_{2},\epsilon\in\parr{0,1}$,
$\delta_{3}\in\parr{\delta_{2},1}$ and let $\mathcal{A}$ be an agnostic
$(\epsilon,\delta_{1})$-learner for function class $\mathcal{F}$
relative to the distribution $D$, and $\mathcal{T}$ be a \emph{$\parr{\delta_{2},\delta_{3}}$}-tester
for the distributional assumption of $\mathcal{A}$. Then, for every
integer $r\geq$1 there is a \emph{$\parr{2\exp\parr{-\frac{2\parr{\delta_{3}-\delta_{2}}^{2}r}{9}},1-3\exp\parr{-\frac{2\parr{\delta_{3}-\delta_{2}}^{2}r}{9}}}$}-tester
$\mathcal{T}'$ for the distributional assumption of $\mathcal{A}$
that consumes only $O\parr r$ times as much samples and run-time
as $\mathcal{T}$. 
\end{prop}

\begin{proof}
The tester $\mathcal{T}'$ is constructed by (i) repeating $\mathcal{T}$
$r$ times (ii) if the fraction of ``Yes'' answers is at least $1-\frac{\delta_{2}+\delta_{3}}{2}$,
then output ``Yes'', otherwise output ``No''. By Hoeffding's bound
with probability at least $1-2\exp\parr{-\frac{2\parr{\delta_{3}-\delta_{2}}^{2}r}{9}}$,
the fraction of ``Yes'' answers observed is within $\frac{\delta_{3}-\delta_{2}}{3}$
of the true probability that $\mathcal{T}$ outputs ``Yes''. So, 
\begin{itemize}
\item Recall that, given access to samples from $D_{\text{pairs}}$, the algorithm $\mathcal{T}$
outputs ``Yes'' with probability at least $1-\delta_{2}$. Therefore,
given access to samples from $D_{\text{pairs}}$, the algorithm $\mathcal{T}'$ outputs
``Yes'' with probability at least $1-2\exp\parr{-\frac{2\parr{\delta_{3}-\delta_{2}}^{2}r}{9}}$.
\item Suppose, given access to samples from $D_{\text{pairs}}$, the algorithm $\mathcal{T}$
outputs ``Yes'' with probability $p$. Then, if $p<1-\delta_{3}$,
the algorithm $\mathcal{T}'$ can output ``Yes'' with probability
only at most $2\exp\parr{-\frac{2\parr{\delta_{3}-\delta_{2}}^{2}r}{9}}$.
Therefore, if the algorithm $\mathcal{T}'$ outputs ``Yes'' with
probability at least $3\exp\parr{-\frac{2\parr{\delta_{3}-\delta_{2}}^{2}r}{9}}$,
it has to be the case that the algorithm $\mathcal{T}$ outputs ``Yes''
with probability at least $1-\delta_{3}$. 
The composability condition then tells us that $\mathcal{A}$ will then satisfy the required bound on the generalization error when run on samples from $D_{\text{pairs}}$.
\end{itemize}
\end{proof}

\subsection{Proof of Observation \ref{obs: chebychev projection has small coefficients}.}
\label{appendix: proof of obs: chebychev projection has small coefficients}
Since for $y\in[-1,1]$ both $f(wy)$ and $T_{k}(y)$ are also in
$[-1,1]$, we have that\footnote{Proof: we have $\abs{a_{k}}=\abs{\frac{1+\indicator_{k>0}}{\pi}\int_{-1}^{1}\frac{f(wy)T_{k}(y)}{\sqrt{1-y^{2}}}\d y}\leq \frac{2}{\pi} \int_{-1}^{1} \frac{1}{\sqrt{1-y^{2}} }\d y=4$, where the integral in the end is evaluated via a standard substitution of $y=\cos(\alpha)$.} all $a_{k}$ are in $[-4,4]$. We also see
that that the largest coefficient among all the monomials of $T_{k}(y)$
is at most $3^{k}$ (this follows by induction via the recursive relation
$T_{k+1}(x)=2xT_{k}(x)-T_{k-1}(x)$). Since $w\geq1$, the largest
coefficient among all the monomials of $T_{k}\parr{\frac{y}{w}}$
is also at most $3^{k}$. Thus, the largest coefficient of $f_{d}(x):=\sum_{k=0}^{d}a_{k}T_{k}(\frac{x}{w})$
can only be at most $O\parr{d3^{d}}$.

\subsection{Proof of Proposition \ref{prop:approximation outside window, assuming moment bound}.}
\label{appendix: proof of prop:approximation outside window, assuming moment bound}

We have
\begin{multline*}
\E_{x\in_{R}D}\pars{\abs{f(x)-f_{d}(x)}\indicator_{\abs x>w}}\overbrace{\leq\E_{x\in_{R}D}\pars{\parr{1+\abs{f_{d}(x)}}\indicator_{\abs x>w}}}^{\text{Since \ensuremath{\abs{f(x)}}\ensuremath{\ensuremath{\leq}1.}}}\leq\\
\overbrace{\leq\E_{x\in_{R}D}\pars{\indicator_{\abs x>w}}+O\parr{d3^{d}}\sum_{k=0}^{d}\E_{x\in_{R}D}\pars{\abs x^{k}\indicator_{\abs x>w}}}^{\text{Breaking \ensuremath{f_{d}} into monomials, then using triangle inequality and Observation \ref{obs: chebychev projection has small coefficients} }}\\
\leq 
\underbrace{O\parr{4^{d}\max_{0\leq k\leq d}\E_{x\in_{R}D}\pars{\abs x^{k}\indicator_{\abs x>w}}}=O\parr{4^{d}\E_{x\in_{R}D}\pars{\abs x^{d}\indicator_{\abs x>w}}}}_{\text{Since $w\geq 1$, when $\abs{x}>w$ the value of $\abs{x}^k$ grows with $k$.}}
\end{multline*}

\subsection{Proof of Proposition \ref{prop: moment tail bound for bounded moment distributions}.}
\label{appendix: proof of prop: moment tail bound for bounded moment distributions}

Applying Markov's inequality to $\abs x^{d_{0}}$, we have 
\[
\Pr\pars{\abs x\geq\tau}\leq\parr{\frac{\beta}{\tau}}^{d_{0}}.
\]
The above covers the case when $k=0$. When $k>0$ we proceed by using the inequality above as follows, 
\begin{multline*}
\E_{x\in_{R}D}\pars{\abs x^{k}\indicator_{\abs{x}>w}}
=
\E_{x\in_{R} D}\pars{ w^k \indicator_{\abs{x}>w}+ 
\int_{\tau=w}^{\infty} k\tau^{k-1}  \indicator_{\abs{x}>\tau} \d t
} =
\\
\underbrace{
=w^{k}\Pr_{x\in_{R}D}\pars{|x|>w}+\int_{w}^{\infty}k\tau^{k-1}\Pr_{x\in_{R}D}\pars{|x|>\tau}\d\tau
}_{\text{Via linearity of expectation and Tonelli's theorem.}}\leq\\
w^{k}\parr{\frac{\beta}{w}}^{d_{0}}+\int_{w}^{\infty}k\tau^{k-1}\parr{\frac{\beta}{\tau}}^{d_{0}}\d\tau=w^{k}\parr{\frac{\beta}{w}}^{d_{0}}+\frac{k}{d_{0}-k}w^{k}\parr{\frac{\beta}{w}}^{d_{0}}=w^{k}\parr{\frac{\beta}{w}}^{d_{0}}\frac{d_{0}}{d_{0}-k}\leq2w^{k}\parr{\frac{\beta}{w}}^{d_{0}}.
\end{multline*}

\subsection{Proof of Observation \ref{obs: similar to gaussian means projection is similar to gaussian}.}
\label{appendix: proof of obs: similar to gaussian means projection is similar to gaussian}
Without loss of generality, assume $\Delta=\frac{1}{\epsilon^{4}}\ln^{4}\parr{\frac{1}{\epsilon}}$.

We have
\begin{multline*}
\abs{\E_{\vect x\sim D}\pars{\parr{\vect v\cdot\vect x}^{d}}-\E_{\vect x\sim\N(0,I_{n\times n})}\pars{\parr{\vect v\cdot\vect x}^{d}}}=\\
\abs{\sum_{\substack{\alpha_{1},\cdots\alpha_{n}\in\Z_{\geq0}\\
\alpha_{1}+\cdots+\alpha_{n}=d
}
}\parr{\E_{\vect x\sim D}\pars{\prod_{i=1}^{n}\left(v_{i}x_{i}\right)^{\alpha_{i}}}-\E_{\vect x\sim\N(0,I_{n\times n})}\pars{\prod_{i=1}^{n}\left(v_{i}x_{i}\right)^{\alpha_{i}}}}}\leq\\
\underbrace{\sum_{\substack{\alpha_{1},\cdots\alpha_{n}\in\Z_{\geq0}\\
\alpha_{1}+\cdots+\alpha_{n}=d
}
}\parr{\abs{\prod_{i=1}^{n}v_{i}^{\alpha_{i}}}\frac{1}{n^{\Delta}}}\leq n^{d}\frac{1}{n^{\Delta}}.}_{\text{Using \ensuremath{\abs{v_{i}}\leq}1 and bounding the number of \ensuremath{\left\{ \alpha_{i}\right\} }}}
\end{multline*}

\subsection{Tester-learner pair for decision lists}
\label{appendix: proof for decision lists}
First, recall the definition of a decision list (a more general definition is given in \cite{Rivest87}):
\begin{defn}
\label{definition: decision list}
For some ordering of the
variables $x_{\pi(1)},
\ldots, x_{\pi(n)}$, values $v_1,\ldots,v_n \in \{\pm 1\}$ and bits $b_1,\ldots,b_n \in \{\pm 1\}$,
~a \textbf{decision list} does the following:  For $i = 1 {\rm ~to~} n$,
if $x_{\pi(i)}=b_{\pi(i)}$ output $v_{\pi(i)}$, else continue. If the decision list reaches the end of execution without outputting anything, it outputs $0$.
\end{defn}
Now, we will present the tester-learner pair for decision lists. The tester will check that the distribution on examples is close to $k$-wise independent. The insight behind the learning algorithm is that any decision list is well-approximated by a short decision list if the distribution on examples is close to $k$-wise independent. 

\noindent
{\bf Tester-learner pair for learning decision lists over $\{0,1\}^n$:}
\begin{itemize}
\item Define $k:=\log \frac{1}{\epsilon}$.
\item \textbf{Learning algorithm $\mathcal{A}_{\text{DL}}$.} Given access to i.i.d.
labeled samples $(\vect x,y)\in\{\pm 1\}^{n}\times\left\{ \pm1\right\} $
from an unknown distribution:
\begin{itemize}
\item Take $\frac{100}{ \epsilon^2} k^3 \log n$ samples $\{(\vect x_i,y_i)\}$.
\item Enumerate over all functions $f$ that can be represented as decision lists on any size-$k$ subset $S$ of $\{1,\cdots n\}$:
Compute the fraction of example label pairs on which $f$ gives the wrong answer. Denote it as $\widehat{\text{err}}(f)$. 
\item Among the functions just considered, output the function $f_0$ that fits best. In other words, the function for which $\widehat{\text{err}}(f)$ was smallest.
\end{itemize}
\item \textbf{Testing algorithm $\mathcal{T}_{\text{DL}}$}. Given access to i.i.d.
labeled examples $\vect x\in \{\pm 1\}^{n}$ from an unknown distribution:
\begin{enumerate}
\item Use a tester from literature (see \cite{ODonnellZ18,AlonAKMRX07,AlonGM03}) for testing $k$-wise independent distributions against distributions that are $\epsilon$-far from $k$-wise independent.
\item Output the same response as the one given by the $k$-wise independence tester. 
\end{enumerate}
\end{itemize}

\begin{thm}
[\textbf{Tester-learner pair for learning decision lists under uniform distribution on $\{\pm 1\}^n$}]
\label{thm: main testing learning theorem for decision lists on cube} Assume
$n$ and $\frac{1}{\epsilon}$ are larger than some sufficiently large
absolute constants. Then, the algorithm $\mathcal{A}_{\text{DL}}$ is an agnostic
$(O(\epsilon),0.1)$-learner for the function class of decision lists (see \Cref{definition: decision list}) over $\{\pm 1\}^{n}$ under the uniform distribution 
and the algorithm $\mathcal{T}_{\text{DL}}$ is an 
assumption tester
for $\mathcal{A}_{\text{DL}}$. The algorithms $\mathcal{A}_{\text{DL}}$ and $\mathcal{T}_{\text{DL}}$
both require only $n^{O\parr{\log \frac{1}{\epsilon}}}$ samples
and run-time. Additionally, the tester $\mathcal{T}_{\text{DL}}$ is label-oblivious.   
\end{thm}
The testers from the literature for $k$-wise independence take $n^{O(k)}/\eta^2$ samples and run-time to distinguish a $k$-wise independent distribution and a distribution that is $\eta$-far from $k$-wise independent (see \cite{ODonnellZ18,AlonAKMRX07,AlonGM03}). Thus, the run-time of tester $\mathcal{T}_{\text{DL}}$ is $n^{O(\log(1/\epsilon))}$. The same run-time bound of $n^{O(\log(1/\epsilon))}$ holds for $\mathcal{A}_{\text{DL}}$ for the following reason. There are only at most $n^k$ of size-$k$ subsets of $\{1,\cdots , n\}$ and there are at most $2^k \cdot k^k$ decision lists on each size-$k$ set. Substituting $k=\log(1/\epsilon)$ gives a bound of $n^{O(\log(1/\epsilon))}$ on the number of functions $f$ considered by the algorithm and hence on the run-time.

The only thing remaining to prove is that the algorithm $\mathcal{A}_{\text{DL}}$ is indeed a $(O(\epsilon),0.1)$-agnostic learning algorithm for the class of decision lists on $\{\pm 1\}^n$ with respect to distributions $D$ that are $\epsilon$-close to $k$-wise independent.

Let $D_{\text{pairs}}$ be the distribution from which we are getting example-label pairs.
For any function $g:\{\pm 1\}^n\righ \{\pm 1\}$ we let \emph{error of $g$} denote the flowing:
\[
\text{err}(g):
=
\Pr_{(\vect x, y)\sim D_{\text{pairs}}}
\pars{
g(\vect x)\neq y
}
\]
Let $\text{opt}$ be the smallest error among all decision lists.
We want to show that if the distribution of examples is $\epsilon$-close to $k$-wise independent, then the function $f_0$ that $\mathcal{A}_{\text{DL}}$ outputs has $\text{err}(f_0)$ that is at most $\text{opt}+O(\epsilon)$.

First of all, by the Hoeffding bound and the union bound, we have that with probability at least $0.9$ for every function considered by the algorithm $\mathcal{A}_{\text{DL}}$ it is the case that 
\[
\abs{
\widehat{\text{err}}(f)
-
\text{err}(f)
}
\leq
O(\epsilon).
\]
Thus, the only thing left to prove is that among the functions $f$ considered by $\mathcal{A}_{\text{DL}}$ there is one for which $\text{err}(f)$ is at most $\text{opt}+O(\epsilon)$. That follows from the following proposition:
\begin{prop}
\label{prop: low-degree approximation for close to k-wise independent on cube for decision lists}
Let $g$ be a decision list over $\{\pm1\}^n$ and let $k:=\log \frac{1}{\epsilon}$. Also let $D$ be a distribution that is $\epsilon$-close in TV distance to $k$-wise independent. Then, there is a decision list $f$ on a size-$k$ subset of $\{1,\cdots n\}$ for which
\[
\Pr_{\vect{x} \sim D }
\pars{
g(x) \neq f(x)
}
=
O(\epsilon)
\]
\end{prop}
\begin{proof} First, we recall the definition of a decision list.
For some ordering of the
variables $x_{\pi(1)},
\ldots, x_{\pi(n)}$, values $v_1,\ldots,v_n \in \{\pm 1\}$ and bits $b_1,\ldots,b_n \in \{\pm 1\}$,
the decision list $g$ does the following:  For $i = 1 {\rm ~to~} n$,
if $x_{\pi(i)}=b_{\pi(i)}$ it outputs $v_{\pi(i)}$, else it continues.

Let the decision list $g$ be defined on the first $k$ variables in the ordering $x_{\pi(1)},
\ldots, x_{\pi(n)}$ and let $g$ repeat the same comparisons and outputs as $f$ until it reaches the $k+1$-st variable. 

Recall that $D$ is only $\epsilon$-close in TV distance to a $k$-wise independent distribution. 
Let $D'$ be the closest $k$-wise independent distribution to $D$. Then, $D'$ is uniform on the first $k$ variables in the ordering on $x_{\pi(1)},
\ldots, x_{\pi(n)}$. Therefore, the execution of $f$ will reach past the $k$-th comparison only with probability at most $2^k=O(\epsilon)$. The same is true for function $g$ and therefore we have
\[
\Pr_{\vect{x} \sim D' }
\pars{
g(x) \neq f(x)
}
=
O(\epsilon).
\]

But from the definition of the TV distance we have that the function $\indicator_{g(x)\neq f(x)}$ should not allow us to distinguish $D$ and $D'$ with advantage better than $\epsilon$. Therefore
\[
\abs{
\Pr_{\vect{x} \sim D' }
\pars{
g(x) \neq f(x)
}
-
\Pr_{\vect{x} \sim D }
\pars{
g(x) \neq f(x)
}
}
\leq 
\epsilon.
\]
Together with the previous equation we conclude \[
\Pr_{\vect{x} \sim D}
\pars{
g(x) \neq f(x)
}
=
O(\epsilon).
\]

\end{proof}

\end{document}